\documentclass[lettersize,journal]{IEEEtran}
\usepackage{amsmath,amsfonts}
\usepackage{algcompatible}
\usepackage{algorithm}
\usepackage{array}
\usepackage[caption=false,font=normalsize,labelfont=sf,textfont=sf]{subfig}
\usepackage{textcomp}
\usepackage{stfloats}
\usepackage{url}
\usepackage{verbatim}
\usepackage{graphicx}
\usepackage{cite}
\usepackage{amsthm}

\usepackage{caption}
\usepackage{csvsimple}
\usepackage{filecontents} 
\usepackage{tikz}
\usepackage{amsmath,graphicx}
\usepackage{algorithm, algpseudocode}
\usepackage[utf8]{inputenc} 
\usepackage[T1]{fontenc}    
\usepackage{hyperref}       
\usepackage{url}            
\usepackage{booktabs}       
\usepackage{amsfonts}       
\usepackage{bm}
\usepackage{nicefrac}       
\usepackage{microtype}      
\usepackage{xcolor}         
\usepackage{cite}
\usetikzlibrary{shapes}
\usetikzlibrary{shadows}
\usepackage{pdfplots}
\usepackage{xspace}

\newcommand\strongconvparam[1]{\alpha^{(#1)}}
\newcommand\clusteropterr[1]{\delta^{(#1)}}

\newcommand{\estmse}{{\rm MSE}}

\newtheorem{assumption}{Assumption}

\definecolor{pinegreen}{cmyk}{0.92,0,0.59,0.25}
\definecolor{royalblue}{cmyk}{1,0.50,0,0}
\definecolor{lavander}{cmyk}{0,0.48,0,0}
\definecolor{violet}{cmyk}{0.79,0.88,0,0}
\tikzstyle{ncyan}=[circle, draw=cyan!70, thin, fill=white, scale=0.8, font=\fontsize{11}{0}\selectfont]
\tikzstyle{ngreen}=[circle,  draw=green!70, thin, fill=white, scale=0.8, font=\fontsize{11}{0}\selectfont]
\tikzstyle{nred}=[circle, draw=red!70, thin, fill=white, scale=0.8, font=\fontsize{11}{0}\selectfont]
\tikzstyle{ngray}=[circle, draw=gray!70, thin, fill=white, scale=0.55, font=\fontsize{14}{0}\selectfont]
\tikzstyle{nyellow}=[circle, draw=yellow!70, thin, fill=white, scale=0.55, font=\fontsize{14}{0}\selectfont]
\tikzstyle{norange}=[circle,  draw=orange!70, thin, fill=white, scale=0.55, font=\fontsize{10}{0}\selectfont]
\tikzstyle{npurple}=[circle,draw=purple!70, thin, fill=white, scale=0.55, font=\fontsize{10}{0}\selectfont]
\tikzstyle{nblue}=[circle, draw=blue!70, thin, fill=white, scale=0.55, font=\fontsize{10}{0}\selectfont]
\tikzstyle{nteal}=[circle,draw=teal!70, thin, fill=white, scale=0.55, font=\fontsize{10}{0}\selectfont]
\tikzstyle{nviolet}=[circle, draw=violet!70, thin, fill=white, scale=0.55, font=\fontsize{10}{0}\selectfont]
\tikzstyle{qgre}=[rectangle, draw, thin,fill=green!20, scale=0.8]
\tikzstyle{rpath}=[ultra thick, red, opacity=0.4]
\tikzstyle{legend_isps}=[rectangle, rounded corners, thin,fill=gray!20, text=blue, draw]
\usetikzlibrary{positioning}
\newtheorem{theorem}{Theorem}
\newtheorem{definition}[theorem]{Definition}
\newtheorem{lemma}[theorem]{Lemma}

\newcommand{\graphmnist}{\graph^{(\rm MNIST)}}
\newcommand{\nodesmnist}{\nodes^{(\rm MNIST)}}
\newcommand{\edgesmnist}{\edges^{(\rm MNIST)}}
\newcommand{\pdgap}[1]{{\rm gap}_{#1}} 
\newcommand{\localvalsetsize}[1]{m^{(\rm val)}_{#1}} 
\newcommand{\localtrainsetsize}[1]{m^{(\rm train)}_{#1}} 

\newcommand\defeq{:=}

\newcommand{\vx}[0]{{\bf x}}
\newcommand{\vv}[0]{{\bf v}}
\newcommand{\vu}[0]{{\bf u}}

\newcommand{\mW}[0]{{\bf W}}

\newcommand{\mD}[0]{{\bf D}}

\newcommand{\vw}[0]{{\bf w}}

\newcommand{\mI}{\mathbf{I}}

\newcommand{\mQ}{\mathbf{Q}}

\newcommand{\va}[0]{{\bf a}}

\newcommand{\vz}[0]{{\bf z}}

\newcommand{\neighbourhood}[1]{\mathcal{N}^{(#1)}}

\newcommand{\norm}[1]{\Vert  {#1} \Vert}
\newcommand{\normgeneric}[2]{\left\Vert  {#1} \right\Vert_{#2}}

\newcommand{\bmx}[0]{\begin{bmatrix}}
\newcommand{\emx}[0]{\end{bmatrix}}

\newcommand{\samplingset}{\mathcal{M}}
\newcommand{\trainingset}{\samplingset}
\newcommand{\samplesize}{m}
\newcommand{\sampleidx}{r}

\newcommand{\clusteridx}{c} 

\newcommand{\nrcluster}{k} 
 
\newcommand{\featureidx}{j}

\newcommand\truelabel{y}

\newcommand\featurevec{\vx}

\newcommand{\regparam}{\lambda}

\DeclareMathOperator*{\argmin}{argmin}
\newcommand{\itercntr}{k}

\newcommand{\iteridx}{k}

\newcommand{\cluster}[1]{\mathcal{C}^{(#1)}}
\newcommand{\clustergeneric}{\mathcal{C}}

\newcommand{\weights}{\vw}

\newcommand{\nriter}{R}

\newcommand{\proximityop}[3]{{\rm\bf prox}_{#1,#3}(#2)}

\newcommand{\locallossfunc}[2]{L_{#1}\left(#2 \right)}
\newcommand{\conjlocallossfunc}[2]{L^{*}_{#1}\left(#2 \right)}

\newcommand{\localdataset}[1]{\mathcal{X}^{(#1)}}
\newcommand{\localdatasetval}[1]{\mathcal{X}_{\rm val}^{(#1)}}
\newcommand{\localdatasettrain}[1]{\mathcal{X}_{\rm train}^{(#1)}}
\newcommand{\edges}{\mathcal{E}}

\newcommand{\edgeweight}{A}

\newcommand{\genericnodeset}{\mathcal{A}}
\newcommand{\edgeidx}{e}

\newcommand{\graph}{\mathcal{G}}
\newcommand{\nodes}{\mathcal{V}}
\newcommand{\incidencemtx}{\mathbf{D}}
\newcommand{\incidencemtxentry}[2]{D_{#1,#2}}

\newcommand{\nodeidx}{i}
\newcommand{\nrnodes}{n}

\newcommand{\nodespace}{\mathcal{W}}
\newcommand{\edgespace}{\mathcal{U}}

\newcommand{\edge}[2]{\{#1,#2\}}
\newcommand{\directededge}[2]{\left(#1,#2\right)}

\newcommand{\gtvpenalty}{\phi} 
\newcommand{\partition}{\mathcal{P}}

\newcommand{\localsamplesize}[1]{m_{#1}}
\newcommand{\localsampleidx}{r}
\newcommand{\dimlocalmodel}{d}
\newcommand{\pair}[2]{\left( #1,#2 \right)}
\newcommand{\localparams}[1]{\mathbf{w}^{(#1)}}
\newcommand{\localflowscalar}[1]{u^{(#1)}}
\newcommand{\elemflowscalar}[2]{u^{(#2)}_{#1}}
\newcommand{\simplepath}[1]{\mathcal{P}^{(#1)}}
\newcommand{\localflowvec}[1]{\mathbf{u}^{(#1)}}
\newcommand{\cycleidx}{r}
\newcommand{\estlocalflowvec}[1]{\widehat{\mathbf{u}}^{(#1)}}
\newcommand{\optlocalflowvec}[1]{\widehat{\mathbf{u}}^{(#1)}}
\newcommand{\optlocalflowvecclusteredg}[1]{\widetilde{\vu}^{(#1)}}
\newcommand{\estlocalparamsclusteredg}[1]{\widetilde{\vw}^{(#1)}}
\newcommand{\flowvec}{\mathbf{u}}
\newcommand{\estlocalparams}[1]{\widehat{\mathbf{w}}^{(#1)}}

\newcommand{\clusterobj}[2]{f^{(#1)}\left( #2 \right)}
\newcommand{\locallipsch}[1]{\beta^{(#1)}}
\newcommand{\netparams}{\mathbf{w}}
\newcommand{\nodeweight}[1]{\rho_{#1}}
\newcommand{\primalupdate}[2]{\mathcal{PU}^{(#1)} \left\{#2\right\}}
\newcommand{\dualupdate}{\mathcal{DU}}
\newcommand{\clusterwideopt}[1]{\overline{\weights}^{(#1)}}
\newcommand{\demand}[1]{\mathbf{s}^{(#1)}}

\newcommand{\demandbound}[1]{\delta^{(#1)}}
\newcommand{\thresholdwassdist}{\eta}

\begin{document}

\title{Clustered Federated Learning via Generalized Total Variation Minimization}
%
%
%

\author{Yasmin~SarcheshmehPour, Yu Tian, Linli Zhang and Alexander Jung
	\thanks{}
	\thanks{}
}

\maketitle

\markboth{Some Journal}%
{Shell \MakeLowercase{\textit{et al.}}: Bare Demo of IEEEtran.cls for IEEE Journals}

\begin{abstract}
We study optimization methods to train local (or personalized) models for 
decentralized collections of local datasets with an intrinsic network structure. This 
network structure arises from domain-specific notions of similarity between local datasets. 
Examples for such notions include spatio-temporal proximity, statistical dependencies 
or functional relations. Our main conceptual contribution is to formulate federated learning as 
generalized total variation (GTV) minimization. This formulation unifies and considerably 
extends existing federated learning methods. It is highly flexible and can be combined 
with a broad range of parametric models, including generalized linear models or deep neural 
networks. Our main algorithmic contribution is a fully decentralized federated learning algorithm. 
This algorithm is obtained by applying an established primal-dual method to solve GTV minimization. It 
can be implemented as message passing and is robust against inexact computations that arise 
from limited computational resources including processing time or bandwidth. 
Our main analytic contribution is an upper bound on the deviation between the local model 
parameters learnt by our algorithm and an oracle-based clustered federated learning 
method. This upper bound reveals conditions on the local models and the network structure 
of local datasets such that GTV minimization is able to pool (nearly) homogeneous local datasets. 

\end{abstract}


%

\section{Introduction}
\label{sec:intro}

Many important application domains generate collections of local datasets that are related via 
an intrinsic network structure \cite{BigDataNetworksBook}. Two timely application domains 
generating such networked data are (i) the high-precision management of pandemics and (ii) the Internet of Things (IoT) \cite{Wollschlaeger2017}.
Such local datasets are generated by smartphones, wearables or industrial IoT devices \cite{Ates:2021ug}.
These local datasets are related via physical contact networks, social networks \cite{NewmannBook}, 
co-morbidity networks \cite{NetMedNat2010}, or communication networks \cite{Grantz:2020wn}. 

Federated learning (FL) is an umbrella term for machine learning (ML) techniques that collaboratively 
train models on decentralized collections of local datasets \cite{pmlr-v54-mcmahan17a,Cheng2020,Smith2017}. 
These methods carry out computations such as gradient descent steps during model training at the location 
of data generation, rather than first collecting all data at a central location \cite{ShipCompute}. 
FL methods are appealing for applications involving sensitive data (such as healthcare) as they 
do not require the exchange of raw data but only model (parameter) updates without leaking sensitive information in local datasets \cite{Cheng2020}. 
Moreover, FL methods can offer robustness 
against malicious data perturbation due to its intrinsic averaging or aggregation over large collections 
of (mostly benign) datasets \cite{Sattler2020}.

FL applications often face local datasets with different statistical properties \cite{Ghosh2020}. 
Each local dataset induces a separate learning task that consists of learning (or optimizing) 
the parameters of a local model. This paper studies an optimization method to train local models 
that are tailored (or ``personalized'') to the statistical properties of the corresponding local dataset. This method 
is an instance of regularized empirical risk minimization (or structural risk minimization). In particular, it 
uses a measure (see Section \ref{sec_nLasso}) for the variation of local model parameters as regularizer. We 
solve the resulting optimization (or learning) problem using a primal-dual method that can be implemented 
as message passing over the network structure of local datasets (see Section \ref{sec_federatedml}). 

Clustered FL addresses the heterogeneity of local datasets using various forms of a clustering 
assumption \cite{NIPS2008_fccb3cdc,Ghosh2020,SemiSupervisedBook}. Informally, our clustering 
assumption requires local datasets and their associated learning tasks to form a few disjoint
subsets or clusters. As a result, local datasets belonging to the same cluster have similar 
statistical properties and, in turn, similar optimal parameter values for the corresponding local models. 
Section \ref{sec:format} makes this clustering assumption precise via Assumption \ref{asspt_weights_clustered}. 
The main contribution of this paper is a detailed characterization of the cluster structure and local model 
geometry for local datasets that allow our methods to pool local datasets that form statistically 
homogeneous clusters of datasets (see Section \ref{sec_when_does_it_work}).

What sets our approach apart from existing methods for clustered FL \cite{Ghosh2020,NIPS2008_fccb3cdc} 
is that we exploit known pairwise similarities between local datasets. These similarities are encoded by 
the weighted undirected edges of an \emph{empirical graph} \cite{SemiSupervisedBook}. 
Instead of a trivial combination of clustering methods and cluster-wise model training, 
our FL method (see Algorithm \ref{alg1}) interweaves the pooling of local datasets with model training. 
We use the connectivity of the empirical graph to guide this pooling (see Section \ref{sec_federatedml}).


Our FL method requires a useful choice for the empirical graph of networked data. 
If a useful choice for the empirical graph is not obvious, we might use principled statistical tests 
for the similarity between two datasets. These tests could be based on parametric models 
such as (mixtures of) Gaussian distributions or non-parametric methods for density estimation \cite{Ge2015,Lee2023,PerezCruz2008,ProofTriangleWass}. 
We demonstrate some of these methods in the numerical experiments of Section \ref{sec_numexp}. 
However, the details of graph learning methods for collections of local datasets is beyond the scope 
of this paper (see Section \ref{sec_conclusion}). 

\vspace*{-3mm}
\subsection{Related Work}

Similar to \cite{NetworkLasso,LocalizedLinReg2019,Smith2017,Nassif2020,Xu2011}, 
we use regularized empirical risk minimization (RERM) to learn tailored models for local datasets. 
For each local dataset, we obtain a separate learning task that amounts to finding 
an (approximately) optimal choice for the parameters of a local model. These individual 
learning tasks are coupled via the undirected weighted edges of an empirical graph (see Section \ref{sec:format}). 
In contrast to \cite{Xu2011}, which uses a probabilistic model for the empirical graph, 
we consider the empirical graph as fixed and known (non-random). 

To capture the intrinsic cluster structure of networked data, we use a generalized total 
variation (GTV) of the local model parameters as the regularizer. GTV unifies and extends 
several existing notions of total variation \cite{NetworkLasso,LocalizedLinReg2019,Smith2017,Nassif2020}. 
GTV is parametrized by a penalty function which is used to measure the difference of local model 
parameters at neighbouring nodes in the empirical graph. Computationally, the main restriction for 
the choice of penalty function is that it must allow for efficient computation of the corresponding proximal  
operator \eqref{equ_def_proximity_operator}. Some authors refer to such functions as ``proximable'' \cite{Condat2013}. 


Our analysis reveals conditions on the network structure between local datasets and their 
local models such that GTV minimization is able to identify the cluster structure of the empirical graph. 
This is relevant for the application of GTV minimization to clustered FL \cite{NIPS2008_fccb3cdc,Ghosh2020}. 
In contrast to existing work on clustered FL, we exploit a known similarity structure between local 
datasets. We represent these similarities by the edges of an empirical graph. 


GTV minimization unifies and considerably extends well-known optimization models for FL, 
including the TV minimization and network Lasso (nLasso) \cite{NetworkLasso,LocalizedLinReg2019,Smith2017,Nassif2020}. 
GTV is an instance of the non-quadratic regularizer put forward in \cite{Nassif2020}. 
In contrast to \cite{Nassif2020}, which uses a combination of gradient descent and distributed averaging 
methods, we use a primal-dual method to solve the resulting GTV minimization. We provide precise 
conditions such that GTV minimization captures the inherent cluster structure of networked data. 


A substantial body of work studies computational aspects of GTV minimization \cite{NIPS2008_fccb3cdc}. 
Efficient algorithms for GTV minimization have been proposed for relevant computational infrastructures 
such as wireless networks of low-complexity devices \cite{DistrOptStatistLearningADMM,NedicTransAC2009}. 
We would like to highlight the recent study \cite{NEURIPS2018_8fb21ee7} of intrinsic computational complexity 
and efficient primal-dual methods for GTV minimization. While \cite{NEURIPS2018_8fb21ee7} only uses the 
network diameter of the empirical graph to characterize the convergence of optimization methods, our analysis 
involves more fine-grained properties of the empirical graph. Moreover, our analysis aims at the estimation error 
instead of convergence properties of FL algorithms. 

We obtain practical FL methods by solving GTV minimization using an established primal-dual (``Chambolle-Pock'') 
method for non-smooth convex optimization \cite[Alg. 6]{pock_chambolle_2016}. As the name suggests, this primal-dual 
method jointly solves GTV minimization and a dual problem. This method is widely used in image processing (see \cite{pock_chambolle_2016} 
and references therein) and has been applied to a special case of GTV minimization in our previous work \cite{LocalizedLinReg2019}. 
This paper generalizes the methods and analysis of \cite{LocalizedLinReg2019} to a significantly larger class of local models 
and total variation measures. In particular, \cite{LocalizedLinReg2019} studies the special case of GTV minimization 
obtained for local linear models and absolute error loss. Here, we consider GTV minimization methods that can be 
combined with a wide range of (potentially non-linear) parametrized models including graphical Lasso or deep 
neural networks \cite{HastieWainwrightBook,Goodfellow-et-al-2016}. 

The primal-dual method \cite[Alg. 6]{pock_chambolle_2016} is well-suited for FL applications 
in several aspects. First, as we show in Section \ref{sec_federatedml}, the primal-dual method \cite[Alg. 6]{pock_chambolle_2016} 
can be implemented as a message-passing protocol over the empirical graph. Message-passing algorithms 
are scalable to massive collections of local datasets as long as their empirical graph is sparse (e.g., a bounded 
degree network) \cite{Yedidia:2011aa}. Moreover, the primal-dual method\cite[Alg. 6]{pock_chambolle_2016} also 
offers robustness against limited computational resources and imperfections \cite{Rasch:2020tx}. 
This robustness is crucial for the applicability of our FL method as its basic computational step 
is a (separate) regularized model training for each local dataset. Given finite computational 
resources, such a local model updates can only be solved up to some non-zero optimization error. 

This paper develops and exploits a duality between GTV minimization and network flow 
optimization \cite{JungDualitynLasso}. It lends naturally to the design and analysis of 
primal-dual methods for solving instances of GTV minimization arising in FL. Our approach 
differs conceptually and algorithmically from existing primal-dual methods for FL \cite{Smith2017,DistrOptStatistLearningADMM,NetworkLasso}. 
Moreover, it significantly generalizes our previous work \cite{LocalizedLinReg2019} 
on a primal-dual method for a special case of GTV minimization. 

Algorithmically, our method is an instance of the proximal-point algorithm \cite{He2014}, which is 
different from the dual coordinate ascent method in \cite{Smith2017,JMLR:v14:shalev-shwartz13a} 
and also different from the alternating direction method of multipliers (ADMM) used in \cite{NetworkLasso}. 
We refer to \cite{pock_chambolle_2016,SmithCoCoA} for more discussion of the differences and similarities between 
duality-based methods, including ADMM \cite{DistrOptStatistLearningADMM}. 

Another main difference between \cite{Smith2017,NetworkLasso} and our method is that we allow 
for a wider range of penalty functions to measure the variation of local model parameters across an 
edge in the empirical graph. Indeed, our Algorithm can be combined with any penalty function that is 
proximal in the sense of having a proximal operator that can be computed efficiently. 

Conceptually, the dual problem in our approach is a network flow optimization problem with optimization 
variables interpreted as (vector-valued) flows along edges in the empirical graph (see Section \ref{sec_interpreations}). 
In contrast, the dual problem in \cite{Smith2017} does not allow for an obvious interpretation as a network flow optimization. 
As a case in point, the dual variables in \cite{Smith2017} are associated with the nodes of the empirical graph, 
their dimension being the local sample size. We introduce a dual variable (vector) for 
each edge in the empirical graph. These dual vectors have identical lengths which is the (common) 
dimension of the local models at the nodes of the empirical graph. 

In contrast to its computational aspects, the statistical aspects of GTV minimization are far less 
understood \cite{NetworkLasso,Smith2017}. It is possible to frame GTV minimization as the 
learning or recovery of group-sparse models which have been thoroughly studied within high-dimensional 
statistics \cite{BuhlGeerBook,Wain2019}. However, it is unclear how the group-sparse models underlying GTV 
minimization are related to the fine-grained properties (such as cluster structure) of the empirical graph. 
Our main contribution is a characterization of the clustering behaviour of GTV minimization.  

The closest to our work is the recent analysis of convex clustering \cite{JMLR:v22:18-694}. 
Indeed, convex clustering is a special case of GTV minimization (see Section \ref{sec_interpreations}). 
Moreover, in contrast to \cite{JMLR:v22:18-694}, we characterize the cluster structure of GTV minimization 
using network flows. These flows are injected and absorbed at the nodes via the gradient of the loss function 
used to train the local models. 

Finally, we like to comment on the relation between our work and graph clustering methods \cite{IntroLGCAlgSoftw2018,Lang2004,Veldt2019,JungLocalGraphClustering}. 
In particular, GTV minimization generalizes graph clustering methods in the sense of not only 
taking into account the connectivity (edges) in the empirical graph but also the shape of local 
loss functions (which are used to train the local models). The main theme of this paper is the interplay 
between the edge connectivity of the empirical graph and the shape of local loss functions within GTV minimization.



\subsection{Contribution} 

We next enumerate the main contributions of this paper. 
\begin{itemize} 
\item We propose GTV minimization as a flexible design principle for distributed FL algorithms. 
GTV minimization is an instance of RERM using the variation of local model parameters as a regularizer.
 GTV minimization unifies and extends existing optimization models for FL, including nLasso \cite{NetworkLasso,LocalizedLinReg2019}, 
MOCHA \cite{Smith2017} and clustering methods \cite{JungLocalGraphClustering,JMLR:v22:18-694}. 



\item We show that GTV minimization is dual (in a very precise sense) to vector-valued network flow optimization \cite{BertsekasNetworkOpt}. 
This duality generalizes some well-known duality results for network optimization \cite{RockNetworks} and our 
own recent work on special cases of GTV minimization \cite{JungDualitynLasso}. 

\item We present a novel FL algorithm which is obtained by applying an established primal-dual method 
to solve GTV minimization and its dual  \cite{Condat2013}  \cite[Alg. 6]{pock_chambolle_2016}. 
The resulting Algorithm \ref{alg1} 
can be combined with a wide range of parametric local models and variants of TV 
(obtained for different GTV penalty functions). The only requirement on the local models 
is the existence of efficient RERM implementations (see \eqref{equ_node_wise_primal_update_def}). 
Likewise, the choice for the GTV penalty function is only restricted by requiring an efficient way to 
evaluate its convex conjugate (see \eqref{equ_edge_wise_dual_update_min_def}). 
\item Using a clustering assumption on the local datasets, we derive an upper bound on the estimation 
error incurred by GTV minimization. This upper bound reveals sufficient conditions on the local models 
and their network structure such that GTV minimization is able to pool local datasets in the same cluster. 
We hasten to note that our analysis only applies to GTV minimization (i) using a penalty function being a 
norm and (ii) local models resulting in convex training problems. Thus, our bounds do not apply to methods 
that either use graph Laplacian quadratic form as regularizer (such as MOCHA) or local models resulting in 
non-convex loss functions (deep nets). 
\end{itemize}

\subsection{Outline} 
Section \ref{sec:format} introduces the concept of an empirical graph to represent collections of local 
datasets, the corresponding local models as well as their similarity structure. 
Section \ref{sec_nLasso} 
introduces GTV as a measure for the variation of local model parameters across the edges in 
the empirical graph. As discussed in Section \ref{sec_primal_problem}, GTV minimization balances  
the variation of local model parameters over well-connected local datasets (forming a cluster) and incurring 
a small loss (training error) for each local dataset. The dual problem to GTV minimization is then 
explained in Section \ref{sec_dual_problem_GTV_Min}. Section \ref{sec_interpreations} presents several 
useful interpretations of GTV minimization and its dual. 
Section \ref{sec_federatedml} applies a well-known primal-dual optimization method 
to solve GTV minimization and its dual in a fully distributed fashion via message passing over the empirical graph (see Algorithm \ref{alg1}). 
The results of numerical experiments are discussed in Section \ref{sec_numexp}. 

\subsection{Notation} 
The identity matrix of size $n\!\times\!n$ is denoted $\mathbf{I}_{n}$, with the subscript omitted
if the size $n$ is clear from the context. We use $\| \cdot \|$ to denote some norm defined on the Euclidean 
space $\mathbb{R}^{\dimlocalmodel}$ and $\| \cdot \|_{*}$ to denote its dual norm \cite[Appx. 1.6.]{BoydConvexBook}. 
Two important examples are the Euclidean norm $\| \vw \|_{2} \!\defeq\!\sqrt{\sum_{\featureidx=1}^{\dimlocalmodel} w_{\featureidx}^{2}}$ and 
the $\ell_{1}$ norm $\| \vw \|_{1} \!\defeq\!\sum_{\featureidx=1}^{\dimlocalmodel} |w_{\featureidx}|$ of a vector $\vw \!= \!(w_{1},\ldots,w_{\dimlocalmodel})^{T} \in \mathbb{R}^{\dimlocalmodel}$ . 
It will be convenient to use the notation $(1/2\tau)$ instead of $(1/(2\tau))$. 
We will need the (vector-wise) clipping operator
\vspace{-2mm}
\begin{equation} 
\vspace{-2mm}
	\label{equ_vector_clipping}
	\mathcal{T}^{(\gamma)}(\vw)\!\defeq\! \begin{cases} \gamma \vw/\|\vw\|_{2} & \mbox{ for } \| \vw \|_{2} \!\geq\!\gamma \\
		\vw  & \mbox{ otherwise.} \end{cases} 
\end{equation} 
The scalar clipping operator $\mathcal{T}^{(\gamma)}(w)$ is obtained as a special case of \eqref{equ_vector_clipping} 
by considering the scalar $w$ as a vector with a single entry (where $\normgeneric{\mathbf{w}}{2} = |w|$). 
Given a closed proper convex function $f(\vx)$ with domain being a subset of $\mathbb{R}^{\dimlocalmodel}$, 
we define its associated convex conjugate function as \cite{BoydConvexBook}
\begin{equation} 
\label{equ_def_conv_conj}
f^{*} (\vx) \defeq \sup\limits_{\vz \in \mathbb{R}^{\dimlocalmodel}}  \vx^{T} \vz - f(\vz). 
\vspace{-2mm}
\end{equation} 
We will use the proximal operator of a closed proper convex function $f(\vx)$, 
defined as \cite{DistrOptStatistLearningADMM} 
\begin{equation}
\label{equ_def_proximity_operator}
\proximityop{f}{\vx}{\rho}\!\defeq\!\argmin_{\vx'} f(\vx')\!+\!(\rho/2) \| \vx - \vx'\|^{2}_{2} \mbox{ with } \rho\!>\!0.
\end{equation} 
Note that the minimum in \eqref{equ_def_proximity_operator} exists and is unique since the objective 
function is strongly convex \cite{BoydConvexBook}.

\section{Problem Formulation}
\label{sec:format}


We find it useful to represent networked data by an undirected weighted \emph{empirical graph} $\graph=(\nodes,\edges)$. 
For notational convenience, we identify the nodes of an empirical graph with natural numbers, $\nodes = \{1,\ldots,\nrnodes\}$. 
Each node $ \nodeidx \!\in\!\nodes$ of the empirical graph $\graph$ carries a separate local dataset $\localdataset{\nodeidx}$. 
It might be instructive to think of a local dataset  $\localdataset{\nodeidx}$ as a labeled dataset 
\vspace{-2mm}
\begin{equation} 
	\label{equ_def_local_dataset_plain}
\localdataset{\nodeidx} \defeq \left\{ \big(\featurevec^{(\nodeidx,1)},\truelabel^{(\nodeidx,1)}\big), \ldots,\big(\featurevec^{(\nodeidx,\samplesize_{\nodeidx})},\truelabel^{(\nodeidx,\samplesize_{\nodeidx})}\big) \right\}. 
\vspace{-2mm}
\end{equation} 
Here, $\featurevec^{(\sampleidx)}$ and $\truelabel^{(\sampleidx)}$ denote, respectively, the 
feature vector and true label of the $\sampleidx$-th data point in the local dataset $\localdataset{\nodeidx}$. 
Note that the size $\samplesize_{\nodeidx}$ of the local dataset might vary across nodes $\nodeidx \in \nodes$. 
Figure \ref{fig_local_dataset} depicts an empirical graph with $\nrnodes\!=\!11$ nodes $\nodes=\{1,\ldots,\nrnodes\}$, 
each carrying a local dataset $\localdataset{\nodeidx}$.  

We highlight that our method (see Section \ref{sec_federatedml}) is not restricted to local 
datasets of the form \eqref{equ_def_local_dataset_plain}. Indeed, Algorithm \ref{alg1} and its 
analysis (see Section \ref{sec_when_does_it_work}) only requires indirect access to $\localdataset{\nodeidx}$ 
via the evaluation of some local loss function $\locallossfunc{\nodeidx}{\vv}$. 
The value $\locallossfunc{\nodeidx}{\vv}$ measures how well a model with parameters $\vv$ fits the 
local dataset $\localdataset{\nodeidx}$ (see Section \ref{sec_net_models}). We will study different 
choices for the local loss function in Section \ref{sec_numexp}. 

Let us point out two particular aspects of our data-access model via the evaluation of local loss functions. 
First, it lends naturally to privacy-friendly methods as they do not need to share raw data $\localdataset{\nodeidx}$. 
Instead, our methods only exchange (local) information about local loss function $\locallossfunc{\nodeidx}{\cdot}$ 
such as the gradient $\nabla \locallossfunc{\nodeidx}{\vv}$ or the proximal operator value $\proximityop{\locallossfunc}{\vv}{\rho}$ \eqref{equ_def_proximity_operator}
for a given choice $\localparams{\nodeidx}=\vv$ for the local model parameter vector. This information is 
typically obtained from averages over data points and therefore revealing only a little information about 
individual data points (if the sample size is not too small). 

Besides its privacy-friendliness, our data access model also handles applications where only 
a fraction of local datasets are accessible. This is relevant for wireless sensor networks that 
consist of battery-powered devices for computation and wireless communication \cite{Mahmoudi202}. 
The lack of access to the local dataset at some node $\nodeidx$ can be taken into account by 
using a trivial loss function $\locallossfunc{\nodeidx}{\vv}  =0$ for all parameter vectors $\vv \in \mathbb{R}^{\dimlocalmodel}$ (see Section \ref{sec_numexp}). 



\begin{figure} 
\begin{center}
\begin{tikzpicture}[scale=8/5]
\tikzstyle{every node}=[font=\small]
\node[nred] (C1_2) at (0.88+3.7,2.29) {};
\node[left=1 cm of C1_2,nred] (C1_1)  {};
\node[below left =1cm and 1cm of C1_2,nred] (C1_3)  {};
\node[below =0.5cm of C1_2,nred] (C1_4)  {};
\node[ngreen] (C3_3) at (6,2) {};
\node[above left =0.4cm and 0.7cm of C3_3,ngreen] (C3_2)  {};
\node[below left =0.4 and 0.7cm of C3_3,ngreen] (C3_4) {};
\node[left =1.2cm of C3_3,ngreen] (C3_1) {};
\node[ncyan] (C2_2) at (2.0,2.23) {};
\node[below left =0.4cm and 0.4cm of C2_2,ncyan] (C2_1)  {};
\node[below right =0.4cm and 0.4cm of C2_2,ncyan] (C2_3)  {};
\node[below right = 0.01cm and 0.00cm of C2_3, font=\fontsize{8}{0}\selectfont,anchor=west] {$\mathcal{X}^{(\nodeidx)}, \weights^{(\nodeidx)}$}; 
\node[above right = 0.01cm and 0.00cm of C1_1, font=\fontsize{8}{0}\selectfont,anchor=south] {$\mathcal{X}^{(\nodeidx')}, \weights^{(\nodeidx')}$}; 
\node[above right = 0.01cm and 0.00cm of C1_2, font=\fontsize{8}{0}\selectfont,anchor=west] {$\overline{\weights}^{(\cluster{1})}$}; 
\node[below right = 0 and 0.00cm of C2_1,font=\fontsize{8}{0}\selectfont,anchor=north east] {$\overline{\weights}^{(\cluster{2})}$}; 
\node[below right = 0 and 0.0cm of C3_3,font=\fontsize{8}{0}\selectfont,anchor=north west]  {$\overline{\weights}^{(\cluster{3})}$}; 

\draw [line width=0.3mm,-] (C2_3)--(C1_1) node[draw=none,fill=none,font=\fontsize{8}{0}\selectfont,midway,above] {$\edgeweight_{\nodeidx,\nodeidx'}$};
\draw [line width=0.6mm,-] (C1_2)--(C1_1);
\draw [line width=0.4mm,-] (C1_2)--(C1_3);
\draw [-] (C1_1)--(C1_3);
\draw [-] (C1_3)--(C1_4);
\draw [-] (C1_2)--(C1_4);
\draw [line width=0.3mm,-] (C1_4)--(C3_1);
\draw [line width=0.6mm,-] (C2_1)--(C2_2);
\draw [line width=0.4mm,-] (C2_2)--(C2_3);
\draw [line width=0.4mm,-] (C2_1)--(C2_3);
\draw [-] (C3_1)--(C3_2);
\draw [-] (C3_2)--(C3_3);
\draw [line width=0.6mm,-] (C3_3)--(C3_4);
\draw [-] (C3_2)--(C3_4);
\draw [-] (C3_1)--(C3_4);
\end{tikzpicture}
\caption{\label{fig_local_dataset} We represent networked data and corresponding models using an undirected empirical graph $\graph=\big(\nodes,\edges\big)$. 
	Each node $\nodeidx \in \nodes$ of the graph carries a local dataset $\localdataset{\nodeidx}$ and model parameters $\localparams{\nodeidx}$ which 
	are scored using a local loss function $\locallossfunc{\nodeidx}{\localparams{\nodeidx}}$ (that encapsulates the local dataset $\localdataset{\nodeidx}$). 
	Two nodes are connected by a weighted edge $\{\nodeidx,\nodeidx'\}$ if they carry datasets with similar statistical properties. 
	The amount of similarity is encoded in an edge weight $\edgeweight_{\nodeidx,\nodeidx'}>0$ (indicated by the thickness of the links). 
	We rely on a clustering assumption, requiring optimal parameter vectors for nodes in the same cluster 
	$\cluster{\clusteridx} \subseteq \nodes$ to be nearly identical. 
	The empirical graph is partitioned into three disjoint clusters $\cluster{1},\cluster{2},\cluster{3}$. 
	Note that our FL method does not require the (typically unknown) partition but rather learns the 
	partition based on the local datasets and network structure of $\graph$. 
}
\end{center}
\vspace{-3mm}
\end{figure} 

An undirected edge $\{\nodeidx,\nodeidx'\}\!\in\!\edges$ indicates that the corresponding local 
datasets $\localdataset{\nodeidx}$ and $\localdataset{\nodeidx'}$ have similar statistical properties. 
In particular, two connected local datasets $\localdataset{\nodeidx}$ and $\localdataset{\nodeidx'}$ 
might be pooled together to obtain a training set for a single model. We will make this vague characterization 
more precise in Assumption \ref{asspt_FIM_lower_bound} and Definition \ref{equ_def_well_connected_cluster}. 

The strength of the similarity between two connected nodes $\nodeidx,\nodeidx'$ is quantified by the 
edge weight $\edgeweight_{\nodeidx,\nodeidx'}\!>\!0$. We also use $\edgeweight_{\edgeidx} \defeq  \edgeweight_{\nodeidx,\nodeidx'}$ for 
an edge $\edgeidx = \edge{\nodeidx}{\nodeidx'}$. It will be convenient to indicate the absence of an edge between 
two nodes $\nodeidx,\nodeidx' \in \nodes$ by a zero weight, i.e., $\edgeweight_{\nodeidx,\nodeidx'} = 0$ if 
and only if $\{\nodeidx,\nodeidx'\} \notin \edges$. 

Our approach requires the edges $\edges$ of the empirical graph to be known. While this assumption might 
seem restrictive, we argue that many important application domains offer a natural choice or construction 
for the edge set $\edges$. In particular, the construction of a useful edge set might be straightforward if 
we interpret the data points in a local dataset $\localdataset{\nodeidx}$ as i.i.d.\ realizations from a probability distribution $p^{(\nodeidx)}(\cdot)$. 
We can then use different estimators for the similarity between probability distributions. This includes 
parametric methods, e.g.,based on sample mean and covariance \cite{LC,hastie01statisticallearning}, 
and non-parametric methods (see \cite{ProofTriangleWass,coverthomas} and Section \ref{mnist_section}).  

The undirected edge $\edge{\nodeidx}{\nodeidx'} \in \edges$ encodes a symmetric notion of similarity 
between local datasets. If the local dataset at node $\nodeidx$ is (statistically) similar to the local 
dataset at node $\nodeidx'$ then also vice-versa. The symmetric nature of the similarities between 
local datasets is also reflected in the edge weights, 
\vspace{-2mm}
$$\edgeweight_{\nodeidx,\nodeidx'} = \edgeweight_{\nodeidx', \nodeidx} \mbox{ for any two nodes } \nodeidx, \nodeidx' \in \nodes. \vspace{-2mm}$$
It will be convenient for the formulation and analysis of our FL method (see Algorithm \ref{alg1}) to orient the edges in $\edges$. 
In particular, we define the head and tail of an undirected edge $\edgeidx=\{\nodeidx,\nodeidx'\} \in \edges$ 
as $\edgeidx_{+} \defeq \min\{\nodeidx,\nodeidx'\}$ and $\edgeidx_{-} \defeq \max\{\nodeidx,\nodeidx'\}$, respectively. 
The entire set of directed edges for an empirical graph is obtained as  
\begin{equation} 
	\label{equ_def_directed_edges_empgraph}
\overrightarrow{\edges} \defeq \big\{\directededge{\nodeidx}{\nodeidx'}: \nodeidx,\nodeidx' \in \nodes, \nodeidx < \nodeidx' \mbox{ and } \edge{\nodeidx}{\nodeidx'} \in \edges \big\}. 
\end{equation} 
We abuse notation and use $\edges$ not only to denote the set of undirected edges but 
also to denote the set \eqref{equ_def_directed_edges_empgraph} of directed edges in the empirical graph $\graph$. 
 
There are two vector spaces that are naturally associated with an empirical graph $\graph$. 
The ``node space'' $\nodespace$ consists of maps 
$\netparams: \nodes \rightarrow \mathbb{R}^{\dimlocalmodel}: \nodeidx \mapsto \localparams{\nodeidx}$
that assign each node $\nodeidx \in \nodes$ a vector $\localparams{\nodeidx} \in \mathbb{R}^{\dimlocalmodel}$. 
The ``edge space'' $\edgespace$ of all maps 
$\vu: \edges \rightarrow \mathbb{R}^{\dimlocalmodel}: \edgeidx \mapsto \localflowvec{\edgeidx} $
that assign each edge $\edgeidx \in \edges$ a vector $\localflowvec{\edgeidx} \in \mathbb{R}^{\dimlocalmodel}$. 
These two spaces are linked via the block-incidence matrix $\incidencemtx$ with entries 
$\incidencemtxentry{\edgeidx}{\nodeidx} = 1$ for $\nodeidx = \edgeidx_{+}$, $\incidencemtxentry{\edgeidx}{\nodeidx} = -1$ for $\nodeidx = \edgeidx_{-}$, and $\incidencemtxentry{\edgeidx}{\nodeidx} = 0$ otherwise.
The  block-incidence matrix $\incidencemtx$ represents a linear map 
\vspace{-2mm}
\begin{equation}
\label{equ_def_block_incidence_matrix}
 \mD: \nodespace \rightarrow \edgespace: \netparams \mapsto \flowvec \mbox{ with } \localflowvec{\edgeidx} = \localparams{\edgeidx_{+}} - \localparams{\edgeidx_{-}} 
 \vspace{-2mm}
 \end{equation} 
with the adjoint (transpose) $\incidencemtx^{T}$ representing another linear map, 
 \begin{equation}
 \label{equ_def_block_incidence_transp_matrix}
 	\incidencemtx^{T}: \edgespace\!\rightarrow\!\nodespace: \flowvec\!\mapsto\!\netparams \mbox{, } \localparams{\nodeidx} = \sum_{\edgeidx \in \edges} 
 	\sum_{\nodeidx = \edgeidx_{+}} \localflowvec{\edgeidx}\!-\!\sum_{\nodeidx = \edgeidx_{-}}  \localflowvec{\edgeidx}.
  \vspace{-4mm}
 \end{equation}

\subsection{Networked Models}
\label{sec_net_models}

A networked model consists of a separate local model for each local dataset $\localdataset{\nodeidx}$. 
Our approach to FL allows for a large variety of design choices for the local models. We only 
require all local models to be parametrized by a common finite-dimensional Euclidean space $\mathbb{R}^{\dimlocalmodel}$. 
This setting covers some widely-used ML models such as (regularized) generalized linear 
models or linear time series models \cite{LocalizedLinReg2019,Brockwell91}. 
However, our setting does not cover non-parametric local models such as decision trees. 

Networked models are parametrized by a map $\netparams \in \nodespace$ that assigns each node $\nodeidx \in \nodes$ 
in the empirical graph $\graph$ a local model parameter vector $\localparams{\nodeidx} \in \mathbb{R}^{\dimlocalmodel}$, \footnote{With 
	a slight abuse of notation we will refer by $\localparams{\nodeidx}$ also to the entire collection of local model parameters.}
$\weights: \nodes \rightarrow \mathbb{R}^{\dimlocalmodel}: \nodeidx \mapsto \weights^{(\nodeidx)}.$
%
We measure the usefulness of a particular choice for the local model parameters $\localparams{\nodeidx}$ 
by some local loss function $\locallossfunc{\nodeidx} {\localparams{\nodeidx}}$. Unless stated otherwise, 
we consider local loss functions that are convex and differentiable. The FL method proposed in 
Section \ref{sec_federatedml} allows for different choices for the local loss functions. These different choices 
might be obtained, in turn, from different combinations of ML models and performance metrics \cite[Ch. 3]{MLBasics}. 

From a computational perspective, our main requirement on the choice for the local loss function 
$\locallossfunc{\nodeidx} {\localparams{\nodeidx}}$ is that it allows for an efficient solving of the 
regularized problem, 
\begin{equation}
\label{equ_def_reguarlizated-local_loss}
	\min_{\weights' \in \mathbb{R}^{\dimlocalmodel}}  \locallossfunc{\nodeidx} {\weights'}+ \regparam \| \weights' - \weights'' \|^{2}_{2}. 
\end{equation} 
The computational complexity of our FL method (see Algorithm \ref{alg1}) depends on the 
ability to efficiently solve \eqref{equ_def_reguarlizated-local_loss} for any given $\regparam \in \mathbb{R}_{+}$ 
and $\weights'' \in \mathbb{R}^{\dimlocalmodel}$. Note that solving \eqref{equ_def_reguarlizated-local_loss} 
is equivalent to evaluating the proximal operator $\proximityop{\locallossfunc{\nodeidx}{\cdot}}{\weights''}{2 \regparam}$. 

Optimization methods for \eqref{equ_def_reguarlizated-local_loss} have been implemented for some widely used 
combinations of local models and loss functions \cite{JMLR:v12:pedregosa11a,PyTorchNeurips}. 
In general, these optimization methods are able to solve \eqref{equ_def_reguarlizated-local_loss} only up to 
some non-zero optimization error. However, our method is robust against such 
optimization errors (see our discussion below Algorithm \ref{alg1}). 


The FL method in Section \ref{sec_federatedml} applies to parametric models that can be trained 
by minimizing a loss function $\locallossfunc{\nodeidx}{\cdot}$ whose proximal operator can be evaluated efficiently. 
Convex functions for which the proximal operator can be computed efficiently are sometimes referred 
to as ``proximable'' or ``simple'' \cite{Condat2013}. Note that the shape of the loss function typically 
depends on both, the choice for the local model and the metric used to measure prediction errors \cite[Ch.\ 4]{MLBasics}. 

Our focus is on applications where the local loss functions $\locallossfunc{\nodeidx}{\localparams{\nodeidx}}$ 
do not carry sufficient statistical power to guide the learning of model parameters $\localparams{\nodeidx}$. As a 
case in point, consider a local dataset $\localdataset{\nodeidx}$ of the form \eqref{equ_def_local_dataset_plain}, with 
feature vectors $\featurevec^{(\sampleidx)} \in \mathbb{R}^{\dimlocalmodel}$ with $\localsamplesize{\nodeidx} \ll \dimlocalmodel$. 
We would like to learn the parameter vector $\localparams{\nodeidx}$ of a linear hypothesis $h(\featurevec) =\big( \localparams{\nodeidx} \big)^{T}   \featurevec$. Linear regression methods learn the parameter vector by minimizing the average  
squared error loss $\locallossfunc{\nodeidx}{\localparams{\nodeidx}}=(1/\samplesize_{\nodeidx}) \sum_{\sampleidx=1}^{\localsamplesize{\nodeidx}} \big( \truelabel^{(\sampleidx)} - \big( \localparams{\nodeidx} \big)^{T} \featurevec^{(\sampleidx)} \big)^{2}$. 
However, for $\localsamplesize{\nodeidx} \ll \dimlocalmodel$ (the ``high-dimensional'' regime) the minimum of 
$\locallossfunc{\nodeidx}{\cdot}$ is not unique and might also provide a poor hypothesis incurring large 
prediction errors on data points outside $\localdataset{\nodeidx}$
\cite[Ch. 6]{MLBasics}. Training linear models in the high-dimensional regime requires regularization such as in 
ridge regression or Lasso \cite{hastie01statisticallearning}. 

The main theme of this paper is to use the empirical graph $\graph$ to regularize the learning 
of local model parameters by requiring them not to vary too much over edges with large weights. 
Section \ref{sec_nLasso} introduces the concept of GTV as a quantitative measure for the variation 
of local parameter vectors. Regularization by requiring a small GTV is an instance of the 
smoothness assumption used in semi-supervised learning \cite{SemiSupervisedBook}. 

Our analysis of GTV minimization in Section \ref{sec_when_does_it_work} relates its underlying 
smoothness assumption to a clustering assumption. Section \ref{sec_clustering_assumption} formalizes 
this clustering assumption which requires local model parameters to be constant over subsets (clusters) 
of nodes in the empirical graph. Theorem \ref{thm_main_result} then offers precise conditions on the 
empirical graph and local loss functions such that GTV minimization successfully recovers the clusters 
of nodes.

\subsection{Clustering Assumption}
\label{sec_clustering_assumption}
Consider networked data with empirical graph $\graph = \pair{\nodes}{\edges}$. 
Each node $\nodeidx$ in the graph carries a local dataset $\localdataset{\nodeidx}$ and a local model with parameters 
$\localparams{\nodeidx}$. Our goal is to learn the local model parameters $\localparams{\nodeidx}$ for each node $\nodeidx \in \nodes$. 
The key assumption of clustered FL is that the local datasets form clusters with local datasets in the same 
cluster having similar statistical properties \cite{NIPS2008_fccb3cdc}. Given a 
cluster $\clustergeneric$ of nodes, it seems natural to pool their local datasets or, equivalently, add their 
local functions to learn a cluster-specific parameter vector 
\begin{equation} 
	\label{equ_def_opt_cluster}
	\overline{\weights}^{(\clustergeneric)} =  \argmin_{\vv \in \mathbb{R}^{\dimlocalmodel}} \clusterobj{\clustergeneric}{\vv} \mbox{ with } \clusterobj{\clustergeneric}{\vv} \defeq \sum_{\nodeidx \in \clustergeneric}\locallossfunc{\nodeidx}{\vv}.
 \vspace{-2mm}
\end{equation}
Note that \eqref{equ_def_opt_cluster} cannot be implemented in practice since we typically do not know 
the cluster $\clustergeneric$. The main analytical contribution of this paper is an upper bound for the deviation 
between solutions of GTV minimization and the cluster-wise (but impractical) learning problem \eqref{equ_def_opt_cluster}. 
This bound characterizes the statistical properties of FL algorithms that are obtained by applying 
optimization techniques for solving GTV minimization (see Section \ref{sec_federatedml}). 

The solution $\overline{\weights}^{(\clustergeneric)}$ of \eqref{equ_def_opt_cluster} minimizes the 
aggregation (sum) of all local loss functions that belong to the same cluster $\clustergeneric \subseteq \nodes$. 
Thus, $\overline{\weights}^{(\clustergeneric)}$ is the optimal model parameter for a training set obtained by 
pooling all local datasets that belong to the cluster $\clustergeneric$. As indicated by our notation, we 
tacitly assume that the solution to \eqref{equ_def_opt_cluster} is unique. The uniqueness of the 
solution in \eqref{equ_def_opt_cluster} will be ensured by Assumption \ref{asspt_FIM_lower_bound} below. 

We now make our assumption of datasets in the same cluster ``having similar statistical properties'' precise. 
In particular, we require the local loss functions at nodes $\nodeidx \in \clustergeneric$ in the same cluster $\clustergeneric$ 
to have nearby minimizers. Thus, we require a small deviation $\normgeneric{\vv^{(\nodeidx)} - \overline{\weights}^{(\clustergeneric)}}{}$ 
between the minimizer $\vv^{(\nodeidx)}$ of $\locallossfunc{\nodeidx}{\cdot}$ and the corresponding 
cluster-wise optimal parameter vector \eqref{equ_def_opt_cluster}. This requirement is, for differentiable and convex loss 
functions, equivalent to requiring a small gradient of the local loss functions at the cluster-wise 
minimizer \eqref{equ_def_opt_cluster}. It will be convenient for our analysis to formulate this 
requirement by upper bounding the dual norm $\normgeneric{\nabla \locallossfunc{\nodeidx}{\overline{\weights}^{(\clustergeneric)}} }{*}$ 
of the local loss gradient. 
\begin{assumption}[Clustering]
\label{asspt_weights_clustered}
Consider some networked data represented by an empirical graph $\graph$ whose nodes 
carry local loss functions $\locallossfunc{\nodeidx}{\vv}$, for $\nodeidx \in \nodes$. 
There is a partition of the nodes $\nodes$ into disjoint clusters 
\begin{equation}
\label{equ_def_parition_asspt}
\begin{aligned}
 &{\partition\!=\!\{\cluster{1},\ldots,\cluster{\nrcluster} \} \mbox{ with } \cluster{\clusteridx} \cap \cluster{\clusteridx'} = \emptyset ,} \\
 &{ \quad \mbox{ for } \clusteridx\!\neq\!\clusteridx' \mbox{ and } \nodes = \cluster{1} \cup \ldots \cup \cluster{\nrcluster}.} 
\end{aligned}
 \vspace{-2mm}
\end{equation} 
Moreover, for each cluster $\cluster{\clusteridx}\in \partition$, 
\begin{equation}
\label{equ_upper_bound_norm_gradient}
\normgeneric{\nabla \locallossfunc{\nodeidx}{\overline{\weights}^{(\clusteridx)}} }{*}  \leq \clusteropterr{\nodeidx} \mbox{ for all } \nodeidx \in \cluster{\clusteridx}. 
 \vspace{-2mm}
\end{equation}
Here, $\overline{\weights}^{(\clusteridx)}\!\in\!\mathbb{R}^{\dimlocalmodel}$ denotes the solution of the cluster-wise minimization
\eqref{equ_def_opt_cluster} for cluster $\cluster{\clusteridx}$. 
\end{assumption} 
The clustering assumption requires the dual norm $\normgeneric{\nabla \locallossfunc{\nodeidx}{\ \overline{\weights}^{(\clusteridx)}}}{*}$ 
to be bounded by a constant $\clusteropterr{\nodeidx}$ for each nodes $\nodeidx \in \nodes $ in the empirical graph. 
We can interpret this norm as a measure for the discrepancy between the cluster-wise minimizer $\overline{\weights}^{(\clusteridx)}$ 
(see \eqref{equ_def_opt_cluster}) and the minimizers of the local loss functions $\locallossfunc{\nodeidx}{\overline{\weights}^{(\clusteridx)}}$ 
for each node $\nodeidx \in \cluster{\clusteridx}$. 

Section \ref{sec_federatedml} uses the cluster-wise minimization \eqref{equ_def_opt_cluster} as a 
theoretical device to analyze the solutions of GTV minimization \eqref{equ_gtvmin}. It is important to 
note that \eqref{equ_def_opt_cluster} does not inform a practical FL method as it requires knowledge 
of the clusters in the partition \eqref{equ_def_parition_asspt}. It might be unrealistic to assume perfect 
knowledge of the partition \eqref{equ_def_parition_asspt} postulated by Assumption \ref{asspt_weights_clustered}. 
Rather, we show that GTV minimization is able to recover this partition using solely the edges of the empirical 
graph $\graph$. 

Section \ref{sec_nLasso} formulates FL as GTV minimization which is an instance of RERM. GTV minimization is 
enforces ``clusteredness'' of local model parameters by requiring a small variation across edges in the 
empirical graph. Under Assumption \ref{asspt_weights_clustered}, this regularization strategy will be useful if many (large weight) 
edges connect nodes in the same cluster but only few (small weight) edges connect nodes in different clusters. 
Section \ref{sec_when_does_it_work} presents a precise condition on the network structure such that GTV 
minimization succeeds in capturing the true underlying cluster structure of the local loss functions. 



The analysis of the FL method proposed in Section \ref{sec_when_does_it_work} requires the 
local loss functions to be convex and smooth. Moreover, we require their (partial) sums in the cluster-wise 
objective $\clusterobj{\clusteridx}{\cdot}$ \eqref{equ_def_opt_cluster} to be strongly 
convex \cite[Exercise 1.9]{BertCvxAnalOpt}. 
\begin{assumption}[Convexity and Smoothness]
	\label{asspt_FIM_lower_bound}
For each node $\nodeidx \in \nodes$, the local loss function $\locallossfunc{\nodeidx}{\localparams{\nodeidx}}$ is convex 
and differentiable with gradient satisfying 
\begin{equation}
	\label{equ_def_cond_locallipsch}
\normgeneric{ \nabla \locallossfunc{\nodeidx}{\vv'} - \nabla \locallossfunc{\nodeidx}{\vv}}{*} \leq \locallipsch{\nodeidx} \normgeneric{\vv'- \vv}{}. 
\end{equation} 
For each cluster $\cluster{\clusteridx} \in \partition$ in the partition \eqref{equ_def_parition_asspt}, 
the cluster-wise objective $\clusterobj{\clusteridx}{\cdot}$ \eqref{equ_def_opt_cluster} is strongly convex, 
\begin{equation} 
\label{equ_strong_convexit}
\begin{aligned}
    \clusterobj{\clusteridx}{\vv'} \geq \clusterobj{\clusteridx}{\vv}  +
 &{\big( \vv' \!-\! \vv \big)^{T} \partial \clusterobj{\clusteridx}{\vv} + (\strongconvparam{\clusteridx}/2) \norm{\vv' \!-\! \vv}^{2},} \\
 &{\mbox{ for any } \vv', \vv \in \mathbb{R}^{\dimlocalmodel}.} 
\end{aligned}
\end{equation} 
Here, $\strongconvparam{\clusteridx}> 0$ is a positive constant that might be different for different clusters $\cluster{\clusteridx}$. 
The norm $\normgeneric{\cdot}{}$ in \eqref{equ_def_cond_locallipsch}, \eqref{equ_strong_convexit} is the dual of 
the norm $\normgeneric{\cdot}{*}$ used in \eqref{equ_def_cond_locallipsch} and \eqref{equ_upper_bound_norm_gradient}. 
\end{assumption}  
Assumption \ref{asspt_FIM_lower_bound} is rather standard in FL literature \cite{Ghosh2020,Nassif2020}. 
In particular, Assumption \ref{asspt_FIM_lower_bound} is satisfied by many important ML models \cite{LocalizedLinReg2019}. 
Assumption \ref{asspt_FIM_lower_bound} does not hold for many deep learning models that result in 
non-convex loss functions \cite{pmlr-v54-mcmahan17a}. Nevertheless, we expect our theoretical analysis 
to provide useful insight also for settings where Assumption \ref{asspt_FIM_lower_bound} is violated.  


We emphasize that Assumption \ref{asspt_FIM_lower_bound} does not require strong convexity for each local loss 
function $\locallossfunc{\nodeidx}{\cdot}$ individually. Rather, it only requires their cluster-wise sums \eqref{equ_def_opt_cluster} 
to be strongly convex. We also allow for trivial local loss functions that are constant and might represent 
the inaccessibility of local datasets due to  privacy-constraints or lack of computational resources. 
The FL method in Section \ref{sec_federatedml} can tolerate the presence of non-informative local loss functions by exploiting the 
similarities between local datasets as reflected by the edges in the empirical graph $\graph$. 


\section{Generalized Total Variation Minimization}
\label{sec_nLasso}

The clustering Assumption \ref{asspt_weights_clustered} suggests to learn the model parameters $\localparams{\nodeidx}$ via 
the cluster-wise optimization \eqref{equ_def_opt_cluster}. For each cluster $\cluster{\clusteridx}$ in the partition \eqref{equ_def_parition_asspt}, 
we use the solution of \eqref{equ_def_opt_cluster} as the local model parameters at all nodes $\nodeidx \in \cluster{\clusteridx}$. 
However, this approach is not practical since the partition \eqref{equ_def_parition_asspt} is typically unknown and therefore 
we cannot directly implement \eqref{equ_def_opt_cluster}. Instead, we use the empirical graph $\graph$ to penalize variations 
of local model parameter $\localparams{\nodeidx}$ over well-connected nodes (see Section \ref{sec_primal_problem}). 

We hope that penalizing their variation (over the edges in the empirical graph) favours local model parameters 
that are approximately constant over nodes in the same cluster \eqref{equ_def_parition_asspt}. For this approach 
to be successful, the nodes in the same cluster must be densely connected by many edges (or large weight) in 
the empirical graph, which should have only few edges (with low weight) between nodes in different clusters. 
We will make this informal assumption precise in Section \ref{sec_when_does_it_work}. For now, we use the 
informal clustering assumption to motivate GTV as a useful regularizer for learning the local model parameters. 

If the cluster structure of $\graph$ is reflected by a high density of edges within clusters and  
few boundary edges between them, it seems reasonable to require a small variation of local 
parameter vectors $\localparams{\nodeidx}$ across edges. We measure the variation of local parameter 
vectors $\netparams \in \nodespace$ across the edges in $\graph$ via the variation $\flowvec: \edgeidx \in \edges \mapsto \localflowvec{\edgeidx} \defeq \localparams{\edgeidx_{+}} -  \localparams{\edgeidx_{-}}$. 
Using the block-incidence matrix \eqref{equ_def_block_incidence_matrix} we can express the variation of $\netparams$
more compactly as $\flowvec = \incidencemtx \netparams$. 

A quantitative measure for the variation of local model parameters $\netparams$ is the GTV
\begin{equation} \label{eq:5}
\begin{aligned}
    &{\normgeneric{\netparams}{\rm GTV}   \defeq \sum_{\edge{\nodeidx}{\nodeidx'} \in \edges} \edgeweight_{\nodeidx,\nodeidx'} \phi\big(\localparams{\nodeidx'} -\localparams{\nodeidx}  \big) }
\end{aligned}
\vspace{-2mm}
\end{equation}
 with some convex penalty function $\gtvpenalty(\cdot): \mathbb{R}^{\dimlocalmodel} \rightarrow \mathbb{R}$.
We also define the GTV for a subset of edges $\mathcal{S} \subseteq \edges$ as 
\begin{equation} \label{eq:def_GTV_subset}
	\| \netparams \|_{\mathcal{S}}  \defeq \sum_{\edge{\nodeidx}{\nodeidx'} \in \mathcal{S}} \edgeweight_{\nodeidx,\nodeidx'} \gtvpenalty\big(\localparams{\nodeidx'} -\localparams{\nodeidx}  \big).
 \vspace{-2mm}
\end{equation} 
The GTV \eqref{eq:5} provides a whole ensemble of variation measures. This ensemble is parametrized by a 
penalty function $\gtvpenalty(\vv) \in \mathbb{R}$ which we tacitly assume to be convex. The penalty 
function $\gtvpenalty(\cdot)$ is an important design choice that determines the computational and 
statistical properties of the resulting GTV minimization problem (see Section \ref{sec_federatedml} 
and Section \ref{sec_when_does_it_work}). Two popular choices are $\gtvpenalty(\vv) \defeq \| \vv\|_{2}$, 
which is used by nLasso \cite{NetworkLasso}, and $\gtvpenalty(\vv) \defeq (1/2)\| \vv\|^{2}_{2}$ 
which is used by the method {\rm MOCHA} \cite{Smith2017}. Another recent FL method uses 
the choice $\gtvpenalty(\vv) \defeq \| \vv \|_{1}$ \cite{Sarchesh2021}. 

Different choices for the penalty function offer different trade-offs between computational complexity and 
statistical properties of the resulting FL algorithms. As a case in point, the penalty $\gtvpenalty(\flowvec)= \|\flowvec\|_{2}$ (used in nLasso) 
is computationally more challenging than the penalty $\gtvpenalty(\vu)= (1/2) \|\flowvec\|_{2}^{2}$ (used in {\rm MOCHA} \cite{Smith2017}). 
On the other hand, nLasso is more accurate in learning models for data with specific 
network structures (such as chains) that are challenging for GTV minimization method using the 
smooth penalty $\gtvpenalty(\vv) \defeq (1/2)\| \vv\|^{2}_{2}$ \cite{WhenIsNLASSO}.  

Section \ref{sec_federatedml} designs FL methods whose main computational steps include the computation of 
the proximal operator $\proximityop{\phi^{*}}{\cdot}{\rho}$ for the convex conjugate $\gtvpenalty^{*}$ of the 
GTV penalty function $\gtvpenalty(\cdot)$. Thus, for these methods to be computationally tractable we must 
choose $\gtvpenalty(\cdot)$ such that the proximal operator $\proximityop{\phi^{*}}{\cdot}{\rho}$ can be computed (evaluated) efficiently.\footnote{The difficulty 
	of computing the proximal operator $\proximityop{\phi^{*}}{\vu}{\rho}$ is essentially the same as that 
	of computing the proximal operator $\proximityop{\phi}{\vu}{\rho}$. Indeed, these two proximal operators 
	are related via the identity $\vu = \proximityop{\phi}{\vu}{1}+ \proximityop{\phi^{*}}{\vu}{1}$  \cite{ProximalMethods}.}

\subsection{The Primal Problem}
\label{sec_primal_problem}
%
%

GTV minimization learns the local model parameters $\localparams{\nodeidx}$ by balancing (the sum of) 
local loss functions and GTV \eqref{eq:5}, 
\begin{equation} 
\label{equ_gtvmin}
\hspace*{-2mm}\widehat{\netparams}\!\in\!\underset{\netparams \in \nodespace}{\mathrm{arg \ min}}\sum_{\nodeidx \in \nodes} \hspace*{-1mm}\locallossfunc{\nodeidx}{\localparams{\nodeidx}}\!+\!\regparam \|\netparams\|_{\rm GTV} \mbox{ with some } \regparam\!>\!0. 
\end{equation}
The regularization parameter $\regparam\!>\!0$ in \eqref{equ_gtvmin} steers the preference for learning 
parameter vectors $\localparams{\nodeidx}$ with small GTV versus incurring small local loss 
$\sum_{\nodeidx \in \nodes} \locallossfunc{\nodeidx}{\localparams{\nodeidx}}$. The choice of $\regparam$ 
can be guided by cross validation \cite{hastie01statisticallearning} or by our analysis of the solutions of \eqref{equ_gtvmin} 
in Section \ref{sec_when_does_it_work}. 

Increasing the value of $\regparam$ results in the solutions of \eqref{equ_gtvmin} becoming 
increasingly clustered, with local model parameters $\estlocalparams{\nodeidx}$ being constant over (increasingly) large 
subsets of nodes. Choosing $\regparam$ larger than some critical value, that depends on the shape of 
the local loss functions and the edges of $\graph$, results in $\estlocalparams{\nodeidx}$ 
being constant over all nodes $\nodeidx \in \nodes$. Section \ref{sec_when_does_it_work} offers precise 
conditions on the local loss functions and the empirical graph such that the solutions of \eqref{equ_gtvmin} 
capture the (unknown) underlying partition \eqref{equ_def_parition_asspt}. 

The computational and statistical properties of GTV minimization depend on the design choices for 
local models (which determine the shape of local loss functions) and GTV penalty function (see \eqref{eq:5}). 
Section \ref{sec_nLasso} applies the primal-dual method \cite[Alg. 6]{pock_chambolle_2016} to compute 
(approximate) solutions of \eqref{equ_gtvmin} when $\locallossfunc{\nodeidx}{\cdot}$ and 
$\gtvpenalty(\cdot)$ allow for an efficient computation of their proximal operators (see \eqref{equ_def_proximity_operator}). 
For convex $\locallossfunc{\nodeidx}{\cdot}$ and $\gtvpenalty(\cdot)$ being a norm, we will 
characterize the solutions of \eqref{equ_gtvmin} in our main result Theorem \ref{thm_main_result}. 

GTV minimization \eqref{equ_gtvmin} is an instance of RERM, using the scaled GTV $\regparam \|\netparams\|_{\rm GTV}$ as 
regularizer. The empirical risk incurred by the local model parameters $\netparams \in \nodespace$ 
is measured by the sum of the local loss functions $\sum_{\nodeidx \in \nodes} \locallossfunc{\nodeidx}{\localparams{\nodeidx}}$. 
GTV minimization \eqref{equ_gtvmin} unifies and considerably extends some well-known methods for 
distributed optimization and learning. In particular, the nLasso \cite{NetworkLasso} is obtained from 
\eqref{equ_gtvmin} for the choice $\gtvpenalty(\vv) \defeq \normgeneric{\vv}{2}$. The {\rm MOCHA} method \cite{Smith2017} is obtained 
from \eqref{equ_gtvmin} for the choice $\gtvpenalty(\vv) \defeq (1/2) \normgeneric{\vv}{2}^{2}$. Another 
special case of \eqref{equ_gtvmin}, obtained for the choice $\gtvpenalty(\vv) \defeq \normgeneric{\vv}{1}$, has been 
studied recently \cite{Sarchesh2021}. 

\subsection{The Dual Problem} 
\label{sec_dual_problem_GTV_Min}

The solutions of GTV minimization \eqref{equ_gtvmin} can be conveniently characterized and computed 
by introducing another optimization problem that is dual (in a sense that we make precise promptly) 
to \eqref{equ_gtvmin}. We obtain this dual problem by using the convex conjugate $h^{*}$ (see \eqref{equ_def_conv_conj}) of 
a convex function $h(\vx)$ (see \cite{BoydConvexBook}). The convex conjugate offers an alternative (or dual) 
representation of a convex function $h(\vx)$ via 
\begin{equation}
\label{equ_def_cvx_conj}
h(\vx) \defeq \sup\limits_{\vz} \vx^{T} \vz - h^{*}(\vz).
\vspace{-1mm}
\end{equation}
This alternative (or dual) representation of convex functions lends naturally to a 
dual problem for GTV minimization \eqref{equ_gtvmin} . 

While the domain of GTV minimization \eqref{equ_gtvmin} are the local model parameters $\localparams{\nodeidx} \in \mathbb{R}^{\dimlocalmodel}$, 
for all nodes $\nodeidx \in \nodes$, the domain of the dual problem will be flow vectors $\localflowvec{\edgeidx} \in \mathbb{R}^{\dimlocalmodel}$, 
for each $\edgeidx \in \edges$. To formulate the dual problem of GTV minimization \eqref{equ_gtvmin}, 
we first rewrite it more compactly as (see \eqref{equ_def_block_incidence_matrix} and \eqref{eq:5}),   
\begin{equation} \label{equ_nLasso_compactly}
\begin{aligned}
    &{\widehat{\netparams} \in \underset{{\netparams} \in \nodespace}{\mathrm{arg \ min}}\ f(\netparams)\!+\!g(\incidencemtx \netparams )}\\ 
    &{\mbox{with } 
	f(\netparams)\defeq\sum_{\nodeidx \in \nodes} \locallossfunc{\nodeidx}{\localparams{\nodeidx}}}, \\
    &{\mbox{and } g( \flowvec ) \!\defeq\!\regparam \sum_{\edgeidx \in \edges} \edgeweight_{\edgeidx} \gtvpenalty \big( \localflowvec{\edgeidx} \big).} 
\end{aligned}
\vspace{-1mm}
\end{equation}
The objective function in \eqref{equ_nLasso_compactly} is the sum of two convex functions $f(\netparams)$ and $g(\flowvec )$ whose 
arguments are coupled as $\flowvec =\incidencemtx \netparams$. 
Representing $f(\netparams)$ and $g(\flowvec)$ via their convex conjugates (see \eqref{equ_def_cvx_conj}) 
and interchanging the minimization with the maximization (``taking the supremum'') in \eqref{equ_def_cvx_conj} results in 
the dual problem
\begin{equation}
\label{equ_dual_nLasso}
\max_{\flowvec \in \edgespace} -g^{*}(\flowvec) - f^{*}(-\incidencemtx^{T} \flowvec).
\vspace{-1mm}
\end{equation}
The domain of the dual problem \eqref{equ_dual_nLasso} is the space $\edgespace$ 
of maps $\flowvec: \edges \rightarrow \mathbb{R}^{\dimlocalmodel}$ that assign a flow 
vector $\localflowvec{\edgeidx}$ to each edge $\edgeidx \in \edges$ of the empirical graph $\graph$. 

The objective function of the dual problem \eqref{equ_dual_nLasso} is composed of the 
convex conjugates 
\begin{align}
	\label{equ_conv_conjugate_g_dual_proof}
	g^{*}(\flowvec) &\defeq \sup_{\vz \in \edgespace} \sum_{\edgeidx \in \edges} \big( \localflowvec{\edgeidx} \big)^{T}\vz^{(\edgeidx)} - g(\vz)   \nonumber \\ 
	& \stackrel{\eqref{equ_nLasso_compactly}}{=} \sup_{\vz \in \edgespace}  \sum_{\edgeidx \in \edges} \big( \localflowvec{\edgeidx} \big)^{T}\vz^{(\edgeidx)}\!-\! \regparam \edgeweight_{\edgeidx} \phi \big( \vz^{(\edgeidx)}\big) \nonumber \\ 
	& =  \sum_{\edgeidx \in \edges} \regparam  \edgeweight_{\edgeidx} \gtvpenalty^{*}\big( \localflowvec{\edgeidx}/(\regparam \edgeweight_{\edgeidx}) \big), 
 \vspace{-3mm}
\end{align}
and
\begin{align}
	\label{equ_dual_f_fun}
	f^{*}(\netparams) & \!\defeq\! \sup_{\vz \in \nodespace} \sum_{\nodeidx \in \nodes} \big( \localparams{\nodeidx} \big)^{T}\vz^{(\nodeidx)}- f(\vz) \nonumber \\  
	& \stackrel{\eqref{equ_nLasso_compactly}}{=} \sup_{\vz \in \nodespace} \sum_{\nodeidx \in \nodes} \big( \localparams{\nodeidx} \big)^{T}\vz^{(\nodeidx)} - \sum_{\nodeidx \in \nodes} \locallossfunc{\nodeidx}{\vz^{(\nodeidx)}} \nonumber \\ 
	& =\sum_{\nodeidx \in \nodes}  \conjlocallossfunc{\nodeidx}{\localparams{\nodeidx}} .
 \vspace{-3mm}
\end{align}
Note that the convex conjugate in \eqref{equ_conv_conjugate_g_dual_proof} is constituted by 
the values of the convex conjugate of the penalty function $\gtvpenalty$, evaluated at the flow 
vectors $\localflowvec{\edgeidx}$ across each edge $\edgeidx \in \edges$. Similarly, the convex 
conjugate in \eqref{equ_dual_f_fun} is constituted by the convex conjugates of the local loss functions, 
evaluated at the local model parameters. 

The duality between \eqref{equ_gtvmin} and \eqref{equ_dual_nLasso} is made precise in \cite[Ch.\ 31]{RockafellarBook} 
(see also \cite[Sec. 3.5]{pock_chambolle_2016}). First, the optimal values of both problems coincide \cite[Cor.\ 31.2.1]{RockafellarBook},
\begin{equation}
	\label{equ_equal_primal_dual}
	\min_{\netparams \in \nodespace} \hspace*{0mm} f(\netparams)\!+\!g(\incidencemtx \netparams ) \!=\! \hspace*{0mm}\max_{\flowvec \in \edgespace} -g^{*}(\flowvec)\!-\!f^{*}(-\incidencemtx^{T} \flowvec).
 \vspace{-2mm}
\end{equation}
Moreover, a necessary and sufficient condition for $\widehat{\netparams}$ to solve \eqref{equ_gtvmin}
and $\widehat{\flowvec}$ to solve \eqref{equ_dual_nLasso} is \cite[Thm. 31.3]{RockafellarBook} (see also \cite[Ch. 7]{BertCvxAnalOpt}) 
\begin{equation}
	\label{equ_opt_condition_Rocka_KKT}
	-\incidencemtx^{T} \widehat{\flowvec} \in \partial f(\widehat{\netparams}) \mbox{ , and } \incidencemtx \widehat{\netparams} \in  \partial  g^{*}(\widehat{\flowvec}).
\end{equation}

The identity \eqref{equ_equal_primal_dual} (which is an instance of ``strong duality'' \cite[Ch. 5]{BoydConvexBook})
allows to bound the sub-optimality of some given local model parameters $\localparams{\nodeidx}$. 
Indeed, for any given dual variable $\widehat{\flowvec} \in \edgespace$, the objective function 
value $-g^{*}(\widehat{\flowvec})\!-\!f^{*}(-\incidencemtx^{T} \widehat{\flowvec})$ is a 
lower bound for the optimal value of \eqref{equ_gtvmin}. Such a bound on the sub-optimality of given 
local model parameters can be useful for defining a stopping criterion for iterative optimization methods (see Section \ref{sec_federatedml}) 

It is instructive to rewrite the dual problem \eqref{equ_dual_nLasso} and the optimality 
condition \eqref{equ_opt_condition_Rocka_KKT} in terms of the local parameter vectors 
$\localparams{\nodeidx}$, for each node $\nodeidx \in \nodes$, and the local flow 
vectors $\localflowvec{\edgeidx}$, for each $\edgeidx \in \edges$. Indeed, the final 
expressions in \eqref{equ_dual_f_fun} and \eqref{equ_conv_conjugate_g_dual_proof} 
allow us to rewrite the dual problem \eqref{equ_dual_nLasso} as 
\begin{equation} \label{equ_def_duality_nLasso_edge_node}
    \begin{aligned}
        \max_{\flowvec \in \edgespace} &{- \sum_{\nodeidx \in \nodes} \conjlocallossfunc{\nodeidx}{\localparams{\nodeidx}} - 
        \regparam \sum_{\edgeidx \in \edges}   \edgeweight_{\edgeidx}  \gtvpenalty^{*}\big( \localflowvec{\edgeidx} /  ( \regparam  \edgeweight_{\edgeidx}) \big) } \\ 
        &{ \mbox{ subject to } - \localparams{\nodeidx} = \sum_{\edgeidx \in \edges} \sum_{\nodeidx = \edgeidx_{+}} \localflowvec{\edgeidx} - \sum_{\nodeidx = \edgeidx_{-}}  \localflowvec{\edgeidx}} \\
        &{ \mbox{ for all nodes } \nodeidx \in \nodes.} 
    \end{aligned}
    \vspace{-2mm}
\end{equation}

Using the block-incidence matrix \eqref{equ_def_block_incidence_matrix} and its 
transpose \eqref{equ_def_block_incidence_transp_matrix}, we can also rewrite the optimality 
condition \eqref{equ_opt_condition_Rocka_KKT} as 
\begin{align}
        \sum_{\edgeidx \in \edges} 
      	\sum_{\nodeidx = \edgeidx_{+}} & { \optlocalflowvec{\edgeidx} - \sum_{\nodeidx = \edgeidx_{-}}  \optlocalflowvec{\edgeidx}  = - \nabla \locallossfunc{\nodeidx}{\estlocalparams{\nodeidx}}} {\mbox{ for each } \nodeidx\!\in\!\nodes} \nonumber \\ 
       \estlocalparams{\edgeidx_{+}} - & \hspace*{-2mm} { \estlocalparams{\edgeidx_{-}}  \!\in\! ( \regparam  \edgeweight_{\edgeidx}) \partial \gtvpenalty^{*}( \optlocalflowvec{\edgeidx} /  ( \regparam  \edgeweight_{\edgeidx}) )}  \mbox{ for each } \edgeidx\!\in\!\edges.  \label{equ_opt_condition_node_edge}
    \vspace{-2mm}
\end{align}


Let us now specialize the dual problem \eqref{equ_def_duality_nLasso_edge_node} for a 
GTV penalty function $\gtvpenalty$ being a norm $\| \cdot \|$ on $\mathbb{R}^{\dimlocalmodel}$ \cite{Golub1980}. 
For such a penalty function $\gtvpenalty\big(\localflowvec{\edgeidx}\big) = \normgeneric{\localflowvec{\edgeidx}}{}$, 
the convex conjugate is the indicator of the dual-norm ball \cite[Example 3.26]{BoydConvexBook}, 
\begin{equation} 
	\label{equ_dual_gtv_pen_conj}
\gtvpenalty^{*} \big( \localflowvec{\edgeidx}  \big) = \begin{cases} 0 & \mbox{ for } \big\| \localflowvec{\edgeidx} \big\|_{*} \leq 1  \\ 
	\infty & \mbox{ else.} \end{cases}
 \vspace{-2mm}
\end{equation}
Inserting \eqref{equ_dual_gtv_pen_conj} into \eqref{equ_def_duality_nLasso_edge_node}, 
\begin{align}
	\label{equ_def_duality_nLasso_edge_node_norm_gtv_pen}
	\max_{\flowvec \in \edgespace}  - \sum_{\nodeidx \in \nodes} \conjlocallossfunc{\nodeidx} {\localparams{\nodeidx}} &  \nonumber \\ 
	 \mbox{ subject to } - \localparams{\nodeidx}  & = 
	\sum_{\edgeidx \in \edges} \sum_{\nodeidx = \edgeidx_{+}} \localflowvec{\edgeidx} - \sum_{\edgeidx \in \edges: \nodeidx = \edgeidx_{-}}  \localflowvec{\edgeidx} \nonumber \\
 & \quad \quad \mbox{ for each } \nodeidx\!\in\!\nodes \nonumber \\
	 \| \localflowvec{\edgeidx} \|_{*} &    \leq  \regparam  \edgeweight_{\edgeidx} \mbox{ for each } \edgeidx\!\in\!\edges. 
   \vspace{-2mm}
\end{align}
Thus, when the GTV penalty function is a norm $\gtvpenalty(\cdot) = \| \cdot \|$, the optimality condition \eqref{equ_opt_condition_node_edge} 
becomes (see \cite[p. 215]{RockafellarBook}) 
\begin{align} 
\label{equ_opt_condition_node_edge_norm}
    \hspace{-5mm} \sum_{\edgeidx \in \edges:\nodeidx\!=\!\edgeidx_{+}} \hspace*{-3mm}\optlocalflowvec{\edgeidx} -  \sum_{\edgeidx \in \edges:\nodeidx\!=\!\edgeidx_{-}}  & \optlocalflowvec{\edgeidx}  = - \nabla \locallossfunc{\nodeidx}{\estlocalparams{\nodeidx}} \mbox{ for each } \nodeidx\!\in \!\nodes \nonumber, \\ 
    \quad \quad \| \optlocalflowvec{\edgeidx} \|_{*} \!\leq\!\regparam  \edgeweight_{\edgeidx}  & \mbox{ for each } \edgeidx\!\in\!\edges,  \\ 
    \estlocalparams{\edgeidx_{+}} = \estlocalparams{\edgeidx_{-}} 	\mbox{ for} &  \mbox{ each } \edgeidx \in \edges \mbox{ with }  \| \optlocalflowvec{\edgeidx} \|_{*} \!<\!\regparam  \edgeweight_{\edgeidx}.  \nonumber
 \vspace{-2mm}
\end{align} 

\subsection{Interpretations} 
\label{sec_interpreations} 
We next discuss some useful interpretations of GTV minimization \eqref{equ_gtvmin}. 
These interpretations relate GTV minimization with some well-known ML principles \cite{MLBasics,hastie01statisticallearning,SemiSupervisedBook} 
and network optimization \cite{RockafellarBook,BertsekasNetworkOpt}. 

\emph{Generalization of Graph Clustering.} Our main result Theorem \ref{thm_main_result} bounds the deviation 
between the solutions of GTV minimization \eqref{equ_gtvmin} and local model parameters that 
are constant over well-connected (see Definition \ref{equ_def_well_connected_cluster}) subsets of nodes in the 
empirical graph. Thus, we can interpret GTV minimization \eqref{equ_gtvmin} as a method for clustering the nodes 
in the empirical graph. In contrast to basic graph clustering methods \cite{Luxburg2007,Ng2001}, GTV minimization 
\eqref{equ_gtvmin} does not solely depend on the edge connectivity of the empirical graph but also on the shape of 
the local loss functions $\locallossfunc{\nodeidx}{\localparams{\nodeidx}}$ at its nodes. In particular, our method 
might deliver different clusters for two empirical graphs having identical edge sets but carrying different local loss 
functions at its nodes.

\emph{Generalization of Convex Clustering.} GTV minimization generalizes convex clustering \cite{JMLR:v22:18-694} 
which is a special case of \eqref{equ_gtvmin}, obtained for the specific local loss functions 
\begin{equation} 
	\label{equ_lossfunc_cvxclustering}
\locallossfunc{\nodeidx}{\localparams{\nodeidx}} = \norm{\localparams{\nodeidx} - \va^{(\nodeidx)}}^{2}, \mbox{ for all nodes } \nodeidx \in \nodes
 \vspace{-1mm}
\end{equation} 
and GTV penalty $\gtvpenalty(\vu) = \normgeneric{\vu}{p}$ being a $p$-norm $ \normgeneric{\vu}{p} \defeq \big( \sum_{j=1}^{\dimlocalmodel} |u_{j}|^{p} \big)^{1/p}$ with some $p \geq 1$. The vectors $\va^{(\nodeidx)}$ in \eqref{equ_lossfunc_cvxclustering} 
are the observations that we wish to cluster. In its most basic form, convex clustering uses a fully 
connected empirical graph in \eqref{equ_gtvmin} with uniform edge weights \cite{Pelckmans2005}. 
However, there is also recent work that studies a more general convex clustering model, still using a 
fully connected graph but taking into account potentially varying edge weights $\edgeweight_{\nodeidx,\nodeidx'}$ \cite{JMLR:v22:18-694}.

\emph{Vector-Valued Network Flow Optimization.} The dual problem \eqref{equ_dual_nLasso} of GTV minimization \eqref{equ_gtvmin} is closely 
related to network flow optimization. Indeed, the dual problem in the form \eqref{equ_def_duality_nLasso_edge_node} 
generalizes the optimal flow problem \cite[Sec. 1J]{RockNetworks} to vector-valued flows. The special case of the dual 
problem \eqref{equ_def_duality_nLasso_edge_node_norm_gtv_pen}, obtained when the GTV penalty function $\gtvpenalty$ 
is a norm, is equivalent to a generalized minimum-cost flow problem \cite[Sec. 1.2.1]{BertsekasNetworkOpt}. Indeed, the maximization 
problem \eqref{equ_def_duality_nLasso_edge_node_norm_gtv_pen} is equivalent to the minimization 
\begin{align}
	\label{equ_def_duality_nLasso_edge_node_norm_gtv_pen_min_equiv}
	\min_{\flowvec \in \edgespace}   \sum_{\nodeidx \in \nodes} \conjlocallossfunc{\nodeidx}{\localparams{\nodeidx}}&   \nonumber \\ 
	\mbox{ subject to } & - \localparams{\nodeidx}  = \sum_{\edgeidx \in \edges} 
	\sum_{\nodeidx = \edgeidx_{+}} \localflowvec{\edgeidx} - \sum_{\nodeidx = \edgeidx_{-}}  \localflowvec{\edgeidx} \nonumber \\
    & \quad \quad \mbox{ for each } \nodeidx\!\in\!\nodes \nonumber \\
	& \| \localflowvec{\edgeidx} \|_{*}   \leq  \regparam  \edgeweight_{\edgeidx} \mbox{ for each } \edgeidx\!\in\!\edges. 
  \vspace{-2mm}
\end{align}
The optimization problem \eqref{equ_def_duality_nLasso_edge_node_norm_gtv_pen_min_equiv} reduces to 
the minimum-cost flow problem \cite[Eq. (1.3) - (1.5)]{BertsekasNetworkOpt} for local model parameters of length $\dimlocalmodel=1$ (i.e., 
local models parametrized by a scalar). 

\emph{Locally Weighted Learning.} Yet another interpretation of GTV minimization is as an instance of locally 
weighted learning \cite{LocallyWeightedLearning}. Indeed, our analysis in Section \ref{sec_when_does_it_work} 
reveals that, under certain conditions on the connectivity of the empirical graph and local loss functions, 
GTV minimization effectively implements cluster-wise optimization \eqref{equ_def_opt_cluster}. 
In other words, if the node $\nodeidx$ belongs to the cluster $\clustergeneric$, the solution $\estlocalparams{\nodeidx}$ of GTV 
approximates the solution $\overline{\weights}^{(\clustergeneric)}$ of \eqref{equ_def_opt_cluster}, 
$\estlocalparams{\nodeidx} \approx \overline{\weights}^{(\clustergeneric)}$. The deviation between $\estlocalparams{\nodeidx}$ 
and $\overline{\weights}^{(\clustergeneric)}$ will be bounded by Theorem \ref{thm_main_result}. 
The cluster-wise optimization \eqref{equ_def_opt_cluster} is a locally 
weighted learning problem \cite[Sec.\ 3.1.2]{LocallyWeightedLearning}
\begin{equation} 
	\label{equ_def_locally_weighted_problem}
\overline{\weights}^{(\clustergeneric)} =  \argmin_{\vw \in \mathbb{R}^{\dimlocalmodel}} \sum_{\nodeidx \in \nodes} \nodeweight{\nodeidx} \locallossfunc{\nodeidx}{\vw}.
\vspace{-2mm}
\end{equation}
with $\nodeweight{\nodeidx}\!=\!1$ for $\nodeidx\!\in\!\clustergeneric$ and $\nodeweight{\nodeidx}=0$ for $\nodeidx\!\in\!\nodes \setminus \clustergeneric$. 
Note that \eqref{equ_def_locally_weighted_problem} is an adaptive pooling of local datasets into a cluster $\clustergeneric$. 
The pooling of local datasets is driven jointly by the geometry (connectivity) of the empirical graph $\graph$ and the 
geometry (shape) of local loss functions (see Theorem \ref{thm_main_result}). 

\emph{Multitask-Learning.} GTV minimization implements a form of multi-task learning \cite{Caruana:1997wk}. 
Indeed, each cluster (see Assumption \ref{asspt_weights_clustered}) of local datasets gives rise to a separate 
learning task, i.e., to learn cluster-wise model parameters \eqref{equ_def_opt_cluster}. These cluster-wise 
learning tasks are determined by the shape of the local loss functions and the connectivity of the empirical graph (see Theorem \ref{thm_main_result}). 

\section{A Primal Dual Method for GTV Minimization}
\label{sec_federatedml}

The GTV minimization problem \eqref{equ_gtvmin} is a non-smooth convex optimization problem with a 
particular structure. In particular, the objective function in \eqref{equ_gtvmin} consists of 
two components that could be optimized easily when considered separately. Indeed, the 
first component $\sum_{\nodeidx \in \nodes} \locallossfunc{\nodeidx}{\localparams{\nodeidx}}$ 
could be minimized trivially by separately minimizing $\locallossfunc{\nodeidx}{\localparams{\nodeidx}}$, 
for each node $\nodeidx \in \nodes$. The second component $\regparam \|\netparams\|_{\rm GTV}$ 
is minimized trivially by using constant networked model parameters. 

Primal-dual methods use tools from convex duality to solve problems with a composite objective 
function such as \eqref{equ_gtvmin} \cite{RockafellarBook,pock_chambolle_2016}. 
We obtain Algorithm \ref{alg1} by applying the primal-dual method \cite[Alg. 6]{pock_chambolle_2016} to 
jointly solve \eqref{equ_gtvmin} and \eqref{equ_dual_nLasso}. 
Note that Algorithm \ref{alg1} is parametrized by the local loss function $\locallossfunc{\nodeidx}{\cdot}$ 
and the GTV penalty function $\gtvpenalty(\cdot)$ (see \eqref{eq:5}). 

At its core, Algorithm \ref{alg1} computes and distributes (via the edges of the 
empirical graph $\graph$) the node-wise primal and the edge-wise dual updates in steps \eqref{primal_udpate} 
and \eqref{eq_dual_update}, respectively. The primal update (operator) at node $\nodeidx \in \nodes$ in step \eqref{primal_udpate} is 
\begin{equation}
\label{equ_node_wise_primal_update_def}
\begin{aligned}
\primalupdate{\nodeidx}{\vv}\!\defeq\!& \argmin_{\vz \in \mathbb{R}^{\dimlocalmodel}} \locallossfunc{\nodeidx}{\vz}\!+\!\left( \big|\neighbourhood{\nodeidx}\big|/2 \right) \|\vv\!-\!\vz\|^{2}. 
\end{aligned}
\vspace{-2mm}
\end{equation}
Here, we used the neighbourhood $\neighbourhood{\nodeidx} \defeq \{ \nodeidx' \in \nodes: \{\nodeidx,\nodeidx'\} \in \edges\}$ 
of a node $\nodeidx \in \nodes$. Comparing \eqref{equ_node_wise_primal_update_def} with \eqref{equ_def_proximity_operator} 
reveals that the primal update $\primalupdate{\nodeidx}{\cdot}$ is exactly the proximal  
operator of the local loss function $\locallossfunc{\nodeidx}{\cdot}$. 
$$\primalupdate{\nodeidx}{\vv} = \proximityop{\locallossfunc{\nodeidx}{\cdot}}{\vv}{\rho} \mbox{ with }\rho=\big|\neighbourhood{\nodeidx}\big|. \vspace{-2mm}$$

The primal update \eqref{equ_node_wise_primal_update_def} is instance of RERM using  
the local loss $\locallossfunc{\nodeidx}{\cdot}$ \cite{MLBasics} as training error. The regularization 
term in \eqref{equ_node_wise_primal_update_def} is the squared Euclidean distance between local 
model parameters and the argument $\vv$ of the primal update. It enforces 
similar local model parameters at well-connected nodes in the same cluster. The amount of regularization is controlled  
by the node degree $\big|\neighbourhood{\nodeidx}\big|$. In particular, the larger the node degree, the smaller the influence 
of the local loss function on the resulting of the primal update \eqref{equ_node_wise_primal_update_def}. 

Algorithm \ref{alg1} alternates between the primal updates \eqref{equ_node_wise_primal_update_def}, 
for each node $\nodeidx \in \nodes$, and dual updates 
\begin{align}
\label{equ_edge_wise_dual_update_min_def}
 \dualupdate^{(\edgeidx)} \left\{ \vv \right \}\defeq & \argmin_{\vz \in \mathbb{R}^{\dimlocalmodel}} \regparam \edgeweight_{\edgeidx} \gtvpenalty^{*}\big(\vz/(\regparam \edgeweight_{\edgeidx}) \big) \!+\!(1/2\sigma_{\edgeidx}) \|\vv\!-\!\vz\|^{2} \nonumber \\
 & \mbox{ for each } \edgeidx \in \edges. 
\end{align} 
Here, we used the convex conjugate $\gtvpenalty^{*}(\vv) \defeq  \sup_{\vz \in \mathbb{R}^{\dimlocalmodel}} \vv^{T}\vz - \gtvpenalty(\vz)$ 
of the GTV penalty function $\gtvpenalty(\vv)$ (see \eqref{eq:5}). A comparison of \eqref{equ_edge_wise_dual_update_min_def} 
with \eqref{equ_def_proximity_operator} reveals that the dual update $\dualupdate^{(\edgeidx)}$ is essentially  
the (scaled) proximal operator of the convex conjugate $\gtvpenalty^{*}(\vv)$,  
$$ \dualupdate^{(\edgeidx)} \left\{ \vv \right \}=\regparam \edgeweight_{\edgeidx} \proximityop{\gtvpenalty^{*}}{\vv/(\regparam \edgeweight_{\edgeidx})}{\rho} \mbox{ with }\rho=\regparam \edgeweight_{\edgeidx}/\sigma_{\edgeidx}.$$

We can interpret the dual update operator \eqref{equ_edge_wise_dual_update_min_def} as a regularized 
minimization of the convex conjugate $\gtvpenalty^{*}(\vv)$ of the GTV penalty function. The regularization 
term in \eqref{equ_edge_wise_dual_update_min_def} forces the update to not deviate too much from its argument $\vv$. 
Note that the dual update \eqref{equ_edge_wise_dual_update_min_def} is parametrized by the GTV penalty function
 $\gtvpenalty(\vv)$ \eqref{eq:5} (via its convex conjugate). 
Some widely used FL methods are obtained from GTV minimization for specific choices of GTV penalty function  \cite{Smith2017,NetworkLasso,LocalizedLinReg2019,Sarchesh2021}. 

The ``{\rm MOCHA} penalty'' $\gtvpenalty(\vv) \defeq (1/2)\| \vv \|^{2}_{2}$ \cite{Smith2017}, lends to the dual update 
operator $\dualupdate^{(\edgeidx)} \big\{ \vv \big \} \defeq  \vv  / \big(1  + (\sigma_{\edgeidx}/(\regparam \edgeweight_{\edgeidx})) \big)$. 
An obvious generalization of the MOCHA penalty is $\gtvpenalty(\vv)\!\defeq (1/2)\vv^{T}\mQ\vv$ with a fixed positive semidefinite matrix $\mQ$. 
The corresponding dual operator is $\dualupdate^{(\edgeidx)} \big\{ \vv \big \} \defeq \left((\sigma_{\edgeidx}/(\regparam \edgeweight_{\edgeidx})) \mQ^{-1}+\mI\right)^{-1} \vv$, which is a linear operator. 

For the ``nLasso penalty'' $\gtvpenalty(\vv)\!\defeq\!\| \vv \|_{2}$ \cite{NetworkLasso,LocalizedLinReg2019}, 
the dual update operator becomes  the (vector) clipping operator $\dualupdate^{(\edgeidx)}\big\{ \vv \big \} \defeq  \mathcal{T}^{(\regparam \edgeweight_{\edgeidx })} \{\vv \}$ (see \eqref{equ_vector_clipping}). The nLasso penalty is the Euclidean norm $\normgeneric{\vv}{2}$ 
of the flow vector $\vv$ across an edge. Using instead the $\ell_{1}$ norm yields the GTV penalty function 
$\gtvpenalty(\vv) \defeq \| \vv \|_{1}$ \cite{Sarchesh2021}whose associated dual update operator 
$\dualupdate^{(\edgeidx)} \big\{ \vv \big \} \defeq \big( \mathcal{T}^{(\regparam \edgeweight_{\edgeidx})} \big(v_{1}),\ldots,\mathcal{T}^{(\regparam \edgeweight_{\edgeidx})} \big(v_{\dimlocalmodel})\big)^{T}$ is an element-wise application of the scalar clipping operator. 




\begin{algorithm}[htbp]
	\caption{Primal-Dual Method for GTV Minimization}
	\label{alg1}
	{\bf Input}: empirical graph $\graph$; local loss $\big\{ \locallossfunc{\nodeidx}{\cdot}\big\}_{\nodeidx \in \nodes}$, for $\nodeidx \in \nodes$,  
	GTV parameter $\regparam$ and penalty $\gtvpenalty(\cdot)$\\
	{\bf Initialize}: $\itercntr \defeq 0$; $\estlocalparams{\nodeidx}_0 \defeq {\bf 0}, \tau_{\nodeidx}=1/|\neighbourhood{\nodeidx}|$ for all nodes $\nodeidx \in \nodes$; $\estlocalflowvec{\edgeidx}_0 \defeq  {\bf 0}, \sigma_{\edgeidx}=1/2$ for each $\edgeidx\!\in\!\edges$;   
	\begin{algorithmic}[1]
		\While{stopping criterion is not satisfied} \label{equ_stopping_criterion}
	    \For{all nodes $ \nodeidx \in \nodes$}
		\State $\widehat{\weights}_{\iteridx+1}^{(\nodeidx)} \defeq \widehat{\weights}^{(\nodeidx)}_\iteridx -  \tau_{\nodeidx} \sum_{\edgeidx \in \edges} \incidencemtxentry{\edgeidx}{\nodeidx}\widehat{\flowvec}^{(\edgeidx)}_\iteridx $ \label{equ_non_labled_pu}
		\State $\widehat{\weights}_{\iteridx+1}^{(\nodeidx)} \defeq \primalupdate{\nodeidx} {\widehat{\weights}_{\iteridx+1}^{(\nodeidx)}}$ \label{primal_udpate}
		\EndFor
	    \For{all edges $\edgeidx\in \edges $}
	    \State  $\widehat{\flowvec}^{(\edgeidx)}_{\iteridx+1} \defeq \widehat{\flowvec}^{(\edgeidx)}_\itercntr\!+\!\sigma_{\edgeidx} \big(2 \big(\widehat{\weights}^{(\edgeidx_{+})}_{\itercntr\!+\!1}\!-\!\widehat{\weights}^{(\edgeidx_{-})}_{\itercntr+1}\big)\!-\!  \big(\widehat{\weights}^{(\edgeidx_{+})}_{\itercntr}\!-\!\widehat{\weights}^{(\edgeidx_{-})}_{\itercntr} \big)\big)$ \label{alg1_compute_diff_edge}
		\State $\widehat{\flowvec}^{(\edgeidx)}_{\itercntr\!+\!1} \!\defeq\!\dualupdate^{(\edgeidx)} \big\{ \widehat{\flowvec}^{(\edgeidx)}_{\itercntr\!+\!1} \big\}$  \label{eq_dual_update}
		\EndFor
		\State $\itercntr\!\defeq\!\itercntr\!+\!1$
		\EndWhile
 	 \Ensure learnt model parameters $\widehat{\weights}_{\iteridx+1}^{(\nodeidx)}$ for each node $\nodeidx \in \nodes$
	\end{algorithmic}
\end{algorithm}

Algorithm \ref{alg1} can be implemented as message passing over the edges of the empirical graph $\graph$. 
During each iteration of Algorithm \ref{alg1}, local computations are carried out at each node and each edge. 
These local computations are applied to local quantities $\widehat{\flowvec}^{(\edgeidx)}_{\iteridx}, \widehat{\weights}^{(\nodeidx)}_{\iteridx} \in \mathbb{R}^{\dimlocalmodel}$ that are stored at each edge $\edgeidx \in \edges$ and at each node $\nodeidx \in \nodeidx$, 
respectively. In practice, we might equip only nodes with computational resources and therefore, for each edge $\edgeidx \in \edges$, 
maintain a copy of $\widehat{\flowvec}^{(\edgeidx)}_{\iteridx}$ at the nodes $\edgeidx_{+}, \edgeidx_{-} \in \nodes$. 

The primal update step \eqref{primal_udpate}, which is executed in parallel for each node $\nodeidx \in \nodes$, amounts 
to a RERM \eqref{equ_node_wise_primal_update_def}. For each node $\nodeidx \in \nodes$, this RERM involves the local loss function $\locallossfunc{\nodeidx}{\localparams{\nodeidx}}$ and a regularization term that is determined by the current model parameters 
at the neighbours $\neighbourhood{\nodeidx}$. The dual update step \eqref{eq_dual_update}, which is executed  in parallel 
for each edge $\edgeidx \in \edges$, is parametrized by the GTV penalty function $\gtvpenalty$ (see \eqref{eq:5}). 

The results of the node-wise primal and edge-wise dual updates in step \eqref{primal_udpate} and \eqref{eq_dual_update} 
are spread to neighbouring (incident) edges and nodes during steps \eqref{equ_non_labled_pu} and \eqref{alg1_compute_diff_edge} 
of Algorithm \ref{alg1}. Trivially, the number of basic computational steps required by Algorithm \ref{alg1} is directly 
proportional to the number of edges in the empirical graph $\graph$.  


It can be shown that Algorithm \ref{alg1} is robust against errors occurring during the updates 
in steps \eqref{primal_udpate} and \eqref{eq_dual_update} \cite{Rasch:2020tx,Condat2013}. This is 
important for applications where the update operators \eqref{equ_node_wise_primal_update_def} and \eqref{equ_edge_wise_dual_update_min_def} 
can be evaluated only approximately or these updates have to be transmitted over imperfect wireless links.
Let us denote the perturbed update (e.g., obtained 
by a numerical optimization method for solving \eqref{equ_node_wise_primal_update_def}) 
by $\widetilde{{\weights}}_{\itercntr\!+\!1}$ and the exact update by $\estlocalparams{\nodeidx}_{\itercntr+1}$, respectively. 
Then, Algorithm \ref{alg1} is still guaranteed to converge to a solution of \eqref{equ_gtvmin} as long as $\sum_{\itercntr=1}^{\infty} \normgeneric{\widetilde{{\weights}}_{\itercntr\!+\!1} -\estlocalparams{\nodeidx}_{\itercntr+1}}{} < \infty$  (see \cite[Sec.\ 3]{Condat2013}).  



Possible stopping criteria in step \eqref{equ_stopping_criterion} of Algorithm \ref{alg1} include 
the use of a fixed number of iterations. It might also be useful to stop iterating when the 
decrease in the objective function \eqref{equ_gtvmin} falls below a threshold. The  
construction of a stopping criterion can also be based on the primal-dual gap 
\begin{equation} \label{equ_def_pd_gap}
\begin{aligned}
    \pdgap{\iteridx} \defeq & \sum_{\nodeidx \in \nodes} \locallossfunc{\nodeidx}{\widehat{\weights}^{(\nodeidx)}_{\iteridx+1}} + \regparam \| \widehat{\weights}_{\iteridx+1}\|_{\rm GTV} - \\
    & \left(- \sum_{\nodeidx \in \nodes} \conjlocallossfunc{\nodeidx}{-\demand{\nodeidx}} -  \regparam \sum_{\edgeidx \in \edges}   \edgeweight_{\edgeidx}  \gtvpenalty^{*}\big( \widehat{\flowvec}^{(\edgeidx)}_{\itercntr\!+\!1}  /  ( \regparam  \edgeweight_{\edgeidx}) \big)\right)  \\ 
    &\mbox{ with } \demand{\nodeidx} \defeq   \sum_{\edgeidx \in \edges} \sum_{\nodeidx = \edgeidx_{+}}  \widehat{\flowvec}^{(\edgeidx)}_{\itercntr\!+\!1} - \sum_{\nodeidx = \edgeidx_{-}}  \widehat{\flowvec}^{(\edgeidx)}_{\itercntr\!+\!1}.
\end{aligned}
\end{equation} 
By comparing \eqref{equ_def_pd_gap} with \eqref{equ_equal_primal_dual}, we can bound the sub-optimality 
of the iterate $\widehat{\weights}^{(\nodeidx)}_{\iteridx+1}$ as 
\begin{equation*}
    \begin{aligned}
        \sum_{\nodeidx \in \nodes} & \locallossfunc{\nodeidx}{\widehat{\weights}^{(\nodeidx)}_{\iteridx\!+\!1}} + \regparam \| \widehat{\weights}_{\iteridx+1}\|_{\rm GTV} \\
        & - \min_{\netparams \in \nodespace}
 \left[  \sum_{\nodeidx \in \nodes} \locallossfunc{\nodeidx}{\weights^{(\nodeidx)}} + \regparam \| \netparams \|_{\rm GTV} \right] 
 \leq \pdgap{\iteridx}.
    \end{aligned}
    \vspace{-2mm}
\end{equation*}
 Thus, to ensure that Algorithm \ref{alg1} delivers local model parameters with 
 sub-optimality no larger than $\thresholdwassdist$, we only stop when $\pdgap{\itercntr} \leq \thresholdwassdist$. 
 If the local loss functions $\locallossfunc{\nodeidx}{\cdot}$ and the GTV penalty $\gtvpenalty$ are convex, 
 this condition is always fulfilled after a finite number of iterations \cite[Thm. 5.1]{pock_chambolle_2016}.

\section{When Does It Work?}
\label{sec_when_does_it_work}

Section \ref{sec_nLasso} developed Algorithm \ref{alg1} as a novel FL method for distributed training of 
tailored models from collections of local datasets. We obtained Algorithm \ref{alg1} by using 
the primal-dual method \cite[Alg. 6]{pock_chambolle_2016} to solve the GTV minimization problem \eqref{equ_gtvmin} 
jointly with its dual \eqref{equ_def_duality_nLasso_edge_node}. 
The rationale behind Algorithm \ref{alg1}, and GTV minimization \eqref{equ_gtvmin} in the first 
place, was that local loss functions are clustered according to Assumption \ref{asspt_weights_clustered} 
and that this cluster structure is reflected by the connectivity of the empirical graph $\graph$. 

To analyze the statistical properties of Algorithm \ref{alg1}, we must make precise 
the relation between the network structure of $\graph$ and the cluster structure (see \eqref{equ_def_parition_asspt}) 
of local loss functions. To this end, we introduce the notion of a well-connected cluster $\clustergeneric  \subseteq \nodes$ 
whose nodes share (approximately) a common optimal choice for local model parameters (see Assumption \ref{asspt_weights_clustered}).  
\begin{definition}[Well-Connected Cluster] 
\label{equ_def_well_connected_cluster}
Consider an empirical graph $\graph=\pair{\nodes}{\edges}$ with edge weights $\edgeweight_{\nodeidx,\nodeidx'}$ 
for $\nodeidx,\nodeidx' \in \nodes$. The nodes $\nodeidx \in \nodes$ carry local loss functions $\locallossfunc{\nodeidx}{\localparams{\nodeidx}}$. 
that are partitioned into clusters $\partition = \{\cluster{1},\ldots,\cluster{\nrcluster}\}$ according to Assumption \ref{asspt_weights_clustered} 
with clustering error $\clusteropterr{\nodeidx}$ (see \eqref{equ_upper_bound_norm_gradient}). By Assumption \ref{asspt_FIM_lower_bound}, 
we also require local loss functions to be smooth with Lipschitz constant $\locallipsch{\nodeidx}$ (see \eqref{equ_def_cond_locallipsch}) 
and their cluster-wise sum to be strongly convex with parameter $\strongconvparam{\clusteridx}$ (see \eqref{equ_strong_convexit}). 
A cluster $\cluster{\clusteridx} \in \partition$ (see \eqref{equ_def_parition_asspt}), 
with weighted boundary $\big| \partial \cluster{\clusteridx} \big| \defeq \sum_{\nodeidx \in \cluster{\clusteridx}, \nodeidx' \notin \cluster{\clusteridx}} \edgeweight_{\nodeidx,\nodeidx'}$, is well-connected if it contains a node $\nodeidx_{0} \in \cluster{\clusteridx}$ such that 
\begin{equation} 
\label{equ_def_well_connected_condition}
\begin{aligned}
     \sum_{\nodeidx \in \mathcal{A}}  \big[ \sum_{\nodeidx' \notin \cluster{\clusteridx}} \edgeweight_{\nodeidx,\nodeidx'} \big] & + \locallipsch{\nodeidx} \big| \partial  \cluster{\clusteridx} \big | /\strongconvparam{\clusteridx} \\ 
     & + \clusteropterr{\nodeidx}/\regparam  < \sum_{\nodeidx \in \mathcal{A}}  \sum_{\nodeidx' \in \cluster{\clusteridx} \setminus \mathcal{A} }
 \edgeweight_{\nodeidx,\nodeidx'} \\
 & \mbox{ for every subset } \mathcal{A} \subseteq \cluster{\clusteridx} \setminus \big\{ \nodeidx_{0} \big\}. 
\end{aligned}
\end{equation} 
\end{definition}
Our main result is that GTV minimization \eqref{equ_gtvmin} captures an underlying partition of local datasets into nearly 
homogeneous data (see Assumption \ref{asspt_weights_clustered}) if every cluster is well-connected according to 
Definition \ref{equ_def_well_connected_cluster}. More precisely, we have the following upper bound on the deviation 
between the solutions of GTV minimization \eqref{equ_gtvmin} and the cluster-wise optimizers \eqref{equ_def_opt_cluster}. 
\begin{theorem}
\label{thm_main_result}
Consider networked data with empirical graph $\graph$ whose nodes carry local loss functions that are 
clustered into $\partition = \{ \cluster{1},\ldots,\cluster{\nrcluster} \}$ according to Assumption \ref{asspt_weights_clustered} 
and satisfy Assumption \ref{asspt_FIM_lower_bound}. We learn local model parameters $\estlocalparams{\nodeidx}$ 
by solving GTV minimization \eqref{equ_gtvmin} with penalty function $\gtvpenalty$ being a norm. 

If every cluster $\cluster{\clusteridx} \in \partition$ is well-connected, then 
\begin{itemize}
\item the learnt parameter vectors are constant over nodes in the same cluster $\cluster{\clusteridx} \in \partition$, $
	 \estlocalparams{\nodeidx} = \estlocalparams{\nodeidx'} \mbox{ for } \nodeidx,\nodeidx' \in \cluster{\clusteridx}.$
\item the deviation between the GTV solution $\estlocalparams{\nodeidx}$ at some node $\nodeidx \in \cluster{\clusteridx}$ 
and the cluster-wise optimizer $\clusterwideopt{\clusteridx}$ in \eqref{equ_def_opt_cluster} (for $\clustergeneric=\cluster{\clusteridx}$) 
is bounded as 
\begin{equation} 
	\label{equ_main_result_upper_bound}
	\normgeneric{ \estlocalparams{\nodeidx}  - \clusterwideopt{\clusteridx}}{} \leq  2 \big| \partial\cluster{\clusteridx} \big| \regparam/\strongconvparam{\clusteridx}.
\end{equation} 
\end{itemize}
\end{theorem} 
\begin{proof} 
See the supplementary material \ref{app_proof_main_result}. 
\end{proof} 
The error bound \eqref{equ_main_result_upper_bound} involves the shape of the local loss functions 
via the strong convexity parameter $\strongconvparam{\clusteridx}$ for cluster $\cluster{\clusteridx}$. 
Note that Assumption 
\ref{asspt_FIM_lower_bound} requires, for each cluster $\cluster{\clusteridx} \in \partition$ of the 
partition \eqref{equ_def_parition_asspt}, the cluster-wise aggregate loss function $\sum_{\nodeidx\in \cluster{\clusteridx}} \locallossfunc{\nodeidx}{\localparams{\nodeidx}}$ to be strongly convex. 

It is important to note that the bound \eqref{equ_main_result_upper_bound} only applies if the 
cluster $\cluster{\clusteridx}$ is well-connected in the sense of \eqref{equ_def_well_connected_condition}. 
There is trade-off between having a tight bound \eqref{equ_main_result_upper_bound} (favouring a small 
$\regparam$) and ensuring the validity of \eqref{equ_def_well_connected_condition} (favouring a large $\regparam$). 
Moreover, we stress that Theorem \ref{thm_main_result} only applies to GTV minimization \eqref{equ_gtvmin} 
with $\regparam > 0$. 

Theorem \ref{thm_main_result} offers some guidance for the choice of the GTV parameter $\regparam$ (see \eqref{equ_gtvmin}). 
Indeed, according to \eqref{equ_main_result_upper_bound}, $\regparam$ should be small compared to the 
strong convexity parameter $\strongconvparam{\clusteridx}$ in order to ensure 
a small estimation error $\normgeneric{ \estlocalparams{\nodeidx}  - \clusterwideopt{\clusteridx}}{}$ . 
On the other hand, $\regparam$ should be sufficiently large such that the cluster $\cluster{\clusteridx}$ 
is well-connected in the sense of \eqref{equ_def_well_connected_condition} and, in turn, GTV minimization 
is guaranteed to deliver identical model parameters for all nodes in $\cluster{\clusteridx}$. 

Besides the choice for $\regparam$, the validity of the condition \eqref{equ_def_well_connected_condition} 
also depends on the choice for the partition in Assumption \ref{asspt_FIM_lower_bound}. We stress that the 
partition $\partition=\{ \cluster{1},\ldots,\cluster{\nrcluster} \}$ in Assumption \ref{asspt_weights_clustered} 
is only used to analyze the solutions of GTV minimization \eqref{equ_gtvmin}. Algorithm \ref{alg1} only 
requires the empirical graph as input but it does not require any specification of a partition. Theorem 
\ref{thm_main_result} provides a sufficient condition for the solutions of Algorithm \ref{alg1} to conform 
with a partition $\partition = \{\cluster{1},\ldots\}$ that satisfies \eqref{equ_def_well_connected_condition}. 
In general, GTV minimization \eqref{equ_gtvmin} might have several solutions, each having a different 
cluster structure \cite{JungTVMin2019}. 

\section{Numerical Experiments}
\label{sec_numexp}

This section reports the results of some illustrative numerical experiments to verify the performance 
of Algorithm \ref{alg1}. We provide the code to reproduce these experiments in the supplementary 
material. The experiments discussed in Section \ref{sbm_experiment_section} - Section \ref{chain_section} 
revolve around synthetic datasets whose empirical graph is either a chain graph, a star graph 
or a realization of a Stochastic Block Model (SBM) \cite{AbbeSBM2018}. A numerical experiment with a handwritten digit 
image dataset (``MNIST'') is discussed in Section \ref{mnist_section}. 

\subsection{Stochastic Block Model} 
\label{sbm_experiment_section}

This experiment revolves around synthetic data whose empirical graph $\graph^{(\rm SBM)}$ is partitioned 
into two equal-sized clusters $\partition = \{ \cluster{1}, \cluster{2}\}$, with $|\cluster{1}| = |\cluster{2}|$. We denote 
the cluster assignment of node $\nodeidx \in \nodes$ by $\clusteridx^{(\nodeidx)} \in \{1,2\}$. The 
edges in $\graph^{(\rm SBM)}$ are generated via realizations of independent binary 
random variables $b_{\nodeidx,\nodeidx'} \in \{0,1\}$. These random variables are indexed by pairs $\nodeidx,\nodeidx'$ of nodes 
that are connected by an edge $\edge{\nodeidx}{\nodeidx'} \in \edges$ if and only if $b_{\nodeidx,\nodeidx'}=1$. Two nodes 
in the same cluster are connected with probability ${\rm Prob}\{ b_{\nodeidx,\nodeidx'}=1 \} \defeq p_{\rm in}$ if $\nodeidx,\nodeidx'$. 
In contrast, ${\rm Prob}\{ b_{\nodeidx,\nodeidx'}=1 \} \defeq p_{\rm out}$ if nodes $\nodeidx,\nodeidx'$ belong to different clusters. 
Every edge in $\graph^{(\rm SBM)}$ has the same weight, $\edgeweight_{\edgeidx} = 1$ for all $\edgeidx \in \edges$.  

Each node $\nodeidx \in \nodes$ of the empirical graph $\graph^{(\rm SBM)}$ holds a local dataset $\localdataset{\nodeidx}$ of the 
form \eqref{equ_def_local_dataset_plain}. Thus, the dataset $\localdataset{\nodeidx}$ consists of $\localsamplesize{\nodeidx}$ 
data points, each characterized by a feature vector $ \featurevec^{(\nodeidx,\localsampleidx)} \in \mathbb{R}^{\dimlocalmodel}$ 
and a scalar label $\truelabel^{(\nodeidx,\localsampleidx)}$, for $\localsampleidx=1,\ldots,\localsamplesize{\nodeidx}$. 
The feature vectors $ \featurevec^{(\nodeidx,\sampleidx)} \sim \mathcal{N}(\mathbf{0},\mathbf{I}_{\dimlocalmodel \times \dimlocalmodel})$, 
are drawn i.i.d.\ from a standard multivariate normal distribution. The labels of the data points are generated by a 
noisy linear model 
\begin{equation} 
\label{equ_def_true_linear_model_SBM}
\truelabel^{(\nodeidx,\sampleidx)} = \big( \overline{\weights}^{(\nodeidx)}\big)^{T} \featurevec^{(\nodeidx,\sampleidx)} + \sigma \varepsilon^{(\nodeidx,\sampleidx)}. 
\end{equation} 
The noise $\varepsilon^{(\nodeidx,\sampleidx)}$, for $\nodeidx \in \nodes$ and $\sampleidx=1,\ldots,\localsamplesize{\nodeidx}$, 
are i.i.d.\ realizations of a standard normal distribution. The true underlying weight vector 
$\overline{\weights}^{(\nodeidx)}$ is piece-wise constant over the clusters in the partition 
$\partition = \{ \cluster{1}, \cluster{2} \}$, i.e., $\overline{\weights}^{(\nodeidx)} = \overline{\weights}^{(\nodeidx')}$ if $\nodeidx,\nodeidx' \in \cluster{\clusteridx}$. 

To study the robustness of Algorithm \ref{alg1} against node failures, we assume that local datasets 
are only accessible in a subset $\trainingset \subseteq \nodes$. The set $\trainingset$ is 
selected uniformly at random among all nodes $\nodes$. We can access local datasets $\localdataset{\nodeidx}$ 
only in the subset $\trainingset$ of relative size $\rho \defeq |\trainingset|/|\nodes|$. The inaccessibility 
of a local dataset $\localdataset{\nodeidx}$, for $\nodeidx \notin \trainingset$, can be modelled by using a trivial loss 
function. Thus, we learn local model parameters $\estlocalparams{\nodeidx}$ using Algorithm \ref{alg1} with 
local loss functions 
\begin{equation} 
	\label{equ_def_local_loss_squared}
	\locallossfunc{\nodeidx}{\localparams{\nodeidx}} \!\defeq\!\begin{cases} 
 \frac{1}{\localsamplesize{\nodeidx}} 
 \sum_{\sampleidx=1}^{\localsamplesize{\nodeidx}} \big( \big(\featurevec^{(\nodeidx,\sampleidx)}\big)^{T} \localparams{\nodeidx}\!-\! \truelabel^{(\nodeidx,\sampleidx)}\big)^{2} \mbox{ for } \nodeidx\!\in\!\trainingset \\ 
		0 \mbox{  otherwise.}
		\end{cases}. 
\end{equation} 

The special case $\rho=1$ is obtained when all local datasets are accessible, i.e., when $\trainingset = \nodes$. 
For the stopping criterion in Algorithm \ref{alg1} we use a fixed number $\nriter$ of iterations, which 
is $\nriter=3000$ for the results depicted in Figure \ref{fig_noiseless} and $\nriter\!=\!2000$ for 
the results depicted in Figure \ref{fig_fixes_trin_set}. The GTV regularization parameter has 
been set to $\regparam\!=\!10^{-2}$ (see \eqref{equ_def_reguarlizated-local_loss}). We measure 
the estimation error incurred by the learnt parameter vectors $\estlocalparams{\nodeidx}$ using the 
average squared estimation error (MSE), 
\begin{equation}
	\label{equ_def_MSE}
	\estmse \defeq (1/|\nodes|) \sum_{\nodeidx \in \nodes} \|\estlocalparams{\nodeidx}-\overline{\weights}^{(\nodeidx)} \|_{2}^{2}. 
 \vspace{-2mm}
\end{equation} 

\subsubsection{High-Dimensional Linear Regression with Two Clusters}

Table \ref{tab:sbm_table} reports the results obtained by applying Algorithm \ref{alg1} to $\graph^{(\rm SBM)}$ with 
$|\cluster{1}| = |\cluster{2}| = 100$, $p_{\rm in}=0.5$, and $p_{\rm out}=10^{-2}$. Each node carries a local dataset 
of the form \eqref{equ_def_local_dataset_plain} with $\localsamplesize{\nodeidx}\!=\!10$ data points, having a feature vector  
$\featurevec^{(\nodeidx,\localsampleidx)} \in \mathbb{R}^{100}$ and labels generated according to \eqref{equ_def_true_linear_model_SBM} 
with noise strength $\sigma=10^{-3}$. The true underlying parameter vectors in \eqref{equ_def_true_linear_model_SBM} 
are cluster-wise constant, $\overline{\weights}^{(\nodeidx)} = \overline{\weights}^{(\clusteridx)}$ for all $\nodeidx \in \cluster{\clusteridx}$. 
The cluster-wise parameter vector $\overline{\weights}^{(\clusteridx)} \in \{0,0.5\}^{100}$ is constructed entry-wise 
using i.i.d.\ realizations of standard Bernoulli variables $B \in \{0, 0.5\}$ with ${\rm Prob}(B=0)=1/2$. 

We learn local model parameters using Algorithm \ref{alg1} using the local loss functions in \eqref{equ_def_local_dataset_plain}. 
Table \ref{tab:sbm_table} reports the MSE \eqref{equ_def_MSE} incurred by Algorithm \ref{alg1} using 
a fixed number of $\nriter=1000$ iterations. Table \ref{tab:sbm_table} also reports the MSE \eqref{equ_def_MSE} 
incurred by the model parameters learnt by IFCA \cite{Ghosh2020} and FedAvg \cite{Sun2021DecentralizedFA}. 

\begin{table}[htb]
    \centering
    \begin{tabular}{|c|c|} 
    \hline
    Method & MSE \\
    \hline
    Algorithm \ref{alg1} & 1.42e-05 \\ 
    IFCA & 2.82 \\ 
    FedAvg & 2.86 \\ 
    \hline
    \end{tabular}
\vspace*{3mm}
    \caption{MSE \eqref{equ_def_MSE} of the estimation error incurred by 
    	Algorithm \ref{alg1}, IFCA \cite{Ghosh2020} and FedAvg \cite{Sun2021DecentralizedFA}}
    \label{tab:sbm_table}
    \vspace{-3mm}
\end{table}

\subsubsection{High-Dimensional Linear Regression with Multiple Clusters}
The next experiment studies the effect of increasing the number of clusters in $\graph^{(\rm SBM)}$, which is 
divided into $\nrcluster$ equally sized clusters $\partition = { \cluster{1}, \cluster{2}, \cdots, \cluster{\nrcluster} }$. 
The cluster assignment of node $\nodeidx \in \nodes$ is denoted $\clusteridx^{(\nodeidx)} \in \{1,2, \cdots, \nrcluster\}$. 
The local model parameters are learned using Algorithm \ref{alg1}, where we employ local loss 
functions \eqref{equ_def_local_dataset_plain}. Figure \ref{figure:sbm_mutiple_clusters} depicts the MSE \eqref{equ_def_MSE} 
of Algorithm \ref{alg1} (with $\nriter=2000$ iterations), IFCA \cite{Ghosh2020}, FedAvg \cite{Sun2021DecentralizedFA}. 
This figure also depicts the MSE obtained by applying FedAvg to each cluster separately; we denote this approach as {\rm FedAvgC}. 
Note that {\rm FedAvgC} cannot be implemented in practice as the clusters are unknown. Therefore, the 
MSE of {\rm FedAvgC} serves mainly as a benchmark for the MSE of practical methods such as our Algorithm \ref{alg1}. 

\begin{figure}
\centering
\begin{tikzpicture}[scale=0.5]
    \begin{axis}[
    ymode=log, ymax=50.0, xmax=11,
    y label style={at={(axis description cs:0.040,.5)},rotate=0.0,anchor=south},
    label style={font=\Large},
    title style={font=\Large},
    legend style={font=\small},
    yticklabel style = {font=\small},
    xticklabel style = {font=\Large},
    xlabel={number of clusters $\nrcluster$},
    ylabel={MSE},
    legend style={at={(0.95,0.6)},anchor=east},
    ymajorgrids=true,
    grid style=dashed,
    ]
    \addplot coordinates {
      (2,  2.40867595415126e-06 )
      (3,  1.8300178055533646e-05)
      (4,  3.134307032184313e-05)
      (5, 3.345533886855933e-05)
      (10, 0.00018016892870393483)
    } ;\addlegendentry{Algorithm 1}
    \addplot coordinates {
      (2,  2.6200976)
      (3,  1.706273)
      (4,  2.196422)
      (5, 2.429912)
      (10, 3.649451)
    };\addlegendentry{IFCA}
    \addplot coordinates {
      (2,  2.6939742738000567)
      (3,  4.307233581649928)
      (4,  4.767155567977833)
      (5, 5.086188142857921)
      (10, 5.55484034455309)
    };\addlegendentry{FedAvg}
    \addplot coordinates {
      (2,  8.556535952336449e-07)
      (3, 8.494160432828636e-07)
      (4, 8.193665384606393e-07) 
      (5, 8.345243408732243e-07)
      (10, 8.345243408732243e-07)
    }; \addlegendentry{C\_FedAvg}
  \end{axis}
\end{tikzpicture}
\caption{\label{figure:sbm_mutiple_clusters} MSE \eqref{equ_def_MSE} of Algorithm 1, IFCA \cite{Ghosh2020}, FedAvg \cite{Sun2021DecentralizedFA}, and clustered FedAvg (C\_FedAvg) as a function of the number of clusters in $\graph^{(\rm SBM)}$.}
\end{figure}

\subsection{Chain Graph} 
\label{chain_section}
This experiment uses a dataset with $2 \nrnodes$ local datasets whose empirical graph $\graph^{(\rm chain)}$ is a chain graph. 
The edge set of this graph is given by $\edges = \{ \edge{\nodeidx}{\nodeidx+1} : \nodeidx \in \{1,\ldots,2\nrnodes-1 \} \}$. 
The nodes $\nodes = \{1,\ldots,2 \nrnodes\}$ are partitioned into two clusters $\cluster{1} = \{1,\ldots,\nrnodes\}$ 
and $\cluster{2} = \{\nrnodes+1,\ldots,2 \nrnodes\}$. Every intra-cluster edge $\edge{\nodeidx}{\nodeidx'} \in \edges$, 
with $\nodeidx,\nodeidx'$ in the same cluster, has weight $\edgeweight_{\nodeidx,\nodeidx'} =1$. The single 
inter-cluster edge $\edgeidx' = \{\nrnodes,\nrnodes+1\}$ has weight $\edgeweight_{\nrnodes,\nrnodes+1}  = \varepsilon$ 
with some $\varepsilon \geq 0$. 

The nodes of $\graph^{(\rm chain)}$ carry local datasets with the same structure as in Section \ref{sbm_experiment_section}. 
In particular, each local dataset consists of data points that are characterized by feature vectors drawn from a 
multivariate normal distribution and a label that is generated from \eqref{equ_def_true_linear_model_SBM}. 
The true local parameter vectors (``ground truth'') in \eqref{equ_def_true_linear_model_SBM} are piece-wise 
constant over clusters $\cluster{1}$ and $\cluster{2}$. 

We use Algorithm \ref{alg1} to learn the local parameter vectors in \eqref{equ_def_true_linear_model_SBM} 
using the average squared error \eqref{equ_def_local_loss_squared} as local loss function. 
The local datasets are only accessible for a subset $\samplingset \subseteq \nodes$ of $\lceil \rho |\nodes| \rceil$ 
nodes selected uniformly at random among all nodes $\nodes$. We apply Algorithm \ref{alg1} to the local loss 
functions \eqref{equ_def_local_loss_squared} and using a fixed number $\nriter\!=\!2000$ of iterations and 
different choices for the GTV parameter $\regparam$ as indicated in Figure \ref{fig_noiseless_chain} and Figure \ref{fig_fixes_trin_set_chain}. 

Figure \ref{fig_noiseless_chain} and Figure \ref{fig_fixes_trin_set_chain} depict the MSE \eqref{equ_def_MSE} incurred 
by Algorithm \ref{alg1} obtained for varying training size $\rho$, noise strength $\sigma$ (see \eqref{equ_def_true_linear_model_SBM}) 
and inter-cluster edge weight $\varepsilon$. The curves and bars in Figures \ref{fig_noiseless_chain} and \ref{fig_fixes_trin_set_chain} 
represent the average and standard deviation of MSE values of $5$ simulation runs for different choices for $\rho, \varepsilon$ and $\sigma$. 
Figure \ref{fig_noiseless_chain} shows the results in the noise-less case where $\sigma=0$ in \eqref{equ_def_true_linear_model_SBM}, 
and Figure \ref{fig_fixes_trin_set_chain} shows results for varying noise level $\sigma$ in \eqref{equ_def_true_linear_model_SBM} 
and fixed relative training size $\rho=6/10$.

\begin{figure}[htbp]
\begin{minipage}[t]{0.3\columnwidth} 
\begin{tikzpicture}[scale=0.5]
\begin{axis}[
ymode=log, ymin=1e-11, ymax=1e0,
y label style={at={(axis description cs:0.040,.5)},rotate=0.0,anchor=south},
label style={font=\Large},
title style={font=\Large},
legend style={font=\Large},
yticklabel style = {font=\small},
xticklabel style = {font=\Large},
xlabel={$\varepsilon$},
ylabel={MSE},
title={$\phi(\vu)\!=\!\|\vu\|_{1},\regparam\!=\!0.1$},
legend pos=south east,
ymajorgrids=true,
grid style=dashed,
]

\addplot+[color=cyan!40!blue!60, very thick, solid, mark=*, error bars/.cd, y dir=both, y explicit, error bar style={line width=1pt,solid}, error mark options={line width=1pt,mark size=2pt,rotate=90}] table [x=epsilon, y=N1M02MSE, y error=N1M02STD, col sep=comma]{differentSamplingRatiosEpsilon.csv};\addlegendentry{$\rho\!=\!2/10$}
\addplot+[color=orange, very thick, densely dotted, mark=*, error bars/.cd, y dir=both, y explicit, error bar style={line width=1pt,solid}, error mark options={line width=1pt,mark size=2pt,rotate=90}] table [x=epsilon, y=N1M04MSE, y error=N1M04STD, col sep=comma]{differentSamplingRatiosEpsilon.csv};\addlegendentry{$\rho\!=\!4/10$}
\addplot+[color=green!40!teal!60, very thick, densely dashed, mark=*, error bars/.cd, y dir=both, y explicit, error bar style={line width=1pt,solid}, error mark options={line width=1pt,mark size=2pt,rotate=90}] table [x=epsilon, y=N1M06MSE, y error=N1M06STD, col sep=comma]{differentSamplingRatiosEpsilon.csv};\addlegendentry{$\rho\!=\!6/10$}
\end{axis}
\end{tikzpicture}
\end{minipage}
\hspace*{15mm}
\begin{minipage}[t]{0.3\columnwidth} 
\begin{tikzpicture}[scale=0.5]
\begin{axis}[
ymode=log, ymin=1e-11, ymax=1e0,
y label style={at={(axis description cs:0.040,.5)},rotate=0.0,anchor=south},
label style={font=\Large},
title style={font=\Large},
legend style={font=\Large},
yticklabel style = {font=\small},
xticklabel style = {font=\Large},
xlabel={$\varepsilon$},
ylabel={MSE},
title={$\phi(\vu)\!=\!\|\vu\|_{2},\regparam\!=\!0.1$ (nLasso)},
legend pos=south east,
ymajorgrids=true,
grid style=dashed,
]
\addplot+[color=cyan!40!blue!60, very thick, solid, mark=*, error bars/.cd, y dir=both, y explicit, error bar style={line width=1pt,solid}, error mark options={line width=1pt,mark size=2pt,rotate=90}] table [x=epsilon, y=N2M02MSE, y error=N2M02STD, col sep=comma]{differentSamplingRatiosEpsilon.csv};\addlegendentry{$\rho\!=\!2/10$}
\addplot+[color=orange, very thick, densely dotted, mark=*, error bars/.cd, y dir=both, y explicit, error bar style={line width=1pt,solid}, error mark options={line width=1pt,mark size=2pt,rotate=90}] table [x=epsilon, y=N2M04MSE, y error=N2M04STD, col sep=comma]{differentSamplingRatiosEpsilon.csv};\addlegendentry{$\rho\!=\!4/10$}
\addplot+[color=green!40!teal!60, very thick, densely dashed, mark=*, error bars/.cd, y dir=both, y explicit, error bar style={line width=1pt,solid}, error mark options={line width=1pt,mark size=2pt,rotate=90}] table [x=epsilon, y=N2M06MSE, y error=N2M06STD, col sep=comma]{differentSamplingRatiosEpsilon.csv};\addlegendentry{$\rho\!=\!6/10$}
\end{axis}
\end{tikzpicture}
\end{minipage}
\\
\\
\hspace*{22mm}
\begin{minipage}[t]{0.3\columnwidth} 
\begin{tikzpicture}[scale=0.5]
\begin{axis}[
ymode=log, ymin=1e-11, ymax=1e0,
y label style={at={(axis description cs:0.040,.5)},rotate=0.0,anchor=south},
label style={font=\Large},
title style={font=\Large},
legend style={font=\Large},
yticklabel style = {font=\small},
xticklabel style = {font=\Large},
xlabel={$\varepsilon$},
ylabel={MSE},
title={$\phi(\vu)\!=\!\|\vu\|_{2}^{2},\regparam\!=\!0.05$ (``MOCHA'')},
legend pos=south east,
ymajorgrids=true,
grid style=dashed,
]


\addplot+[color=cyan!40!blue!60, very thick, solid, mark=*, error bars/.cd, y dir=both, y explicit, error bar style={line width=1pt,solid}, error mark options={line width=1pt,mark size=2pt,rotate=90}] table [x=epsilon, y=MCM02MSE, y error=MCM02STD, col sep=comma]{differentSamplingRatiosEpsilon.csv};\addlegendentry{$\rho\!=\!0.2$}
\addplot+[color=orange, very thick, densely dotted, mark=*, error bars/.cd, y dir=both, y explicit, error bar style={line width=1pt,solid}, error mark options={line width=1pt,mark size=2pt,rotate=90}] table [x=epsilon, y=MCM04MSE, y error=MCM04STD, col sep=comma]{differentSamplingRatiosEpsilon.csv};\addlegendentry{$\rho\!=\!0.4$}
\addplot+[color=green!40!teal!60, very thick, densely dashed, mark=*, error bars/.cd, y dir=both, y explicit, error bar style={line width=1pt,solid}, error mark options={line width=1pt,mark size=2pt,rotate=90}] table [x=epsilon, y=MCM06MSE, y error=MCM06STD, col sep=comma]{differentSamplingRatiosEpsilon.csv};\addlegendentry{$\rho\!=\!0.6$}

\end{axis}
\end{tikzpicture}
\end{minipage}
\caption{\label{fig_noiseless_chain} MSE \eqref{equ_def_MSE} of the estimation error incurred by 
	Algorithm \ref{alg1} when learning local model parameters for a chain graph $\graph^{(\rm chain)}$ in the 
	noiseless case $\sigma=0$.} 
\vspace{-3mm}
\end{figure}

\begin{figure}[htbp]
\begin{minipage}[t]{0.3\columnwidth} 
\begin{tikzpicture}[scale=0.5]
\begin{axis}[
ymode=log, ymin=1e-7, ymax=1e0,
y label style={at={(axis description cs:0.040,.5)},rotate=0.0,anchor=south},
label style={font=\Large},
title style={font=\Large},
legend style={font=\Large},
yticklabel style = {font=\small},
xticklabel style = {font=\Large},
xlabel={$\epsilon$},
ylabel={MSE},
title={$\phi(\vu)\!=\!\|\vu\|_{1},\regparam\!=\!0.1$},
legend pos=south east,
ymajorgrids=true,
grid style=dashed,
]


\addplot+[color=cyan!40!blue!60, very thick, solid, mark=*, error bars/.cd, y dir=both, y explicit, error bar style={line width=1pt,solid}, error mark options={line width=1pt,mark size=2pt,rotate=90}] table [x=epsilon, y=N1N001MSE, y error=N1N001STD, col sep=comma]{differentNoisesEpsilon.csv};\addlegendentry{$\sigma\!=\!10^{-2}$}
\addplot+[color=orange, very thick, densely dotted, mark=*, error bars/.cd, y dir=both, y explicit, error bar style={line width=1pt,solid}, error mark options={line width=1pt,mark size=2pt,rotate=90}] table [x=epsilon, y=N1N01MSE, y error=N1N01STD, col sep=comma]{differentNoisesEpsilon.csv};\addlegendentry{$\sigma\!=\!10^{-1}$}
\addplot+[color=green!40!teal!60, very thick, densely dashed, mark=*, error bars/.cd, y dir=both, y explicit, error bar style={line width=1pt,solid}, error mark options={line width=1pt,mark size=2pt,rotate=90}] table [x=epsilon, y=N1N10MSE, y error=N1N10STD, col sep=comma]{differentNoisesEpsilon.csv};\addlegendentry{$\sigma\!=\!1$}

\end{axis}
\end{tikzpicture}
\end{minipage}
\hspace*{15mm}
\begin{minipage}[t]{0.3\columnwidth} 
\begin{tikzpicture}[scale=0.5]
\begin{axis}[
ymode=log, ymin=1e-7, ymax=1e0,
y label style={at={(axis description cs:0.040,.5)},rotate=0.0,anchor=south},
label style={font=\Large},
title style={font=\Large},
legend style={font=\Large},
yticklabel style = {font=\small},
xticklabel style = {font=\Large},
xlabel={$\varepsilon$},
ylabel={MSE},
title={$\gtvpenalty(\vu)\!=\!\|\vu\|_{2},\regparam\!=\!0.1$ (nLasso)},
legend pos=south east,
ymajorgrids=true,
grid style=dashed,
]


\addplot+[color=cyan!40!blue!60, very thick, solid, mark=*, error bars/.cd, y dir=both, y explicit, error bar style={line width=1pt,solid}, error mark options={line width=1pt,mark size=2pt,rotate=90}] table [x=epsilon, y=N2N001MSE, y error=N2N001STD, col sep=comma]{differentNoisesEpsilon.csv};\addlegendentry{$\sigma\!=\!10^{-2}$}
\addplot+[color=orange, very thick, densely dotted, mark=*, error bars/.cd, y dir=both, y explicit, error bar style={line width=1pt,solid}, error mark options={line width=1pt,mark size=2pt,rotate=90}] table [x=epsilon, y=N2N01MSE, y error=N2N01STD, col sep=comma]{differentNoisesEpsilon.csv};\addlegendentry{$\sigma\!=\!10^{-1}$}
\addplot+[color=green!40!teal!60, very thick, densely dashed, mark=*, error bars/.cd, y dir=both, y explicit, error bar style={line width=1pt,solid}, error mark options={line width=1pt,mark size=2pt,rotate=90}] table [x=epsilon, y=N2N10MSE, y error=N2N10STD, col sep=comma]{differentNoisesEpsilon.csv};\addlegendentry{$\sigma\!=\!1$}

\end{axis}
\end{tikzpicture}
\end{minipage}
\\
\\
\hspace*{22mm} 
\begin{minipage}[t]{0.3\columnwidth} 
\begin{tikzpicture}[scale=0.5]
\begin{axis}[
ymode=log, ymin=1e-7, ymax=1e0,
y label style={at={(axis description cs:0.040,.5)},rotate=0.0,anchor=south},
label style={font=\Large},
title style={font=\Large},
legend style={font=\Large},
yticklabel style = {font=\small},
xticklabel style = {font=\Large},
xlabel={$\varepsilon$},
ylabel={MSE},
title={$\gtvpenalty(\vu)\!=\!\|\vu\|^2_{2},\regparam\!=\!0.05$ (``MOCHA'')},
legend pos=south east,
ymajorgrids=true,
grid style=dashed,
]


\addplot+[color=cyan!40!blue!60, very thick, solid, mark=*, error bars/.cd, y dir=both, y explicit, error bar style={line width=1pt,solid}, error mark options={line width=1pt,mark size=2pt,rotate=90}] table [x=epsilon, y=MCN001MSE, y error=MCN001STD, col sep=comma]{differentNoisesEpsilon.csv};\addlegendentry{$\sigma\!=\!10^{-2}$}
\addplot+[color=orange, very thick, densely dotted, mark=*, error bars/.cd, y dir=both, y explicit, error bar style={line width=1pt,solid}, error mark options={line width=1pt,mark size=2pt,rotate=90}] table [x=epsilon, y=MCN01MSE, y error=MCN01STD, col sep=comma]{differentNoisesEpsilon.csv};\addlegendentry{$\sigma\!=\!10^{-1}$}
\addplot+[color=green!40!teal!60, very thick, densely dashed, mark=*, error bars/.cd, y dir=both, y explicit, error bar style={line width=1pt,solid}, error mark options={line width=1pt,mark size=2pt,rotate=90}] table [x=epsilon, y=MCN10MSE, y error=MCN10STD, col sep=comma]{differentNoisesEpsilon.csv};\addlegendentry{$\sigma\!=\!1$}
\end{axis}
\end{tikzpicture}
\end{minipage}
\caption{\label{fig_fixes_trin_set_chain} $\estmse$ \eqref{equ_def_MSE} incurred by Algorithm \ref{alg1} when learning 
	the local model parameters for the chain graph $\graph^{(\rm chain)}$ using only local datasets at a fraction of 
	$|\trainingset|/|\nodes|=0.6$ nodes (see \eqref{equ_def_local_loss_squared}).} 
 \vspace{-3mm}
\end{figure}

\subsection{Handwritten Digits} 
\label{mnist_section}

This experiment revolves around a collection of $\nrnodes=40$ local datasets generated from the handwritten image 
dataset (``MNIST dataset'') \cite{LeCun1998}. and represented by the empirical graph $\graphmnist = \pair{\nodesmnist}{\edgesmnist}$. 
Each node $\nodeidx \in \nodesmnist=\{1,\ldots,40\}$ carries a local dataset $\localdataset{\nodeidx}$ which 
consists of $\localsamplesize{\nodeidx}=500$ data points being images of handwritten digits. 
The $\sampleidx$-th data point in $\localdataset{\nodeidx}$ is characterized by the feature vector $\featurevec^{(\nodeidx,\sampleidx)} \in \mathbb{R}^{\dimlocalmodel}$ and a label that specifies the digit that the datapoint belongs to.
The entries of  
the feature vector $\featurevec^{(\nodeidx,\sampleidx)} \in \mathbb{R}^{\dimlocalmodel}$ 
are greyscale levels of $\dimlocalmodel = 28 \times 28$ pixels. 
The nodes $\nodesmnist$ are partitioned into two clusters $\cluster{1},\cluster{2}$ (see Assumption \ref{asspt_weights_clustered}), 
each with $20$ nodes.  The nodes in $\cluster{1}$ carry local datasets that consist of images depicting 
handwritten digits $0$ and $1$. The nodes $\nodeidx \in \cluster{2}$ in the second cluster carry local datasets $\localdataset{\nodeidx}$ 
consisting of images that depict handwritten digits $2$ and $3$. 

Besides the local dataset $\localdataset{\nodeidx}$, the node $\nodeidx \in \nodes$ is also 
assigned a local model in the form of an artificial neural network (ANN),
\begin{equation} 
\label{equ_def_local_ANN_MNIST}
\begin{aligned}
    h^{(\nodeidx)} \big(\featurevec; \localparams{\nodeidx}\big) \defeq {\rm SoftMax}\big( \mW^{(\nodeidx,2)} {\rm ReLU} \big( \mW^{(\nodeidx,1)} \featurevec \big) \big) \\
    \mbox{ with } \localparams{\nodeidx} = {\rm stack} \{\mW^{(\nodeidx,1)} ,\mW^{(\nodeidx,2)}  \}.
\end{aligned}
\vspace{-2mm}
\end{equation} 
The ANN \eqref{equ_def_local_ANN_MNIST} consists of two densely connected layers, with 
the first layer using rectified linear units {\rm ReLU} \cite{Goodfellow-et-al-2016}. The second 
(output) layer uses a ``soft-max'' activation function \cite{Goodfellow-et-al-2016}. 
The weights of connections between layers are stored in the matrices $\mW^{(\nodeidx,1)}$ and $\mW^{(\nodeidx,2)}$, 
respectively. The entries of the weight matrices are collected in the local model parameter vector $\localparams{\nodeidx}$. 

For learning the local model parameters $\localparams{\nodeidx}$ of the ANNs \eqref{equ_def_local_ANN_MNIST}, 
we split each local dataset a training and validation set, $\localdataset{\nodeidx} = \localdatasettrain{\nodeidx} \cup \localdatasetval{\nodeidx}$ 
of size $\localtrainsetsize{\nodeidx}\!=\!400$ and $\localvalsetsize{\nodeidx}\!=\!100$, respectively. 
The local training sets $\localdatasettrain{\nodeidx}$ are also used for the construction of the 
edges in $\graph^{(\rm MNIST)}$ as described next. 

For each $\nodeidx \in \nodesmnist$, we apply {\rm t-SNE} \cite{tSNEPaper} to map the raw feature vector 
$\featurevec^{(\nodeidx,\sampleidx)}$ to the embedding vector $\vz^{(\nodeidx,\sampleidx)} \in \mathbb{R}^{2}$, 
for $\sampleidx =1,\ldots, \localtrainsetsize{\nodeidx}$. We then compute a distance  
${\rm dist}(\nodeidx,\nodeidx')$ between any two different nodes $\nodeidx, \nodeidx' \in \nodesmnist$ 
using a Kullback-Leibler divergence estimator \cite{PerezCruz2008}. 

Given the pairwise distances ${\rm dist}(\nodeidx,\nodeidx')$, for any $\nodeidx,\nodeidx' \in \nodesmnist$, 
we construct the edges of $\edgesmnist$ and their weights as follows. Each node $\nodeidx \in \nodesmnist$ 
is connected with the four other nodes $\nodeidx' \in \nodesmnist \setminus \{ \nodeidx\}$ of 
minimum distance ${\rm dist}(\nodeidx,\nodeidx')$ by an edge $\edgeidx=\edge{\nodeidx}{\nodeidx'}$. 
The edge weights are then constructed by an exponential kernel \cite{Luxburg2007}, $\edgeweight_{\nodeidx,\nodeidx'} = {\rm exp} \big(-{\rm dist}(\nodeidx,\nodeidx') \big)$.

To learn the local parameters $\estlocalparams{\nodeidx}$ of \eqref{equ_def_local_ANN_MNIST}, we use 
Algorithm \ref{alg1} with local loss 
\begin{equation*}
    \begin{aligned}
        \locallossfunc{\nodeidx}{\localparams{\nodeidx}} & \defeq 
         -\big(1/\localtrainsetsize{\nodeidx}\big)
        \sum_{\left(\featurevec,\truelabel\right) \in \localdatasettrain{\nodeidx}} \big( \truelabel \log{h^{(\nodeidx)}\big(\featurevec\big)} \big ) \\
        - & \big(1/\localtrainsetsize{\nodeidx}\big)  \sum_{\left(\featurevec,\truelabel\right) \in \localdatasettrain{\nodeidx}} \big( (1\!-\!\truelabel) \log{\big(1 - h^{(\nodeidx)}\big(\featurevec\big)\big)} \big ).
    \end{aligned}
    \vspace{-2mm}
\end{equation*}

As the stopping criterion in Algorithm \ref{alg1}, we use a fixed number of $\nriter=50$ iterations. The GTV parameter 
has been set to $\regparam=1.0$. We measure the quality of the model parameters learnt by Algorithm \ref{alg1} 
via the accuracy ${\rm ACC}^{(\rm val)}$ achieved on the validation sets $\localdatasetval{\nodeidx}$. 
More precisely, ${\rm ACC}^{(\rm val)}$ is the fraction of correctly classified images in the 
validation sets $\localdatasetval{\nodeidx}$, for $\nodeidx \in \nodesmnist$. 

Figure \ref{fig:lambda_mnist} shows the validation set accuracy achieved by the local model parameters learnt by 
Algorithm \ref{alg1} for different choices of $\regparam$. For the extreme case $\regparam=0$, Algorithm \ref{alg1} 
ignores the edges in $\graph^{(\rm MNIST)}$ and separately minimizes the local loss functions. The other extreme 
case is $\regparam \rightarrow \infty$ where all local model parameter are enforced to be identical, which is equivalent 
to learning a single model on all pooled local datasets. Figure \ref{fig:coonvergence_mnist} compares the ${\rm ACC}^{(\rm val)}$ 
and the convergence rate of Algorithm \ref{alg1} with those of existing methods for personalized FL and for 
clustered FL. Both the ${\rm ACC}^{(\rm val)}$ and the convergence rate of the MNIST dataset is 
improved by Algorithm \ref{alg1}.

\begin{figure}
\begin{center}
\begin{tikzpicture}[scale=0.5]
\begin{axis}[
y label style={at={(axis description cs:0.040,.5)},rotate=0.0,anchor=south},
label style={font=},
title style={font=},
legend style={font=},
yticklabel style = {font=},
xticklabel style = {font=},
xlabel={$\iteridx$},
ylabel={${\rm ACC}^{(\rm val)}$},
title={$\gtvpenalty(\flowvec)\!=\!\|\flowvec\|_{2}$ (nLasso)},
legend pos=south east,
ymajorgrids=true,
grid style=dashed,
]
\addplot+[color=cyan!40!blue!60, very thick, solid, mark=., error bars/.cd, y dir=both, y explicit, error bar style={line width=1pt,solid}, error mark options={line width=1pt,mark size=2pt,rotate=90}] table [x=iter, y=0.0, col sep=comma]{lambda_mnist_journal.csv};\addlegendentry{$\regparam\!=\!0$}
\addplot+[color=orange, very thick, solid, mark=., error bars/.cd, y dir=both, y explicit, error bar style={line width=1pt,solid}, error mark options={line width=1pt,mark size=2pt,rotate=90}] table [x=iter, y=1.0, col sep=comma]{lambda_mnist_journal.csv};\addlegendentry{$\regparam\!=\!1$}
\addplot+[color=green!40!teal!60, very thick, solid, mark=., error bars/.cd, y dir=both, y explicit, error bar style={line width=1pt,solid}, error mark options={line width=1pt,mark size=2pt,rotate=90}] table [x=iter, y=10000.0, col sep=comma]{lambda_mnist_journal.csv};\addlegendentry{$\regparam\!=\!10^{3}$}
\end{axis}
\end{tikzpicture}
\vspace*{-3mm}
\end{center}
\caption{The validation accuracy ${\rm ACC}^{(\rm val)}$ as a function ot the number $\iteridx$ 
	of iterations run by Algorithm \ref{alg1}. Different curves correspond to difference choices for 
	the GTV parameter $\regparam$. \label{fig:lambda_mnist}}
 \vspace{-3mm}
\end{figure}

\begin{figure}[htbp]
\centering
\begin{tikzpicture}[scale=0.7]
\begin{axis}[
y label style={at={(axis description cs:0.040,.5)},rotate=0.0,anchor=south},
label style={font=},
title style={font=},
legend style={font=},
yticklabel style = {font=},
xticklabel style = {font=},
xlabel={iteration $\iteridx$},
ylabel={${\rm ACC}^{(\rm val)}$ },
title={$\regparam=1, \gtvpenalty(\flowvec)\!=\!\|\flowvec\|_{2}$ (nLasso)},
legend pos=south east,
ymajorgrids=true,
grid style=dashed,
]
\addplot+[color=cyan!40!blue!60, very thick, solid, mark=., error bars/.cd, y dir=both, y explicit, error bar style={line width=1pt,solid}, error mark options={line width=1pt,mark size=2pt,rotate=90}] table [x=iter, y=alg1, col sep=comma]{compares_mnist_journal.csv};\addlegendentry{Algorithm \ref{alg1}}
\addplot+[color=orange, very thick, solid, mark=., error bars/.cd, y dir=both, y explicit, error bar style={line width=1pt,solid}, error mark options={line width=1pt,mark size=2pt,rotate=90}] table [x=iter, y=ifca, col sep=comma]{compares_mnist_journal.csv};\addlegendentry{IFCA \cite{Ghosh2020}}
\addplot+[color=green!40!teal!60, very thick, solid, mark=., error bars/.cd, y dir=both, y explicit, error bar style={line width=1pt,solid}, error mark options={line width=1pt,mark size=2pt,rotate=90}] table [x=iter, y=pfedme, col sep=comma]{compares_mnist_journal.csv};\addlegendentry{pFedMe \cite{t2020personalized}}
\addplot+[color=purple!40!red!60, very thick, solid, mark=., error bars/.cd, y dir=both, y explicit, error bar style={line width=1pt,solid}, error mark options={line width=1pt,mark size=2pt,rotate=90}] table [x=iter, y=pytorch, col sep=comma]{compares_mnist_journal.csv};\addlegendentry{PyTorch optimizer}
\end{axis}
\end{tikzpicture}
\caption{Comparing the validation set accuracy ${\rm ACC}^{(\rm val)}$ for two clusters achieved by 
	local model parameters learnt by Algorithm \ref{alg1} with those learnt by IFCA \cite{Ghosh2020}, 
	pFedMe \cite{t2020personalized} and the PyTorch optimizer for \eqref{equ_gtvmin}.}
\label{fig:coonvergence_mnist} 
\end{figure}

Additionally, we have expanded the numerical experiment to incorporate more clusters, 
with each cluster representing a pair of two digits. We use the same setup as during the previous 
experiment, except for the number of clusters. In particular, this experiment revolves 
around a set of $\nrnodes=100$ local datasets generated from the MNIST handwritten image dataset. 
The nodes in $\nodesmnist$ are divided into five clusters, namely $\cluster{1}, \cluster{2}, \cluster{3}, \cluster{4}, \text{and } \cluster{5}$, with each cluster consisting of $20$ nodes. Table \ref{tbl:mnist_five_cluster_images} represents the handwritten digits each cluster contains.
Figure \ref{fig:coonvergence_mnist_5} presents a comparison of the ${\rm ACC}^{(\rm val)}$ and convergence rate 
achieved by Algorithm \ref{alg1} with existing methods for personalized FL and clustered FL. Algorithm \ref{alg1} 
demonstrates improvements in both ${\rm ACC}^{(\rm val)}$ and convergence rate for the MNIST dataset. 
However, it is worth noting that due to lack of computational resources, we could compare against pFedMe \cite{t2020personalized}. 

\begin{table}[h]
    \centering
    \caption{The handwritten digits that each cluster contains.}
    \label{tbl:mnist_five_cluster_images}
     \begin{tabular}{ | p{1.0cm} |  p{0.6cm} | p{0.6cm} |p{0.6cm} |p{0.6cm} | p{0.6cm} | }
        \hline
        Cluster & $\cluster{1}$ & $\cluster{2}$ & $\cluster{3}$ & $\cluster{4}$ & $\cluster{5}$ \\ \hline
        Digits & 0, 1 & 2, 3 & 4, 5 & 6, 7 & 8, 9\\ \hline
        \end{tabular}
    \end{table}

\begin{figure}[htbp]
\begin{center}
\begin{tikzpicture}[scale=0.8]
\begin{axis}[
y label style={at={(axis description cs:0.040,.5)},rotate=0.0,anchor=south},
label style={font=},
title style={font=},
legend style={font=},
yticklabel style = {font=},
xticklabel style = {font=},
xlabel={iteration $\iteridx$},
ylabel={${\rm ACC}^{(\rm val)}$ },
title={$\regparam=1, \gtvpenalty(\flowvec)\!=\!\|\flowvec\|_{2}$ (nLasso)},
legend pos=south east,
ymajorgrids=true,
grid style=dashed,
]
\addplot+[color=cyan!40!blue!60, very thick, solid, mark=., error bars/.cd, y dir=both, y explicit, error bar style={line width=1pt,solid}, error mark options={line width=1pt,mark size=2pt,rotate=90}] table [x=iter, y=alg1, col sep=comma]{compares_mnist_journal_five_clusters.csv};\addlegendentry{Algorithm 1}
\addplot+[color=orange, very thick, solid, mark=., error bars/.cd, y dir=both, y explicit, error bar style={line width=1pt,solid}, error mark options={line width=1pt,mark size=2pt,rotate=90}] table [x=iter, y=ifca, col sep=comma]{compares_mnist_journal_five_clusters.csv};\addlegendentry{IFCA \cite{Ghosh2020}}
\addplot+[color=purple!40!red!60, very thick, solid, mark=., error bars/.cd, y dir=both, y explicit, error bar style={line width=1pt,solid}, error mark options={line width=1pt,mark size=2pt,rotate=90}] table [x=iter, y=pytorch, col sep=comma]{compares_mnist_journal_five_clusters.csv};\addlegendentry{PyTorch optimizer}
\end{axis}
\end{tikzpicture}
\vspace*{-4mm}
\end{center}
\caption{Comparing the validation set accuracy ${\rm ACC}^{(\rm val)}$ for five clusters achieved by local model 
	parameters learnt by Algorithm \ref{alg1} with those learnt by IFCA \cite{Ghosh2020}, and the PyTorch optimizer.}
\label{fig:coonvergence_mnist_5}
\vspace{-3mm}
\end{figure}


\section{Conclusion} 
\label{sec_conclusion}
We have studied GTV minimization methods for the distributed learning of personalized models in 
networked collections of local datasets. GTV minimization is an instance of the RERM principle that 
is obtained using the GTV of local model parameters as the regularization term. This approach is built 
on the assumption that the statistical properties or similarities between collections of local datasets 
are reflected by a known network structure. We obtain a highly scalable FL method by solving GTV 
minimization using a primal-dual method for jointly solving GTV minimization and its dual network flow 
optimization problem. The resulting message-passing algorithm exploits the known 
network structure of data to adaptively pool local datasets into clusters. This pooling allows to learn 
``personalized'' models for local datasets that would not provide sufficient statistical power by itself alone. 
Future research directions include the joint learning of network structure and local model parameters. 
We also aim at a more fine-grained convergence analysis of distributed optimization methods for FL 
that takes into account the precise cluster structure of the data network. Another interesting avenue for 
future research is the extension of our approach to allow for non-parametric local models such as 
decision trees.

\section{Acknowledgement} 
We are grateful to Reza Mirzaeifard who pointed out the Moreau decomposition property 
of proximal operators and provided a careful review of the manuscript. We also like to thank 
Olga Kuznetsova for a careful proof reading of the manuscript. 

\bibliographystyle{IEEEtran}
\bibliography{Literature.bib}


\newpage


{\Large \textbf{Supplementary Material}}
\vspace{0.4cm}


\section{Supplementary Material}

\subsection{Proof of Theorem \ref{thm_main_result}}
\label{app_proof_main_result} 
Our proof technique is an instance of the primal-dual witness technique which has been championed 
for the analysis of Lasso methods \cite[Sec. 11.4.]{HastieWainwrightBook}. A similar instance of this 
technique has been applied to the study of convex clustering which is a special case of GTV minimization \cite[Theorem 5]{JMLR:v22:18-694}, 
In particular, the proof revolves around constructing a pair of optimal networked model 
parameters $\widehat{\weights} \in \nodespace$ and optimal flow vector $\widehat{\flowvec} \in \edgespace$ 
that satisfies the optimality condition \eqref{equ_opt_condition_node_edge_norm}. 
The networked model parameters $\widehat{\weights} \in \nodespace$ assign the local 
model parameters $\estlocalparams{\nodeidx}$ to each node $\nodeidx \in \nodes$. The 
optimal flow vector $\widehat{\flowvec} \in \edgespace$ that assigns each edge the 
flow vector $\optlocalflowvec{\edgeidx}$, for each edge $\edgeidx \in \edges$.

We construct the primal-dual optimal pair for the GTV minimization \eqref{equ_gtvmin} over the 
empirical graph $\graph$ using the solutions of another GTV minimization problem that is defined 
for another empirical graph $\widetilde{\graph} = \pair{\widetilde{\nodes}}{\widetilde{\edges}}$. This 
other empirical graph $\widetilde{\graph}$, which we will refer to as ``cluster graph'', is obtained from 
the original empirical graph $\graph$ by merging all nodes that belong to the same cluster according 
to Assumption \ref{asspt_weights_clustered} (see Figure \ref{fig_cluster_graph}).  

The nodes $\widetilde{\nodes} = \{1,\ldots,\nrcluster\}$ of the cluster graph $\widetilde{\graph}$ represent 
the clusters of the partition \eqref{equ_def_parition_asspt} in Assumption \ref{asspt_weights_clustered}. 
Two different nodes $\clusteridx,\clusteridx' \in \widetilde{\nodes}$ are connected by an edges $\edge{\clusteridx}{\clusteridx'} \in 
\widetilde{\edges}$ if and only if the empirical graph $\graph$ contains two different nodes 
$\nodeidx \in \cluster{\clusteridx}, \nodeidx' \in \cluster{\clusteridx'}$ that are connected by an 
edge $\edge{\nodeidx}{\nodeidx'} \in \edges$. 

Each cluster node $\clusteridx$ carries the sum $\clusterobj{\clusteridx}{\vv} \defeq \sum_{\nodeidx \in \cluster{\clusteridx}}\locallossfunc{\nodeidx}{\vv}$ 
of the local loss function assigned to the nodes $\nodeidx \in \cluster{\clusteridx}$ of the original graph $\graph$ (see our 
clustering assumption in \eqref{equ_def_opt_cluster}). Two cluster nodes $\clusteridx \neq \clusteridx'$ are connected 
by an undirected edge $\edge{\clusteridx}{\clusteridx'}$ with weight 
\begin{equation} 
\label{equ_def_edge_weight_clustered} 
\widetilde{\edgeweight}_{\clusteridx,\clusteridx'} = \sum_{\nodeidx \in \cluster{\clusteridx}, \nodeidx' \in \cluster{\clusteridx'}} \edgeweight_{\nodeidx,\nodeidx'}.
\vspace{-2mm}
\end{equation} 
Thus, each edge in the cluster graph $\widetilde{\graph}$ is obtained by aggregating 
all edges (via summing their weights) in the original empirical graph $\graph$ that connect nodes in cluster $\cluster{\clusteridx}$ 
with nodes in cluster $\cluster{\clusteridx'}$. Figure \ref{fig_cluster_graph} illustrates an example for an empirical graph 
whose nodes are partitioned into clusters along with the corresponding cluster graph. 

The original GTV minimization \eqref{equ_gtvmin} over the empirical graph $\graph$ induces the following 
GTV minimization problem on the cluster graph $\widetilde{\graph}$: 
\begin{equation} 
\label{equ_gtvmin_clustered}
\widetilde{\netparams} \in \underset{\netparams \in \widetilde{\nodespace}}{\mathrm{arg \ min}}\sum_{\clusteridx \in \widetilde{\nodes}} \clusterobj{\clusteridx}{\localparams{\clusteridx}} + \regparam \sum_{\edgeidx \in \widetilde{\edges}} \widetilde{\edgeweight}_{\edgeidx} \norm{\localparams{\edgeidx_{+}} - \localparams{\edgeidx_{-}}}.
\end{equation}
Here, $\widetilde{\nodespace}$ denotes the space of maps $\localparams{\clusteridx}: \widetilde{\nodes} \rightarrow \mathbb{R}^{\dimlocalmodel}$. 
Such a map assign each node $\clusteridx \in \widetilde{\nodes}$ in the cluster graph a parameter 
vector $\localparams{\clusteridx} \in \mathbb{R}^{\dimlocalmodel}$. With slight abuse of notation, 
we refer by $\localparams{\clusteridx}$ to both, the entire map (which is an element of $\widetilde{\nodespace}$), 
and the vector that is assigned to a particular node $\clusteridx \in \widetilde{\nodes}$ by this map. 
Besides the space $\widetilde{\nodespace}$, we associate the space 
 $\widetilde{\edgespace}$ with the cluster graph $\widetilde{\graph}$. This space consists 
of all maps $\flowvec$ that assign each edge $\edgeidx = \edge{\clusteridx}{\clusteridx'} \in \widetilde{\edges}$ 
a vector  $\localflowvec{\edgeidx} \in \mathbb{R}^{\dimlocalmodel}$. 


According to the optimality condition \eqref{equ_opt_condition_node_edge_norm}, any solution $\widetilde{\netparams}$ 
of \eqref{equ_gtvmin_clustered} is characterized by 
\begin{align} 
    & \sum_{\edgeidx \in \widetilde{\edges}:  \edgeidx_{+} =\clusteridx}  \optlocalflowvecclusteredg{\edgeidx} - \sum_{\edgeidx \in \widetilde{\edges}: \edgeidx_{-} = \clusteridx}  \optlocalflowvecclusteredg{\edgeidx}  = - \nabla \clusterobj{\clusteridx}{\estlocalparamsclusteredg{\clusteridx}} \nonumber \\
    & \quad \quad \quad \quad \mbox{ for all nodes } \clusteridx \in \widetilde{\nodes} \label{equ_op_condition_divergence_clusteredg}, \\ 
	& \normgeneric{\optlocalflowvecclusteredg{\edgeidx}}{*}    \leq \regparam  \widetilde{\edgeweight}_{\edgeidx}  \mbox{ for } \edgeidx \in \widetilde{\edges},  \label{equ_op_condition_clusteredg_capacity_constraint} \\ 
	\estlocalparamsclusteredg{\edgeidx_{+}} & = \estlocalparamsclusteredg{\edgeidx_{-}}   \mbox{ for each } \edgeidx \in \widetilde{\edges} \mbox{ with }  \normgeneric{\optlocalflowvecclusteredg{\edgeidx}}{*}   <  \regparam  \edgeweight_{\edgeidx}.  \label{equ_last_cond_clsutered_opt}
\end{align} 
Note that \eqref{equ_op_condition_divergence_clusteredg} - \eqref{equ_last_cond_clsutered_opt} are necessary and sufficient 
for a pair $\estlocalparamsclusteredg{\clusteridx}$, $\optlocalflowvecclusteredg{\edgeidx}$ to solve \eqref{equ_gtvmin_clustered} 
and its dual (which is obtained by adapting \eqref{equ_def_duality_nLasso_edge_node} to the cluster graph $\widetilde{\graph}$), respectively. 
By combining \eqref{equ_op_condition_divergence_clusteredg} with \eqref{equ_op_condition_clusteredg_capacity_constraint}, we 
obtain the necessary condition 
\begin{equation} \label{equ_dondition_gradient_clustered_graph}
    \normgeneric{\nabla \clusterobj{\clusteridx}{\estlocalparamsclusteredg{\clusteridx}}}{*} \leq \regparam \sum_{\clusteridx'} \widetilde{\edgeweight}_{\clusteridx,\clusteridx'},
   \vspace{-2mm}
\end{equation}
which holds for any solution $\estlocalparamsclusteredg{\clusteridx}$ of \eqref{equ_gtvmin_clustered}. 

The upper bound \eqref{equ_dondition_gradient_clustered_graph} on the (norm of the) gradient $ \nabla \clusterobj{\clusteridx}{\estlocalparamsclusteredg{\clusteridx}}$ allows us to upper bound the deviation 
between ${\estlocalparamsclusteredg{\clusteridx}}$ and the cluster-wise optimizer \eqref{equ_def_opt_cluster}. 
To this end, we use the convexity of $ \clusterobj{\clusteridx}{\estlocalparamsclusteredg{\clusteridx}}$ 
(see Assumption \ref{asspt_FIM_lower_bound}) to obtain  
\begin{align}
\clusterobj{\clusteridx}{\estlocalparamsclusteredg{\clusteridx}} - &\clusterobj{\clusteridx}{\overline{\weights}^{(\clusteridx)} }   \leq
 \big( \estlocalparamsclusteredg{\clusteridx} - \overline{\weights}^{(\clusteridx)} \big)^{T}  \nabla \clusterobj{\clusteridx}{\estlocalparamsclusteredg{\clusteridx}} 
 \nonumber \\ 
 & \stackrel{(a)}{\leq}  \normgeneric{ \estlocalparamsclusteredg{\clusteridx} - \overline{\weights}^{(\clusteridx)} }{} \normgeneric{ \nabla \clusterobj{\clusteridx}{\estlocalparamsclusteredg{\clusteridx}} }{*} \nonumber \\ 
 & \stackrel{\eqref{equ_dondition_gradient_clustered_graph}}{\leq}  \normgeneric{ \estlocalparamsclusteredg{\clusteridx} - \overline{\weights}^{(\clusteridx)} }{}
 \regparam \sum_{\clusteridx' \in \widetilde{\nodes} \setminus \{ \clusteridx \}} \widetilde{\edgeweight}_{\clusteridx,\clusteridx'}.  \label{equ_upper_bound_cluster_opt_convx}
\end{align}
Here, step $(a)$ uses the definition of the dual norm \cite{BoydConvexBook}. Next, we use the 
strong convexity of $\clusterobj{\clusteridx}{\cdot}$ (see \eqref{equ_strong_convexit}) to lower 
bound the LHS of \eqref{equ_upper_bound_cluster_opt_convx} as
\begin{align} 
\clusterobj{\clusteridx}{\estlocalparamsclusteredg{\clusteridx}} - \clusterobj{\clusteridx}{\overline{\weights}^{(\clusteridx)} }  \geq (\strongconvparam{\clusteridx}/2)  \normgeneric{\estlocalparamsclusteredg{\clusteridx} - \overline{\weights}^{(\clusteridx)}}{}^{2}. 
\label{equ_lower_bound_cluster_opt_objective_strong_convx}
\end{align} 
Combining \eqref{equ_lower_bound_cluster_opt_objective_strong_convx} with \eqref{equ_upper_bound_cluster_opt_convx} yields 
\begin{equation}
	\label{equ_def_upper_bound_clustered_custerwide}
	\normgeneric{\estlocalparamsclusteredg{\clusteridx} - \overline{\weights}^{(\clusteridx)}}{} \leq  (2\regparam/\strongconvparam{\clusteridx}) \sum_{\clusteridx' \in \widetilde{\nodes} \setminus \{\clusteridx \}} \widetilde{\edgeweight}_{\clusteridx,\clusteridx'}.
 \vspace{-2mm}
\end{equation} 


To recap: We have introduced the cluster graph $\widetilde{\graph}$ by merging the nodes 
of the empirical graph $\graph$ according to the clustering assumption (see Assumption \ref{asspt_weights_clustered}). 
This cluster graph is used to define the ``cluster-wise'' GTV minimization problem \eqref{equ_gtvmin_clustered} 
which is obtained from the original GTV minimization \eqref{equ_gtvmin} by enforcing identical local model parameters 
for all nodes in the same cluster of $\graph$. We have then derived the upper bound \eqref{equ_def_upper_bound_clustered_custerwide} 
on the deviation between any solution $\estlocalparamsclusteredg{\clusteridx}$ of  \eqref{equ_gtvmin_clustered} 
and the cluster-wise optimizers \eqref{equ_def_opt_cluster}. 

The main step in the proof of Theorem \ref{thm_main_result} is to show that the local model parameters  
\begin{equation}
\label{equ_def_optimal_weight_bsed_clustered_solutions} 
\begin{aligned}
     \estlocalparams{\nodeidx} & \defeq  \estlocalparamsclusteredg{\clusteridx} \\
     \mbox{ for all nodes } \nodeidx \in \cluster{\clusteridx} &
\mbox{ in the same cluster } \cluster{\clusteridx} \in \partition\\
& \mbox{ (see Assumption\ \ref{asspt_weights_clustered}) } 
\end{aligned}
\vspace{-2mm}
\end{equation} 
solve the original GTV minimization problem \eqref{equ_gtvmin}. To this end, we construct a flow 
$\optlocalflowvec{\edgeidx}$, for each edge $\edgeidx \in \edges$ in the original empirical graph $\graph$, 
that together with $\estlocalparams{\nodeidx}$ defined by \eqref{equ_def_optimal_weight_bsed_clustered_solutions}, 
solves the optimality conditions \eqref{equ_opt_condition_node_edge_norm}. 

We construct the optimal flow $\optlocalflowvec{\edgeidx}$ by considering separately the inter-cluster edges 
(connecting nodes in different clusters) and the intra-cluster edges (connecting nodes in the same cluster). The 
inter-cluster (or boundary) edges are 
$$ \partial \partition \defeq \left\{ \edge{\nodeidx}{\nodeidx'} \in \edges, \mbox{ with } \nodeidx \in \cluster{\clusteridx}, \nodeidx' \in \cluster{\clusteridx'} \right\}.\vspace{-2mm}$$ 
The intra-cluster edges $\edges \setminus \partial \partition$ connect nodes in the same cluster, i.e., they 
are given by all edges $\edge{\nodeidx}{\nodeidx'}$ with $\nodeidx, \nodeidx' \in \cluster{\clusteridx}$ for 
some $\clusteridx \in \{1,\ldots,\nrcluster\}$. 

To each inter-cluster edge $\edgeidx = \edge{\nodeidx}{\nodeidx'} \in \partial \partition$, we assign the flow value 
\begin{equation} 
\label{equ_def_inter_cluster_edge_local_flow}
\optlocalflowvec{\edgeidx} \defeq \optlocalflowvecclusteredg{\edge{\clusteridx}{\clusteridx'}} \big( \edgeweight_{\nodeidx,\nodeidx'} / \widetilde{\edgeweight}_{\clusteridx,\clusteridx'} \big) \mbox{ (see \eqref{equ_def_edge_weight_clustered})}.
\vspace{-2mm}
\end{equation} 
Note that, for any two different cluster nodes $\clusteridx,\clusteridx'$ (representing the clusters $\cluster{\clusteridx}$ 
and $\cluster{\clusteridx'}$ in the original empirical graph $\graph$), 
\begin{align} 
\label{equ_def_cap_constraint_inter_cluster_edge}
\normgeneric{\optlocalflowvec{\edgeidx}}{*} &  =  \normgeneric{\optlocalflowvecclusteredg{\edge{\clusteridx}{\clusteridx'}}}{*} \big( \edgeweight_{\nodeidx,\nodeidx'} / \widetilde{\edgeweight}_{\clusteridx,\clusteridx'} \big)  \stackrel{\eqref{equ_op_condition_clusteredg_capacity_constraint}}{\leq}  \regparam \edgeweight_{\nodeidx,\nodeidx'} \nonumber \\
& \mbox{ for any } \edgeidx =  \edge{\nodeidx}{\nodeidx'} \mbox{ with } \nodeidx \in \cluster{\clusteridx}, \nodeidx' \in \cluster{\clusteridx'}.
\vspace{-2mm}
\end{align} 
Moreover, if the cluster-wise GTV minimization \eqref{equ_gtvmin_clustered} delivers different  
parameter vectors $\estlocalparamsclusteredg{\clusteridx} \neq \estlocalparamsclusteredg{\clusteridx'}$ 
for two cluster nodes $\clusteridx,\clusteridx' \in \widetilde{\nodes}$, then \eqref{equ_last_cond_clsutered_opt}
requires that the edge $\edge{\clusteridx}{\clusteridx'}$ must be saturated in the sense of 
$\normgeneric{\optlocalflowvecclusteredg{\edge{\clusteridx}{\clusteridx'}}}{*} = \regparam \widetilde{\edgeweight}_{\clusteridx,\clusteridx'}$. 
This implies, in turn via \eqref{equ_def_inter_cluster_edge_local_flow}, that  
\begin{equation} 
\label{equ_proof_saturated_inter_cluster_edges}
\begin{aligned}
    & \normgeneric{\optlocalflowvec{\edge{\nodeidx}{\nodeidx'}}}{*} = \regparam  \edgeweight_{\nodeidx,\nodeidx'} \\
    \mbox{ for any } \nodeidx \in \cluster{\clusteridx}, & \nodeidx' \in \cluster{\clusteridx'} \mbox{ such that } \estlocalparamsclusteredg{\clusteridx} \neq \estlocalparamsclusteredg{\clusteridx'}. 
\end{aligned}
\vspace{-2mm}
\end{equation} 
Note that \eqref{equ_proof_saturated_inter_cluster_edges} is equivalent to the third condition in 
\eqref{equ_opt_condition_node_edge_norm} for the optimal flow values across 
inter-cluster edges $\edgeidx \in \partial \partition$. 

We next construct the flow values $\optlocalflowvec{\edgeidx}$ for the intra-cluster edges $\edgeidx \in \edges \setminus \partial \partition$ 
such that the optimality conditions \eqref{equ_opt_condition_node_edge_norm} are satisfied. To this end, we consider 
a specific cluster $\cluster{\clusteridx} \in \partition$ and the corresponding induced sub-graph $\graph^{(\clusteridx)}=\pair{\cluster{\clusteridx}}{\edges^{(\clusteridx)}}$ 
of the empirical graph $\graph$. Together with \eqref{equ_def_cap_constraint_inter_cluster_edge} 
and \eqref{equ_proof_saturated_inter_cluster_edges}, a sufficient condition for \eqref{equ_opt_condition_node_edge_norm} 
to hold is that, for each cluster sub-graph $\graph^{(\clusteridx)}$, 
\begin{align} \label{equ_opt_condition_node_edge_norm_intra_cluster}
	\sum_{\edgeidx \in \edges^{(\clusteridx)}:\nodeidx\!=\!\edgeidx_{+}} \hspace*{-3mm}\optlocalflowvec{\edgeidx} & - \sum_{\edgeidx \in \edges^{(\clusteridx)}:\nodeidx\!=\!\edgeidx_{-}}  \hspace*{-3mm}\optlocalflowvec{\edgeidx}  = \nonumber \\
    & - \hspace*{-3mm} \sum_{\edgeidx \in  \partial \partition: \nodeidx\!=\!\edgeidx_{+}} \hspace*{-3mm}\optlocalflowvec{\edgeidx} + \sum_{\edgeidx \in \partial \partition:\nodeidx\!=\!\edgeidx_{-}} \hspace*{-3mm} \optlocalflowvec{\edgeidx} - \nabla \locallossfunc{\nodeidx}{\estlocalparams{\nodeidx}} \nonumber \\
    & \quad \quad \mbox{ for each node } \nodeidx \in \cluster{\clusteridx}\mbox{, and } \nonumber\\ 
	\| \optlocalflowvec{\edgeidx} \|_{*}   & < \regparam  \edgeweight_{\edgeidx}   \mbox{ for each edge } \edgeidx \in \edges^{(\clusteridx)}. 
 \vspace{-2mm}
\end{align} 


Regarding the first (flow conservation) condition in \eqref{equ_opt_condition_node_edge_norm_intra_cluster}, 
we note that 
\begin{align}  \label{equ_def_bound_gradient_node_conservation}
& \normgeneric{\nabla  \locallossfunc{\nodeidx}{\estlocalparams{\nodeidx}}}{*}= \nonumber \\
& \quad \normgeneric{\nabla \locallossfunc{\nodeidx}{\estlocalparams{\nodeidx}} -  \nabla \locallossfunc{\nodeidx}{\clusterwideopt{\clusteridx}} +  \nabla \locallossfunc{\nodeidx}{\clusterwideopt{\clusteridx}}}{*} \nonumber \\ 
& \quad \stackrel{(a)}{\leq} \normgeneric{\nabla \locallossfunc{\nodeidx}{\estlocalparams{\nodeidx}} -  \nabla \locallossfunc{\nodeidx}{\clusterwideopt{\clusteridx}}}{*} +   \normgeneric{\nabla \locallossfunc{\nodeidx}{\clusterwideopt{\clusteridx}}}{*} \nonumber \\ 
& \quad \stackrel{\eqref{equ_upper_bound_norm_gradient}}{\leq}\normgeneric{\nabla \locallossfunc{\nodeidx}{\estlocalparams{\nodeidx}} -  \nabla \locallossfunc{\nodeidx}{\clusterwideopt{\clusteridx}}}{*} + \clusteropterr{\nodeidx} \nonumber \\  
& \quad \stackrel{\eqref{equ_def_cond_locallipsch}}{\leq} \locallipsch{\nodeidx}\norm{\estlocalparams{\nodeidx} -\clusterwideopt{\clusteridx}} + \clusteropterr{\nodeidx} \nonumber \\ 
& \quad \stackrel{\eqref{equ_def_upper_bound_clustered_custerwide}}{\leq} \locallipsch{\nodeidx} (2\regparam/\strongconvparam{\clusteridx})|\partial \cluster{\clusteridx}| + \clusteropterr{\nodeidx} \nonumber \\
& \quad \mbox{ for each node } \nodeidx \in \cluster{\clusteridx}.
\end{align} 
Here, step $(a)$ in \eqref{equ_def_bound_gradient_node_conservation} used the triangle inequality 
for the norm $\normgeneric{\cdot}{*}$ \cite{Golub1980}. Another application of the triangle inequality 
yields, for any node $\nodeidx \in \cluster{\clusteridx}$, 
\begin{align}
\label{equ_def_bound_inter_cluster_edges_flow_bound}
\normgeneric{- \hspace*{-3mm} \sum_{\edgeidx \in  \partial \partition: \nodeidx\!=\!\edgeidx_{+}} \hspace*{-3mm}\optlocalflowvec{\edgeidx} + \sum_{\edgeidx \in \partial \partition:\nodeidx\!=\!\edgeidx_{-}} \hspace*{-3mm} \optlocalflowvec{\edgeidx}}{*} &\leq \sum_{\nodeidx' \in \neighbourhood{\nodeidx} \cap \nodes \setminus \cluster{\clusteridx}} 
\normgeneric{\optlocalflowvec{\edge{\nodeidx}{\nodeidx'}}}{*} \nonumber \\
& \stackrel{\eqref{equ_def_cap_constraint_inter_cluster_edge}}{\leq} \sum_{\nodeidx' \in \neighbourhood{\nodeidx} \cap \nodes \setminus \cluster{\clusteridx}} \regparam \edgeweight_{\nodeidx,\nodeidx'} \nonumber \\ 
& \leq \sum_{\nodeidx' \notin \cluster{\clusteridx}} \regparam \edgeweight_{\nodeidx,\nodeidx'}. 
\end{align} 

Let us rewrite the requirement \eqref{equ_opt_condition_node_edge_norm_intra_cluster}, for some cluster $\cluster{\clusteridx}$, as 
\begin{align} 
	\label{equ_opt_condition_node_edge_norm_intra_cluster_demands}
	\sum_{\edgeidx \in \edges_{\clusteridx}:\nodeidx\!=\!\edgeidx_{+}}\optlocalflowvec{\edgeidx} - \sum_{\edgeidx \in \edges_{\clusteridx}:\nodeidx\!=\!\edgeidx_{-}}  \hspace*{-3mm}\optlocalflowvec{\edgeidx}  & = \demand{\nodeidx} \mbox{ for each node } \nodeidx \in \cluster{\clusteridx}, \nonumber\\ 
    \mbox{ and } \| \optlocalflowvec{\edgeidx} \|_{*} < \regparam  & \edgeweight_{\edgeidx}   \mbox{ for each edge } \edgeidx \in \edges^{(\clusteridx)}. 
\end{align} 
with vector-valued demands $ \demand{\nodeidx} = - \sum_{\edgeidx \in  \partial \partition: \nodeidx\!=\!\edgeidx_{+}} \hspace*{-3mm}\optlocalflowvec{\edgeidx} + \sum_{\edgeidx \in \partial \partition:\nodeidx\!=\!\edgeidx_{-}} \hspace*{-3mm} \optlocalflowvec{\edgeidx} - \nabla \locallossfunc{\nodeidx}{\estlocalparams{\nodeidx}}$. 
According to \eqref{equ_def_bound_gradient_node_conservation} and \eqref{equ_def_bound_inter_cluster_edges_flow_bound}, we have 
$$ \normgeneric{\demand{\nodeidx}}{*} \leq   \sum_{\nodeidx' \notin \cluster{\clusteridx}} \regparam \edgeweight_{\nodeidx,\nodeidx'} + \locallipsch{\nodeidx} (2\regparam/\strongconvparam{\clusteridx})|\partial \cluster{\clusteridx}| + \clusteropterr{\nodeidx}.$$
Since Theorem \ref{thm_main_result} requires each cluster $\cluster{\clusteridx}$ to be well-connected 
(see \eqref{equ_def_well_connected_condition}), we can apply Lemma \ref{lem_main_helper} 
to verify the existence of a flow $\optlocalflowvec{\edgeidx} $ that satisfies \eqref{equ_opt_condition_node_edge_norm_intra_cluster}. 

\begin{figure} 
\tikzset{global scale/.style={
		scale=#1,
		every node/.append style={scale=#1}
	}
}
\begin{center}
\begin{tikzpicture}[global scale = 0.75,
	dashedcircle/.style={circle, dash pattern=on 1pt off 2pt on 3pt off 2pt, draw=blue!80!cyan, minimum size=2cm},
	filledcircle/.style={circle, minimum size=0.1cm, draw=blue!80!cyan, fill=blue!70!cyan},
	edge/.style={color=blue!80!cyan, line width=1pt}]
	
	\draw (-6,-3) -- (-0.5,-3) -- (-0.5,3) -- (-6,3) -- (-6,-3);
	
        \node[dashedcircle](dc1) at (-4,1.3) {} node[above = of dc1, xshift=-0.2cm, yshift=-1.3cm, color=black] {$\cluster{1}$};
	\node[dashedcircle](dc2) [right = of dc1, xshift=-1.4cm, yshift=-1.2cm] {} node[above = of dc2, xshift=0.2cm, yshift=-1.2cm, color=black] {$\cluster{2}$};
	\node[dashedcircle](dc3) [below = of dc1, yshift=0.28cm, yshift=0.1cm] {} node[right = of dc3, xshift=-1.2cm, yshift=-0.3cm, color=black] {$\cluster{3}$};
	
	\node[filledcircle](fc1) [above = of dc1, xshift=0.3cm, yshift=-1.8cm] {};
	\node[filledcircle](fc2) [left = of fc1, xshift=0.6cm, yshift=-0.7cm] {} node[left = of fc2, xshift=1.4cm, yshift=0.2cm] {$\nodeidx'$};
	\node[filledcircle](fc3) [below = of fc1, xshift=0cm, yshift=0.2cm] {};
	
	\node[filledcircle](fc4) [below = of fc1, yshift=-1.3cm] {};
	\node[filledcircle](fc5) [below = of fc2, yshift=-1cm] {} node[left = of fc5, xshift=1.3cm, yshift=-0.1cm] {$\nodeidx$};
	\node[filledcircle](fc6) [below = of fc3, yshift=-1.4cm] {};
	
	\node[filledcircle](fc7) [below = of fc1, xshift=1.5cm, yshift=0.5cm] {};
	\node[filledcircle](fc8) [below = of fc2, xshift=3cm, yshift=0.5cm] {};
	\node[filledcircle](fc9) [below = of fc3, xshift=1.5cm, yshift=0.5cm] {};
	
	\draw[edge] (fc1) -- (fc2);
	\draw[edge] (fc1) -- (fc3);
	\draw[edge] (fc2) -- (fc3);
	
	\draw[edge] (fc4) -- (fc5);
	\draw[edge] (fc4) -- (fc6);
	\draw[edge] (fc5) -- (fc6);
	
	\draw[edge] (fc7) -- (fc8);
	\draw[edge] (fc7) -- (fc9);
	\draw[edge] (fc8) -- (fc9);
	
	\draw[edge] (fc2) -- (fc5) node[left = of fc5, xshift=1.5cm, yshift=1cm, color=black] {$\edgeweight_{\nodeidx,\nodeidx'}$};
	\draw[edge] (fc3) -- (fc4);
	\draw[edge] (fc3) -- (fc7);
	\draw[edge] (fc4) -- (fc9);
	
	\draw (6,-3) -- (0.5,-3) -- (0.5,3) -- (6,3) -- (6,-3);
	
	\node[filledcircle](fc10) [right = of fc1, xshift=5cm, yshift=-1.0cm] {} node[left = of fc10, xshift=1.9cm, yshift=0.4cm] {$\clusteridx\!=\!1$};
	\node[filledcircle](fc11) [below = of fc10, xshift=-0.5cm, yshift=-0.6cm] {} node[below = of fc11, yshift=1.2cm] {$\clusteridx'\!=\!3$};
	\node[filledcircle](fc12) [below = of fc10, xshift=1.5cm, yshift=0.3cm] {} node[right = of fc12, xshift=-0.2cm] {$$};
	
	\draw[edge] (fc10) -- (fc11) node[left = of fc10, xshift=1.3cm, yshift=-0.8cm, color=black] {$\widetilde{\edgeweight}_{\clusteridx,\clusteridx'}$};
	\draw[edge] (fc10) -- (fc12);
	\draw[edge] (fc11) -- (fc12);
	
\end{tikzpicture}
\end{center}
\caption{Left: Empirical graph $\graph$ of a networked dataset whose nodes are partitioned into 
	three clusters $\partition = \{ \cluster{1},\ldots,\cluster{\nrcluster} \}$. Right: Cluster graph $\widetilde{\graph}$ 
	whose nodes represent the clusters in $\partition$.} 
\label{fig_cluster_graph}
\end{figure} 


\begin{lemma}[Main Helper] 
	\label{lem_main_helper}
	Consider an empirical graph $\graph' = \big(\nodes',\edges' \big)$ with edge weights $\edgeweight'_{\edgeidx}$. 
	We orient the edges  $\edgeidx = \{\nodeidx,\nodeidx'\} \in \edges'$ by defining its 
	head and tail as $\edgeidx_{+} =\min \{\nodeidx,\nodeidx' \}$ and $\edgeidx_{-} =\max \{\nodeidx,\nodeidx' \}$, respectively.  
	Each node $\nodeidx \in \nodes' = \{1,\ldots,\nrnodes'\}$ carries a (``demand'') vector $\demand{\nodeidx} \in \mathbb{R}^{\dimlocalmodel}$ 
	such that $\sum_{\nodeidx \in \nodes'} \demand{\nodeidx} = \mathbf{0}$ and 
	$\normgeneric{\demand{\nodeidx}}{} \leq \demandbound{\nodeidx}$ with prescribed 
	upper bounds $\demandbound{\nodeidx} \in \mathbb{R}_{+}$. 
	
	If there is a node $\nodeidx_{0} \in \nodes'$ such that 
	\begin{equation} 
		\label{equ_condition_helper_lem}
		\sum_{\nodeidx \in \genericnodeset} \demandbound{\nodeidx} < \sum_{\nodeidx \in \genericnodeset}  \sum_{\nodeidx' \in \nodes' \setminus \genericnodeset } 
		\edgeweight'_{\nodeidx,\nodeidx'}  \mbox{ for all subsets } \genericnodeset \subseteq \nodes' \setminus \big\{ \nodeidx_{0} \big\},
	\end{equation} 
    we can construct a flow $\localflowvec{\edgeidx} \in \mathbb{R}^{\dimlocalmodel}$, for each edge $\edgeidx$, such that 
\begin{align} 
	\label{equ_flow_converv_capacity_main_helper}
	\sum_{\edgeidx \in \edges':\nodeidx\!=\!\edgeidx_{+}} \localflowvec{\edgeidx} - \sum_{\edgeidx \in \edges':\nodeidx\!=\!\edgeidx_{-}}  \localflowvec{\edgeidx}  & = \demand{\nodeidx} \mbox{ for all nodes } \nodeidx \in \nodes' \nonumber, \\ 
	\normgeneric{ \localflowvec{\edgeidx}}{}  <  \edgeweight'_{\edgeidx}  & \mbox{ for each } \edgeidx \in \edges'. 
\end{align}
\end{lemma}
 \begin{proof}
 We first apply the \emph{Feasible Distribution Theorem} \cite[p. 71]{RockNetworks} to \eqref{equ_condition_helper_lem} 
 which yields the existence of a scalar-valued flow $\localflowscalar{\edgeidx}$ such that 
 \begin{align} 
 	\label{equ_flow_converv_capacity_main_helper}
 	\sum_{\edgeidx \in \edges':\nodeidx_{0} \!=\!\edgeidx_{+}} \localflowscalar{\edgeidx} - \sum_{\edgeidx \in \edges':\nodeidx_{0} \!=\!\edgeidx_{-}}  \localflowscalar{\edgeidx}  & = - \sum_{\nodeidx \in \nodes' \setminus \{\nodeidx_{0}\}} \normgeneric{\demand{\nodeidx}}{}  \nonumber, \\ 
 	 	\sum_{\edgeidx \in \edges':\nodeidx\!=\!\edgeidx_{+}} \localflowscalar{\edgeidx} - \sum_{\edgeidx \in \edges':\nodeidx\!=\!\edgeidx_{-}}  \localflowscalar{\edgeidx}  & = \normgeneric{\demand{\nodeidx}}{} \nonumber \\
    \mbox{ for all nodes } \nodeidx \in \nodes' & \setminus \left\{ \nodeidx_{0} \right\} \mbox{, and } \nonumber \\ 
 	| \localflowscalar{\edgeidx}|  <  \edgeweight'_{\edgeidx}  \mbox{ for every} & \mbox{ edge } \edgeidx \in \edges'. 
 \end{align}
Next, we decompose the flow $\localflowscalar{\edgeidx}$ using the \emph{Conformal Decomposition Theorem} \cite[Prop 1.1]{BertsekasNetworkOpt} 
into elementary (scalar-valued) flows $\elemflowscalar{\cycleidx}{\edgeidx} \in \{0,1\}$ with the index $\cycleidx$ representing a simple path, 
\begin{equation}
	\label{equ_def_conf_flow_decomp}
\localflowscalar{\edgeidx} = \sum_{\cycleidx=1}^{\nrnodes'+|\edges'|} a_{\cycleidx} \elemflowscalar{\cycleidx}{\edgeidx}. 
\end{equation}
Each elementary flow $\elemflowscalar{\cycleidx}{\edgeidx}$ in \eqref{equ_def_conf_flow_decomp} is associated with 
separate path $\simplepath{\cycleidx}$ that starts from a node in $\nodes \setminus \{\nodeidx_{0} \}$ and 
terminates at the node $\nodeidx_{0}$. Moreover, each flow $ \elemflowscalar{\cycleidx}{\edgeidx}$ 
in \eqref{equ_def_conf_flow_decomp} conforms to the flow $\localflowscalar{\edgeidx}$ in the 
sense of \cite[p.8]{BertsekasNetworkOpt}
\begin{align} 
\localflowscalar{\edgeidx} & > 0  \mbox{ for any edge } \edgeidx \in \edges' \mbox{ with } \elemflowscalar{\cycleidx}{\edgeidx} > 0 \mbox{, and } \nonumber \\
\localflowscalar{\edgeidx} & <0  \mbox{ for any edge } \edgeidx \in \edges' \mbox{ with } \elemflowscalar{\cycleidx}{\edgeidx}< 0.
\end{align}
Let the set $\mathcal{S}^{(\nodeidx)}$ collect the indices $\cycleidx \in \{1,\ldots,\nrnodes+| \edges| \}$ 
of all paths $\simplepath{\cycleidx}$ starting at node $\nodeidx \in \nodes' \setminus \{ \nodeidx_{0} \}$. 
The final step of the proof is to multiply the elementary flows $\elemflowscalar{\cycleidx}{\edgeidx}$, 
for each path $\cycleidx \in \mathcal{S}^{(\nodeidx)}$, with the vector-valued demand $\demand{\nodeidx}$, 
resulting in the vector-valued flow 
\begin{equation} 
\localflowvec{\edgeidx}  = \sum_{\nodeidx \in \nodes \setminus \{\nodeidx_{0}\}} \demand{\nodeidx} \sum_{\cycleidx \in \mathcal{S}^{(\nodeidx)}} \elemflowscalar{\cycleidx}{\edgeidx}.
\end{equation} 
 \end{proof}

\subsection{Additional Numerical Experiments}
\label{sec_more_numexp}

\subsubsection{SBM Low-Dimensional Linear Regression}

Figure \ref{fig_noiseless} and \ref{fig_fixes_trin_set} depict results for data with empirical graph $\graph^{(\rm SBM)}$ 
consisting of two clusters, $|\cluster{1}| = |\cluster{2}| = 100$. The intra-cluster edge probability is fixed to 
$p_{\rm in}\!=\!0.5$, but different choices for the inter-cluster edge probability $p_{\rm out}$ are used. The local 
datasets $\localdataset{\nodeidx}$ consist of $\localsamplesize{\nodeidx}\!=\!5$ data points with feature vectors 
$\featurevec^{(\nodeidx,\localsampleidx)} \in \mathbb{R}^{2}$ and labels generated via \eqref{equ_def_true_linear_model_SBM} 
where $\overline{\weights}^{(1)} = \big(2, 2\big)^T$ and $\overline{\weights}^{(2)} = \big(-2, 2\big)^T$. 
For each choice of $\rho$ (fraction of accessible local datasets), $\sigma$ (label noise) and $p_{\rm out}$ (inter-cluster edge probability), 
we draw $5$ realizations of $\graph^{(\rm SBM)}$ and compute the average and standard deviation of the 
resulting MSE values \eqref{equ_def_MSE} obtained for each realization. Figure \ref{fig_noiseless} depicts the 
average and standard deviation for the MSE values in the noise-less case where $\sigma=0$ in \eqref{equ_def_true_linear_model_SBM}. Figure \ref{fig_fixes_trin_set} 
depicts average and standard deviation of the MSE obtained when having access to a fraction of $\rho=0.6$ local datasets. 

\begin{figure}[htbp]
\begin{minipage}[t]{0.3\columnwidth} 
\begin{tikzpicture}[scale=0.5]
\begin{axis}[
ymode=log, ymin=1e-16, ymax=50,
y label style={at={(axis description cs:0.040,.5)},rotate=0.0,anchor=south},
label style={font=\Large},
title style={font=\Large},
legend style={font=\Large},
yticklabel style = {font=\small},
xticklabel style = {font=\Large},
xlabel={$p_{\rm out}$},
ylabel={MSE},
title={$\gtvpenalty(\flowvec)\!=\!\|\flowvec\|_{1},\regparam\!=\!10^{-2}$},
legend pos=south east,
ymajorgrids=true,
grid style=dashed,
]
\addplot+[color=cyan!40!blue!60, very thick, solid, mark=*, error bars/.cd, y dir=both, y explicit, error bar style={line width=1pt,solid}, error mark options={line width=1pt,mark size=2pt,rotate=90}] table [x=Pout, y=N1M02MSE, y error=N1M02STD, col sep=comma]{differentSamplingRatios.csv};\addlegendentry{$\rho\!=\!2/10$}
\addplot+[color=orange, very thick, densely dotted, mark=*, error bars/.cd, y dir=both, y explicit, error bar style={line width=1pt,solid}, error mark options={line width=1pt,mark size=2pt,rotate=90}] table [x=Pout, y=N1M04MSE, y error=N1M04STD, col sep=comma]{differentSamplingRatios.csv};\addlegendentry{$\rho\!=\!4/10$}
\addplot+[color=green!40!teal!60, very thick, densely dashed, mark=*, error bars/.cd, y dir=both, y explicit, error bar style={line width=1pt,solid}, error mark options={line width=1pt,mark size=2pt,rotate=90}] table [x=Pout, y=N1M06MSE, y error=N1M06STD, col sep=comma]{differentSamplingRatios.csv};\addlegendentry{$\rho\!=\!6/10$}
\end{axis}
\end{tikzpicture}
\end{minipage}
\hspace*{15mm}
\begin{minipage}[t]{0.3\columnwidth} 
\begin{tikzpicture}[scale=0.5]
\begin{axis}[
ymode=log,ymin=1e-16, ymax=50,
y label style={at={(axis description cs:0.040,.5)},rotate=0.0,anchor=south},
label style={font=\Large},
title style={font=\Large},
legend style={font=\Large},
yticklabel style = {font=\small},
xticklabel style = {font=\Large},
xlabel={$p_{\rm out}$},
ylabel={MSE},
title={$\gtvpenalty(\flowvec)\!=\!\|\flowvec\|_{2},\regparam\!=\!10^{-2}$ (nLasso)},
legend pos=south east,
ymajorgrids=true,
grid style=dashed,
]

\addplot+[color=cyan!40!blue!60, very thick, solid, mark=*, error bars/.cd, y dir=both, y explicit, error bar style={line width=1pt,solid}, error mark options={line width=1pt,mark size=2pt,rotate=90}] table [x=Pout, y=N2M02MSE, y error=N2M02STD, col sep=comma]{differentSamplingRatios.csv};\addlegendentry{$\rho\!=\!2/10$}
\addplot+[color=orange, very thick, densely dotted, mark=*, error bars/.cd, y dir=both, y explicit, error bar style={line width=1pt,solid}, error mark options={line width=1pt,mark size=2pt,rotate=90}] table [x=Pout, y=N2M04MSE, y error=N2M04STD, col sep=comma]{differentSamplingRatios.csv};\addlegendentry{$\rho\!=\!4/10$}
\addplot+[color=green!40!teal!60, very thick, densely dashed, mark=*, error bars/.cd, y dir=both, y explicit, error bar style={line width=1pt,solid}, error mark options={line width=1pt,mark size=2pt,rotate=90}] table [x=Pout, y=N2M06MSE, y error=N2M06STD, col sep=comma]{differentSamplingRatios.csv};\addlegendentry{$\rho\!=\!6/10$}
\end{axis}
\end{tikzpicture}
\end{minipage}
\\
\\
\hspace*{22mm}
\begin{minipage}[t]{0.3\columnwidth} 
\begin{tikzpicture}[scale=0.5]
\begin{axis}[
ymode=log,ymin=1e-16, ymax=50,
y label style={at={(axis description cs:0.040,.5)},rotate=0.0,anchor=south},
label style={font=\Large},
title style={font=\Large},
legend style={font=\Large},
yticklabel style = {font=\small},
xticklabel style = {font=\Large},
xlabel={$p_{\rm out}$},
ylabel={MSE},
title={$\gtvpenalty(\flowvec)\!=\!\|\flowvec\|_{2}^{2},\regparam\!=\!10^{-2}$ (``MOCHA'')},
legend pos=south east,
ymajorgrids=true,
grid style=dashed,
]


\addplot+[color=cyan!40!blue!60, very thick, solid, mark=*, error bars/.cd, y dir=both, y explicit, error bar style={line width=1pt,solid}, error mark options={line width=1pt,mark size=2pt,rotate=90}] table [x=Pout, y=MCM02MSE, y error=MCM02STD, col sep=comma]{differentSamplingRatios.csv};\addlegendentry{$\rho\!=\!0.2$}
\addplot+[color=orange, very thick, densely dotted, mark=*, error bars/.cd, y dir=both, y explicit, error bar style={line width=1pt,solid}, error mark options={line width=1pt,mark size=2pt,rotate=90}] table [x=Pout, y=MCM04MSE, y error=MCM04STD, col sep=comma]{differentSamplingRatios.csv};\addlegendentry{$\rho\!=\!0.4$}
\addplot+[color=green!40!teal!60, very thick, densely dashed, mark=*, error bars/.cd, y dir=both, y explicit, error bar style={line width=1pt,solid}, error mark options={line width=1pt,mark size=2pt,rotate=90}] table [x=Pout, y=MCM06MSE, y error=MCM06STD, col sep=comma]{differentSamplingRatios.csv};\addlegendentry{$\rho\!=\!0.6$}

\end{axis}
\end{tikzpicture}
\end{minipage}
\caption{\label{fig_noiseless} MSE \eqref{equ_def_MSE} incurred by Algorithm \ref{alg1} used to learn the local  
	model parameters \eqref{equ_def_true_linear_model_SBM} of $\graph^{(\rm SBM)}$. These results are obtained 
	for local datasets obtained from a noise-free linear model ($\sigma=0$ in \eqref{equ_def_true_linear_model_SBM}) 
	and which are available only for the nodes in the subset $\trainingset$(see \eqref{equ_def_local_loss_squared}) 
	of relative size $\rho = |\trainingset|/|\nodes|$.} 
\end{figure}

\begin{figure}[htbp]
\begin{minipage}[t]{0.3\columnwidth} 
\begin{tikzpicture}[scale=0.5]
\begin{axis}[
ymode=log,
y label style={at={(axis description cs:0.040,.5)},rotate=0.0,anchor=south},
label style={font=\Large},
title style={font=\Large},
legend style={font=\Large},
yticklabel style = {font=\small},
xticklabel style = {font=\Large},
xlabel={$p_{\rm out}$},
ylabel={MSE},
title={$\gtvpenalty(\flowvec)\!=\!\|\flowvec\|_{1},\regparam\!=\!10^{-2}$},
legend pos=south east,
ymajorgrids=true,
grid style=dashed,
]


\addplot+[color=cyan!40!blue!60, very thick, solid, mark=*, error bars/.cd, y dir=both, y explicit, error bar style={line width=1pt,solid}, error mark options={line width=1pt,mark size=2pt,rotate=90}] table [x=Pout, y=N1N001MSE, y error=N1N001STD, col sep=comma]{differentNoises.csv};\addlegendentry{$\sigma\!=\!10^{-2}$}
\addplot+[color=orange, very thick, densely dotted, mark=*, error bars/.cd, y dir=both, y explicit, error bar style={line width=1pt,solid}, error mark options={line width=1pt,mark size=2pt,rotate=90}] table [x=Pout, y=N1N01MSE, y error=N1N01STD, col sep=comma]{differentNoises.csv};\addlegendentry{$\sigma\!=\!10^{-1}$}
\addplot+[color=green!40!teal!60, very thick, densely dashed, mark=*, error bars/.cd, y dir=both, y explicit, error bar style={line width=1pt,solid}, error mark options={line width=1pt,mark size=2pt,rotate=90}] table [x=Pout, y=N1N10MSE, y error=N1N10STD, col sep=comma]{differentNoises.csv};\addlegendentry{$\sigma\!=\!1$}

\end{axis}
\end{tikzpicture}
\end{minipage}
\hspace*{15mm}
\begin{minipage}[t]{0.3\columnwidth} 
\begin{tikzpicture}[scale=0.5]
\begin{axis}[
ymode=log,
y label style={at={(axis description cs:0.040,.5)},rotate=0.0,anchor=south},
label style={font=\Large},
title style={font=\Large},
legend style={font=\Large},
yticklabel style = {font=\small},
xticklabel style = {font=\Large},
xlabel={$p_{\rm out}$},
ylabel={MSE},
title={$\gtvpenalty(\flowvec)\!=\!\|\flowvec\|_{2},\regparam\!=\!10^{-2}$ (nLasso)},
legend pos=south east,
ymajorgrids=true,
grid style=dashed,
]


\addplot+[color=cyan!40!blue!60, very thick, solid, mark=*, error bars/.cd, y dir=both, y explicit, error bar style={line width=1pt,solid}, error mark options={line width=1pt,mark size=2pt,rotate=90}] table [x=Pout, y=N2N001MSE, y error=N2N001STD, col sep=comma]{differentNoises.csv};\addlegendentry{$\sigma\!=\!10^{-2}$}
\addplot+[color=orange, very thick, densely dotted, mark=*, error bars/.cd, y dir=both, y explicit, error bar style={line width=1pt,solid}, error mark options={line width=1pt,mark size=2pt,rotate=90}] table [x=Pout, y=N2N01MSE, y error=N2N01STD, col sep=comma]{differentNoises.csv};\addlegendentry{$\sigma\!=\!10^{-1}$}
\addplot+[color=green!40!teal!60, very thick, densely dashed, mark=*, error bars/.cd, y dir=both, y explicit, error bar style={line width=1pt,solid}, error mark options={line width=1pt,mark size=2pt,rotate=90}] table [x=Pout, y=N2N10MSE, y error=N2N10STD, col sep=comma]{differentNoises.csv};\addlegendentry{$\sigma\!=\!1$}

\end{axis}
\end{tikzpicture}
\end{minipage}
\\
\\
\hspace*{22mm}
\begin{minipage}[t]{0.3\columnwidth} 
\begin{tikzpicture}[scale=0.5]
\begin{axis}[
ymode=log,
y label style={at={(axis description cs:0.040,.5)},rotate=0.0,anchor=south},
label style={font=\Large},
title style={font=\Large},
legend style={font=\Large},
yticklabel style = {font=\small},
xticklabel style = {font=\Large},
xlabel={$p_{\rm out}$},
ylabel={MSE},
title={$\phi(\flowvec)\!=\!\|\flowvec\|^2_{2},\regparam\!=\!10^{-2}$ (``MOCHA'')},
legend pos=south east,
ymajorgrids=true,
grid style=dashed,
]


\addplot+[color=cyan!40!blue!60, very thick, solid, mark=*, error bars/.cd, y dir=both, y explicit, error bar style={line width=1pt,solid}, error mark options={line width=1pt,mark size=2pt,rotate=90}] table [x=Pout, y=MCN001MSE, y error=MCN001STD, col sep=comma]{differentNoises.csv};\addlegendentry{$\sigma\!=\!10^{-2}$}
\addplot+[color=orange, very thick, densely dotted, mark=*, error bars/.cd, y dir=both, y explicit, error bar style={line width=1pt,solid}, error mark options={line width=1pt,mark size=2pt,rotate=90}] table [x=Pout, y=MCN01MSE, y error=MCN01STD, col sep=comma]{differentNoises.csv};\addlegendentry{$\sigma\!=\!10^{-1}$}
\addplot+[color=green!40!teal!60, very thick, densely dashed, mark=*, error bars/.cd, y dir=both, y explicit, error bar style={line width=1pt,solid}, error mark options={line width=1pt,mark size=2pt,rotate=90}] table [x=Pout, y=MCN10MSE, y error=MCN10STD, col sep=comma]{differentNoises.csv};\addlegendentry{$\sigma\!=\!1$}

\end{axis}
\end{tikzpicture}
\end{minipage}
\caption{\label{fig_fixes_trin_set} MSE incurred by the local model parameters learnt by Algorithm \ref{alg1} 
	for the local datasets of $\graph^{(\rm SBM)}$. Algorithm \ref{alg1} has been applied to local loss functions 
	 \eqref{equ_def_local_loss_squared}, assuming local datasets are accessible for a subset $\trainingset$ 
	 of size $\rho=|\trainingset|/|\nodes| \approx 6/10$. The different sub-plots depict the results obtained from 
	 different choices for the GTV penalty function $\gtvpenalty$ \eqref{eq:5}. } 
\end{figure}

Additionally, we have created a scenario in which the data for clients in different 
clusters exhibit distinct statistics. To achieve this, we utilize different distributions for the 
features in different clusters, such as employing normal distributions with varying covariances 
(e.g., perturbed identity matrix). In this study, we assigned the covariance matrices for the first and second clusters
$\begin{bmatrix}
2.54 & 0.41\\
0.41& 0.51
\end{bmatrix}$ and $\begin{bmatrix}
2.21 & -0.81\\
-0.81 &  0.97
\end{bmatrix}$, respectively. The MSE incurred by Algorithm \ref{alg1} is $4.53 \times 10^{-6}$, whereas, for MOCHA \cite{Smith2017} it amounts to $1.33 \times 10^{-1}$ for $p_{in}=0.5$, $p_{out}=0.01$, $\rho=0.4$ and $\epsilon=0$.
\vspace*{3mm} 

\subsubsection{SBM Low-Dimensional Linear Regression with Unbalanced Clusters}
Moreover, we have expanded the previous numerical experiment to incorporate 
unbalanced clusters. Consequently, the experimental configuration remains unchanged, except 
for the number of nodes in each of the clusters, where $|\cluster{1}| = 50$ and $|\cluster{2}| = N_2$. 
Furthermore, we have made modifications to the experiment to consider scenarios where there are 
slight deviations from the assumption of independent and identically distributed (i.i.d.) data 
within a cluster. Specifically, we introduce perturbations to the cluster-wise ground truth parameter 
vector $\overline{\weights}^{(\clusteridx)}$ by adding a random perturbation of the form $\overline{\weights}^{(c)} + \frac{10^{-3} \epsilon^{(\nodeidx)}}{\norm{\overline{\weights}^{(\clusteridx)}}}$, where $\epsilon^{(i)}$ are i.i.d. realizations 
with common distribution $\mathcal{N}(\mathbf{0}, \mathbf{I})$.

\begin{figure}
\centering
\begin{tikzpicture}[scale=0.5]
  \begin{axis} [
    ymode=log, xmax=220,
    y label style={at={(axis description cs:0.040,.5)},rotate=0.0,anchor=south},
    label style={font=\Large},
    title style={font=\Large},
    legend style={font=\small},
    yticklabel style = {font=\small},
    xticklabel style = {font=\Large},
    xlabel={$N_2$},
    ylabel={MSE},
    legend style={at={(0.95,0.5)},anchor=east},
    ymajorgrids=true,
    grid style=dashed,
    ]
    \addplot coordinates {
      (50,  8.316059901068226e-07)
      (100,  1.85759761480829e-06)
      (150,  1.5536778726025486e-06)
      (200, 3.685241472529234e-06)
    }; \addlegendentry{Algorithm 1}
    \addplot coordinates {
      (50,  1.3160443742044454e-06)
      (100,  8.862304502325737e-07)
      (150,  2.181330218159828e-06)
      (200, 3.944454156004871e-06)
    }; \addlegendentry{non-iid Algorithm 1}
    \addplot coordinates {
      (50,  0.05337709280299785)
      (100,  0.05285579463487625)
      (150,  0.0840212629292639)
      (200, 0.11423456923371304)
    } ; \addlegendentry{MOCHA}
    \addplot coordinates {
      (50,  0.05864458931098675)
      (100,  0.043353568796475415)
      (150,  0.11024686902732468)
      (200, 0.0879835991537431)
    } ; \addlegendentry{non-iid MOCHA}
  \end{axis}
\end{tikzpicture}
\caption{\label{figure:sbm_unbalanced_clusters} MSE \eqref{equ_def_MSE} incurred by Algorithm \ref{alg1} 
	used to learn the local model parameters of $\graph^{(\rm SBM)}$ for $|\cluster{1}| = 50$ and $|\cluster{2}| = N_2$. 
	These results are obtained for local datasets obtained from a noise-free linear model ($\sigma=0$) and which are 
	available only for the nodes in the subset $\trainingset$ of relative size $\rho = 0.4$. \label{fig_sbm_unblanced_clusters}}
\end{figure}

For $\rho=0.4$ and $\sigma=0$ we draw $5$ realizations of $\graph^{(\rm SBM)}$ and 
compute the average and standard deviation of the resulting MSE values \eqref{equ_def_MSE} 
obtained for each realization. Figure \ref{fig_sbm_unblanced_clusters} depicts the average for the MSE values of $N_2 \in \{50, 100, 150, 200 \}$. 

\subsubsection{Star Graph} 
\label{stargraph_experiment_section}
This experiment revolves around a local datasets whose empirical graph $\graph^{(\rm star)}$ is star 
shaped as depicted in Figure \ref{fig:star_graph}. The star graph $\graph^{(\rm star)}$ consists of a 
centre node $\nodeidx=1$ which is connected to $49$ peripheral nodes via edges $\edgeidx = \{1,\nodeidx'\}$, 
for $\nodeidx' = 2,\ldots,\nrnodes=50$. Each of the $\nrnodes-1$ edges $\edgeidx$ has the same 
edge weight $\edgeweight_{\edgeidx} =1$. There are no edges between different peripheral nodes. 

\begin{figure}[htbp]
	\centering
	\begin{tikzpicture}[scale=11/5]
		\tikzstyle{every node}=[font=\small]
		\node[nred, fill={rgb:red,3;green,1;yellow,1}] (C1_2) at (0.88+3.7,2.29) {};
		\node[left=1.3 cm of C1_2,nred] (C1_1)  {};
		\node[below left =1cm and 1cm of C1_2,nred] (C1_3)  {};
		\node[below =1.3cm of C1_2,nred] (C1_4)  {};
		\node[below right =1cm and 1cm of C1_2,nred] (C1_5)  {};
		\node[above =1.3cm of C1_2,nred] (C1_6)  {};
		\node[above left =1cm and 1cm of C1_2,nred] (C1_7)  {};
		\node[above right =1cm and 1cm of C1_2,nred] (C1_8)  {};
		\node[right =1.3cm of C1_2,nred] (C1_9)  {};
		\node[above left = 0.01cm and 0.03cm of C1_1, font=\fontsize{8}{0}\selectfont,anchor=south] {$\localdataset{\nodeidx}, \localparams{\nodeidx}$}; 
		
		\draw [-] (C1_2)--(C1_1)node[draw=none,fill=none,font=\fontsize{8}{0}\selectfont,midway,above] {$A_{i,1}$};
		\draw [-] (C1_2)--(C1_3);
		\draw [-] (C1_2)--(C1_4);
		\draw [-] (C1_2)--(C1_5);
		\draw [-] (C1_2)--(C1_6);
		\draw [-] (C1_2)--(C1_7);
		\draw [-] (C1_2)--(C1_8);
		\draw [-] (C1_2)--(C1_9);
	\end{tikzpicture}
	\caption{Empirical graph $\graph^{(\rm star)}$ being a star with a centre node (filled) and $\nrnodes-1$ peripheral nodes (not filled). 
		Each node carries a local dataset $\localdataset{\nodeidx}$ for which we want to learn a local model parameters $\localparams{\nodeidx}$. 
		The quality of a specific choice for the local model parameters $\localparams{\nodeidx}$ is measured by a local loss function $\locallossfunc{\nodeidx}{\localparams{\nodeidx}}$ (that encapsulates the local dataset $\localdataset{\nodeidx}$). }
	\label{fig:star_graph}
\end{figure} 

Similar to the experiment in Section \ref{sbm_experiment_section}, each node $\nodeidx \in \nodes$ of $\graph^{(\rm star)}$ 
holds a local dataset $\localdataset{\nodeidx}$ of the form \eqref{equ_def_local_dataset_plain}. Each local 
dataset consists of $\localsamplesize{\nodeidx}=5$ data points $\big( \featurevec^{(\nodeidx,1)},\truelabel^{(\nodeidx,1)}\big), \ldots, \big( \featurevec^{(\nodeidx,\localsamplesize{\nodeidx})},\truelabel^{(\nodeidx,\localsamplesize{\nodeidx})}\big)$ with 
feature vectors $ \featurevec^{(\nodeidx,\localsampleidx)} \in \mathbb{R}^{\dimlocalmodel}$ and scalar labels $\truelabel^{(\nodeidx,\localsampleidx)}$, for $\localsampleidx=1,\ldots,\localsamplesize{\nodeidx}$. The feature vectors are realizations (draws) of i.i.d.\ random vectors 
with a common standard multivariate normal distribution $\mathcal{N}(\mathbf{0},\mathbf{I}_{\dimlocalmodel \times \dimlocalmodel})$. 
The labels of the data points are generated by a noisy linear model \eqref{equ_def_true_linear_model_SBM}. 
The true underlying weight vector $\overline{\weights}^{(\nodeidx)} \sim \mathcal{N}(\mathbf{0},\mathbf{I}_{\dimlocalmodel \times \dimlocalmodel})$ 
are i.i.d.\ realizations of a standard multivariate normal distribution $\mathcal{N}(\mathbf{0}, \mathbf{I})$.

We learn the weights $\estlocalparams{\nodeidx}$ using Algorithm \ref{alg1} with local loss $\locallossfunc{\nodeidx}{\localparams{\nodeidx}} \defeq (1/ \localsamplesize{\nodeidx} )\sum_{\sampleidx=1}^{\localsamplesize{\nodeidx}} \big( \big(\featurevec^{(\nodeidx,\sampleidx)}\big)^{T} \localparams{\nodeidx} - \truelabel^{(\nodeidx,\sampleidx)}\big)^{2}$ and a fixed number of iterations $\nriter=1000$. Our main focus here is the hyper-parameter $\regparam$ 
which represents a trade-off for the nodes between training a purely local model and getting a consensus with 
its neighbours. 

As indicated by Figure \ref{fig:lambda_clusters}, choosing GTV parameter $\regparam\geq \regparam^{(\rm crit)}$ 
beyond a critical value $\regparam^{(\rm crit)}$ forces the local model parameters $\estlocalparams{\nodeidx}$ to be 
constant over all nodes $\nodeidx \in \nodes$. This critical value is characterized by Theorem \ref{thm_main_result} 
which tells us that the solutions of GTV minimization \eqref{equ_gtvmin} are constant over well-connected 
clusters (see Definition \eqref{equ_def_well_connected_cluster}) 

\begin{figure}[htbp]
\begin{minipage}[t]{0.2\textwidth} 
\begin{tikzpicture}[scale=0.5]
\begin{axis}[
y label style={at={(axis description cs:0.012,.55)},rotate=-90,anchor=south},
label style={font=\Large},
title style={font=\Large},
yticklabel style = {font=\Large},
xticklabel style = {font=\Large},
xlabel={$\overline{w}_1$},
ylabel={$\overline{w}_2$},
title={ground truth},
ymajorgrids=true,
grid style=dashed,
]
\addplot+[only marks] table [x=w0_0, y=w0_1, col sep=comma]{star_weights.csv};
\end{axis}
\end{tikzpicture}
\end{minipage}
\hspace*{8mm}
\begin{minipage}[t]{0.2\textwidth} 
\begin{tikzpicture}[scale=0.5]
\begin{axis}[
y label style={at={(axis description cs:0.012,.55)},rotate=-90,anchor=south},
label style={font=\Large},
title style={font=\Large},
yticklabel style = {font=\Large},
xticklabel style = {font=\Large},
ymin=-3.5, ymax=-0.5,
ytick={-3,-2,-1},
xmin=0.5, xmax=3.5,
xtick={1,1.5,2,2.5,3},
xlabel={$\widehat{w}_1$},
ylabel={$\widehat{w}_2$},
title={$\regparam=0.4$},
legend pos=south east,
ymajorgrids=true,
grid style=dashed,
]
\addplot+[only marks] table [x=w5_0, y=w5_1, col sep=comma]{star_weights.csv};
\end{axis}
\end{tikzpicture}
\end{minipage}
\\
\\
\begin{minipage}[t]{0.2\textwidth} 
\begin{tikzpicture}[scale=0.5]
\begin{axis}[
y label style={at={(axis description cs:0.012,.55)},rotate=-90,anchor=south},
label style={font=\Large},
title style={font=\Large},
yticklabel style = {font=\Large},
xticklabel style = {font=\Large},
ymin=-3.5, ymax=-0.5,
ytick={-3,-2,-1},
xmin=0.5, xmax=3.5,
xtick={1,1.5,2,2.5,3},
xlabel={$\widehat{w}_1$},
ylabel={$\widehat{w}_2$},
title={$\regparam=0.5$},
legend pos=south east,
ymajorgrids=true,
grid style=dashed,
]
\addplot+[only marks] table [x=w6_0, y=w6_1, col sep=comma]{star_weights.csv};
\end{axis}
\end{tikzpicture}
\end{minipage}
\hspace*{8mm}
\begin{minipage}[t]{0.2\textwidth} 
\begin{tikzpicture}[scale=0.5]
\begin{axis}[
y label style={at={(axis description cs:0.012,.55)},rotate=-90,anchor=south},
label style={font=\Large},
title style={font=\Large},
yticklabel style = {font=\Large},
xticklabel style = {font=\Large},
xlabel={$\widehat{w}_1$},
ylabel={$\widehat{w}_2$},
ymin=-3.5, ymax=-0.5,
ytick={-3,-2,-1},
xmin=0.5, xmax=3.5,
xtick={1,1.5,2,2.5,3},
title={$\regparam\!=\!5$},
legend pos=south east,
ymajorgrids=true,
grid style=dashed,
]
\addplot+[only marks] table [x=w11_0, y=w11_1, col sep=comma]{star_weights.csv};
\end{axis}
\end{tikzpicture}
\end{minipage}
\caption{Scatter plots of the local model parameters $\estlocalparams{\nodeidx}$ learnt for the 
	empirical graph $\graph^{(\rm star)}$ by Algorithm \ref{alg1} using different choices for the regularization 
	parameter $\regparam$. The plots show that for sufficiently large $\regparam$ in \eqref{equ_gtvmin}, 
	its solutions deliver local parameter vectors $\estlocalparams{\nodeidx}$ that are identical for all nodes $\nodeidx \in \nodes$. 
	Each local model parameter vector is depicted by (potentially overlaying) blue markers.}
\label{fig:lambda_clusters}
\end{figure}

\end{document}